\documentclass[11pt]{article}
\usepackage{url}
\usepackage{smile}
\usepackage[normalem]{ulem}
\usepackage{kpfonts}
\usepackage[colorlinks, linkcolor=blue, anchorcolor=blue, citecolor=blue]{hyperref}
\usepackage[margin=1in]{geometry}
\usepackage[square,sort,comma,numbers]{natbib}
\usepackage[american]{babel}
\usepackage{csquotes}
\usepackage{amsmath}
\usepackage{amsthm}
\usepackage{soul}
\allowdisplaybreaks

\usepackage{enumitem}
\setlist[enumerate,1]{label=\normalfont{(\Roman*)},leftmargin=*}

\usepackage{algorithmic}
\usepackage[table]{xcolor}
\usepackage{booktabs}
\usepackage{caption}
\usepackage{subcaption}
\usepackage{graphicx}

\usepackage{bbm}

\newcommand{\diff}{\mathrm{d}}

\newcommand{\Reg}{\epsilon}

\definecolor{longhorn}{rgb}{0.8, 0.33, 0.0}

\usepackage{pifont}

\begin{document}


\title{Generative Neural Networks for Sinkhorn Distributionally Robust Hypothesis Testing}

\author{Fenglin Zhang, Teyan Liu, Jie Wang\thanks{ 
Fenglin Zhang is affiliated with School of Artificial Intelligence, The Chinese University of Hong Kong, Shenzhen; 
Teyan Liu is affiliated with Department of Decision Analytics and Operations, College of Business, City University of Hong Kong; 
Jie Wang is affiliated with School of Artificial Intelligence and School of Data Science, The Chinese University of Hong Kong, Shenzhen. 
Corresponding author: Jie Wang, jwang@cuhk.edu.cn. 
}}

\date{\today}  
\maketitle

\begin{abstract}

This paper studies the Sinkhorn distributionally robust hypothesis testing (SDRHT) problem, seeking a robust detector against least-favorable distributions in Sinkhorn discrepancy-based ambiguity sets centered at the empirical distributions. Existing approaches solve this problem by solving large-scale conic programs, which are not scalable. 
To overcome this, we propose a generative framework that learns least-favorable distributions and supports efficient training and end-to-end sampling. 
For the Sinkhorn discrepancy-based ambiguity sets, we first derive an equivalent conditional-KL-divergence representation with respect to kernel-smoothed reference distributions. 
This property allows us to prove strong duality for both constrained and unconstrained minimax SDRHT formulations. 
Based on the closed-form optimal detector and Brenier's theorem, we reformulate the max-min dual formulation as a maximization problem over convex potentials whose gradients characterize invertible transport maps between kernel-smoothed distributions and their least-favorable counterparts. 
We efficiently approximate these potentials using Hyper Input Convex Neural Networks (HyCNNs) equipped with stochastic gradient estimators and prove the representation power of HyCNNs and the distributional universality of their induced transport maps. 
Numerical results show that the proposed method achieves superior accuracy and robustness across different sample sizes and dimensions, while avoiding the scalability limitations of classical SDRHT methods. 

\vspace{0.5em}

\end{abstract}

\section{Introduction}

Hypothesis testing is a fundamental problem in statistics and scientific discovery, aiming to find an optimal detector that, given new data, can discriminate between two competing hypotheses while minimizing decision errors. 
The Neyman–Pearson lemma~\citep{neyman1933ix} yields the optimal likelihood-ratio test when the densities of the underlying distributions are known. 
In practice, these distributions are typically unknown; instead, the decision-maker has only access to their independent and identically distributed~(i.i.d.) samples. 
This setting motivates data-driven hypothesis tests that are both statistically
reliable and computationally efficient.

However, data-driven tests trained on finite samples are vulnerable to distributional misspecification~\citep{levy2008principles}. 
On the one hand, for some applications like healthcare~\citep{schober2019two} and anomaly detection~\citep{chandola2009anomaly}, where samples are limited and noisy, reliable estimation of the underlying distribution is challenging due to inevitable measurement and statistical errors. 
On the other hand, real data may be non-stationary or may undergo distributional shifts, leading to differences between training and testing data. 
To improve the reliability of hypothesis testing, especially in high-stakes scenarios, robust methods are necessary to circumvent these challenges~\citep{huber1965robust, gul2017minimax}.

Distributionally robust optimization~(DRO)~\citep{kuhn2025distributionally} offers a powerful paradigm for making robust decisions under distributional uncertainty.
When applied to hypothesis testing, DRO optimizes the robust detector against the least-favorable distributions (also known as worst-case distributions) within an ambiguity set consisting of plausible distributions, safeguarding against misspecification and enhancing out-of-sample reliability when applied to new, unseen data~\citep{gao2018robust}. 
The design of the ambiguity set is crucial for balancing computational tractability and robust out-of-sample performance~\citep{mohajerin2018data}. 
Popular choices include Huber's contamination model~\citep{huber1996robust}, $f$-divergence~\citep{levy2008robust, ben2013robust}, moments~\citep{delage2010distributionally, wiesemann2014distributionally},  kernels~\citep{sun2021data, zhu2021kernel}, or Wasserstein distance~\citep{gao2018robust, xie2021robust}. 

DRO with ambiguity sets constructed using the Sinkhorn discrepancy has recently received increasing attention~\cite{wang2025sinkhorn, zhang2026contextual}. 
Conventional Wasserstein DRO-based hypothesis testing restricts the worst-case distributions to have the same support as that of the empirical distributions~\cite{gao2018robust}. 
By adding entropic regularization to the Wasserstein distance, the Sinkhorn discrepancy provides a smooth alternative to the Wasserstein distance~\citep{peyre2019computational}, and the Sinkhorn discrepancy-based ambiguity sets naturally encourage continuous distributional estimates. Thus, unlike the Wasserstein distance-based ambiguity set, if a testing sample is outside the support of the observed data, it is still feasible to leverage the likelihood ratio detector to design a robust test. However, finding the least-favorable distributions in Sinkhorn-based hypothesis testing is an infinite-dimensional optimization problem. Existing approaches approximately solve it using sample average approximation and result in large-scale conic programs~\citep{wang2022data, wang2024non}, whose computational cost grows rapidly with the sample size. 
It is desirable to develop more tractable approaches for solving the robust hypothesis testing problem with Sinkhorn-based ambiguity sets. 

Parallel to these developments in robust statistics, flow-based generative models have shown remarkable success in learning complex distributions as pushforwards of simple reference distributions through invertible and differentiable maps~\citep{lai2025principles, lipman2022flow, dai2025assured}. 
This framework is not only adept at generating new samples but is also naturally suited for solving optimization problems over probability space, e.g., generating least-favorable distributions in DRO~\citep{xu2024flow, wang2025iterative, xu2025gradient}. 
By parameterizing a distribution as a flow-based transformation starting from a probability density, one can efficiently optimize an objective over the probability space using gradient-based methods~\citep{zhu2024distributionally, cheng2025worst}. 

\begin{figure}[!htb]
    \centering
    \includegraphics[width=\linewidth]{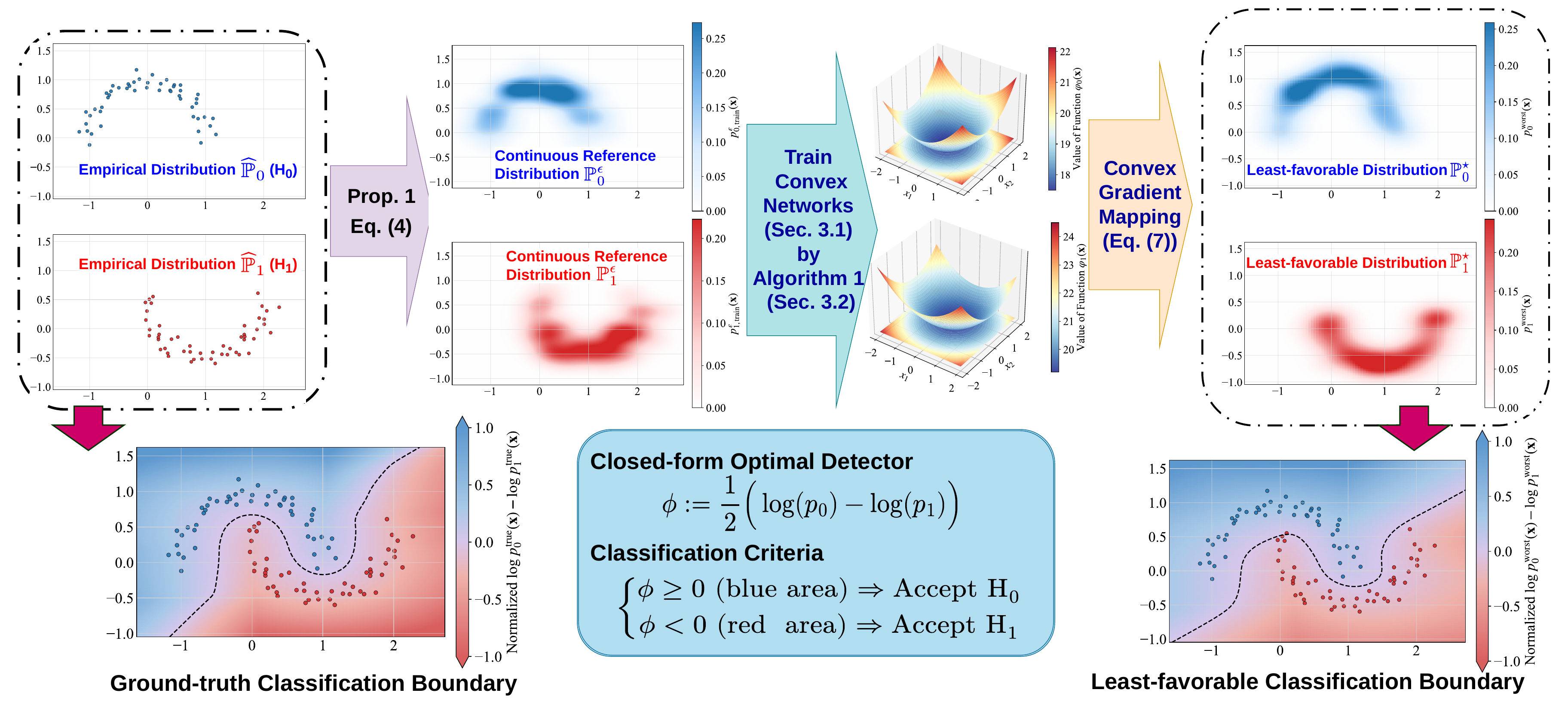}
    \caption{Illustration of the training and testing processes of our least-favorable detector. }
    \label{fig:placeholder2}
\end{figure} 

In this work, we leverage flow-based generative models to solve the Sinkhorn distributionally robust hypothesis testing (SDRHT). 
Based on the theoretical results of optimal transport, we introduce a generative framework to learn continuous least-favorable distributions within the Sinkhorn-based ambiguity sets by optimizing convex potentials. 
By parameterizing the convex potentials, our approach converts the original infinite-dimensional optimization into a finite-dimensional and tractable problem. 
Our main contributions are summarized as follows: 
\begin{itemize}
    \item 
\textbf{Novel Reformulation:} 
We establish a strong duality result for SDRHT and thereby reformulate the original minimax optimization as a max-min problem. 
As the inner minimized detector can be explicitly derived, this problem is equivalent to a maximum problem over distributions. 
Leveraging Brenier's theorem, we further reformulate it as a maximization problem over convex potentials. 

   \item 
\textbf{Generative Framework:} 
We parameterize the convex potentials by Hyper Input Convex Neural Networks (HyCNNs) and develop stochastic gradient estimators for the objective. 
This transforms the infinite-dimensional optimization into a finite-dimensional one, enabling efficient and scalable training and sampling. 
We also prove the representation power of HyCNNs and the distributional universality of their induced gradient maps.


   \item
\textbf{Empirical Validation:} We validate the performance of our method through extensive numerical studies on both synthetic datasets and real-world data with different sample sizes and dimensions, demonstrating its superiority over existing baselines in terms of both accuracy and robustness. 
\end{itemize}

Figure~\ref{fig:placeholder2} illustrates the training and testing processes of the proposed method using a toy Moons dataset. 
By exploiting the geometry of the Sinkhorn-based ambiguity set~(Proposition~\ref{prop-kl-connection} and Eq.~\eqref{Eq-pk-epsilon}), the least-favorable distributions~($\mathbb{P}_0^{\star}, \mathbb{P}_1^{\star}$) are generated by first applying kernel smoothing to the empirical distributions $\widehat{\mathbb{P}}_0 $ and $\widehat{\mathbb{P}}_1$ to obtain $\mathbb{P}_0^{\epsilon} $ and $\mathbb{P}_1^{\epsilon}$, and next robustifying them. 
In the robustification procedure, we learn convex-gradient maps~(see Eq.~\eqref{Eq-SDRO-convex-functional} and Algorithm~\ref{algorithm-HyperCNN}) that generate $\mathbb{P}_0^{\star}$ and $\mathbb{P}_1^{\star}$ based on the kernel-smoothed reference distributions $\mathbb{P}_0^{\epsilon}, \mathbb{P}_1^{\epsilon}$.
The resulting least-favorable distributions induce a nonlinear robust decision boundary with high accuracy.

The remainder of this paper is organized as follows. 
Section~\ref{sec-model} presents the models and properties of SDRHT. 
Section~\ref{sec-algo} develops a framework for generating least-favorable distributions. 
Section~\ref{sec-theory-hycnn} shows the representation power of HyCNNs. 
Section~\ref{Sec-results} reports the numerical results. 
Section~\ref{sec-conclusion} concludes this paper. 
Proofs and additional technical notes are provided in the online appendix to this paper. 

\noindent{\textbf{Notations.}} For $K \in \mathbb{N}$, let $\left [ K \right ] := \left \{ 1,\cdots,K \right \}$. 
We write $\mathbb{P} \ll \nu$ if the probability measure $\mathbb{P}$ is absolutely continuous with respect to the measure $\nu$. 
We denote $\mathbb{P} \otimes \mathbb{Q}$ by the product measure of two probability measures $\mathbb{P}$ and $\mathbb{Q}$. 
For a transport map $T: \Omega \to \Omega$, the pushforward of a probability measure $\mathbb{P}$ is denoted as $T_{\#}\mathbb{P}$, such that $T_{\#}\mathbb{P}(A) = \mathbb{P}(T^{-1}(A))$ for any measurable set $A \subseteq \Omega$. 
For a convex function $\varphi: \mathbb{R}^d \to \mathbb{R} \cup \{\infty\}$, denote its convex conjugate by $\varphi^{\ast}(\boldsymbol{y}) = \sup_{\boldsymbol{x} \in \mathbb{R}^d} \,\, \{ \boldsymbol{y}^{\top}\boldsymbol{x} - \varphi(\boldsymbol{x})\}$ for any $\boldsymbol{y} \in \mathbb{R}^d$. 

\section{Sinkhorn Distributionally Robust Hypothesis Testing} \label{sec-model}

In this section, we introduce the Sinkhorn distributionally robust hypothesis testing (SDRHT), derive its strong duality property, and obtain its equivalent functional optimization reformulation. 

Let $\Omega \subseteq \mathbb{R}^d$ be a closed sample space. 
Denote $\mathscr{P}\left(\Omega\right)$ as the set of all probability measures on $\Omega$, and $\mathcal{M}(\Omega)$ as the set of measures on $\Omega$. 
Denote by $\mathcal{P}_k \subset \mathscr{P}\left(\Omega\right)$ the ambiguity set under hypotheses $H_k$, where $k \in \mathbb{K} = \{0, 1\}$.  
Given a testing sample $\boldsymbol{\omega}$, the composite hypothesis testing decides which one of the following hypotheses is true:
\begin{equation}\nonumber
    H_0: \boldsymbol{\omega}\sim \mathbb{P}_0, \, \mathbb{P}_0 \in \mathcal{P}_0; \qquad H_1: \boldsymbol{\omega} \sim \mathbb{P}_1, \, \mathbb{P}_1 \in \mathcal{P}_1. 
\end{equation}
A hypothesis test~(detector) $\phi: \Omega \to \mathbb{R}$ is a measurable function. 
For any testing sample $\boldsymbol{\omega}$, it accepts the null hypothesis $H_0$ and rejects alternative hypothesis $H_1$ whenever $\phi(\boldsymbol{\omega})\ge 0$, otherwise accepts $H_1$ and rejects $H_0$. 
The exact risk of this detector is defined as the sum of type-I and type-II errors (also known as the false positive rate and the false negative rate, respectively). 
For this non-convex objective, we follow the generating-function approach~\cite{nemirovski2007convex, gao2018robust} and optimize the following convex surrogate risk instead: 
\begin{equation}\nonumber
    \Phi(\phi;\mathbb{P}_0,\mathbb{P}_1) := \mathbb{E}_{\boldsymbol{\omega} \sim \mathbb{P}_0}[\ell\circ(-\phi)(\boldsymbol{\omega})] + \mathbb{E}_{\boldsymbol{\omega} \sim \mathbb{P}_1}[\ell\circ \phi(\boldsymbol{\omega})],
\end{equation} 
where the generating function $\ell$ is defined as follows. 
\begin{definition}[Generating Function] \label{Def-generating}
    A generating function $\ell: \mathbb{R} \to \mathbb{R}_+ \cup \{\infty\}$ is a non-negative, non-decreasing, convex function such that $\ell(0)=1$ and $\lim_{t \to -\infty} \ell(t) = 0$. $\hfill \Diamond $
\end{definition}

For some common choices of the generating function, Table~\ref{tab:generating_functions} summarizes the corresponding closed-form optimal detectors $\phi_{\text{opt}}$ given distributions $\mathbb{P}_0$ and $\mathbb{P}_1$~\cite{gao2018robust, wang2022data}, where $p_0$ and $p_1$ represent the probability density functions of $\mathbb{P}_0$ and $\mathbb{P}_1$. 
In this paper, we choose the generating function $\ell(t):= \exp(t)$. 

\begin{table*}[htbp]
\centering
\caption{Common choices of generating function (first column), together with their corresponding optimal detector (second column) and surrogate risk (third column).}
\label{tab:generating_functions}
\begin{tabular}{ccc}
\toprule 
$\ell(t)$ & \textbf{Optimal} $\phi$ & \textbf{Corresponding optimal} $1-\Phi(\mathbb{P}_{0},\mathbb{P}_{1})/2$ \\ 
\midrule
$\exp(t)$ & $\log\sqrt{p_{0}/p_{1}}$ & $H^{2}(\mathbb{P}_{0},\mathbb{P}_{1})$ \\ 
$\log(1+\exp(t))/\log 2$ & $\log(p_{0}/p_{1})$ & $\text{JS}(\mathbb{P}_{0},\mathbb{P}_{1})/\log 2$ \\ 
$(t+1)^{2}$ & $1-2p_{0}/(p_{0}+p_{1})$ & $\chi^{2}(\mathbb{P}_{0},\mathbb{P}_{1})$ \\ 
$(t+1)_{+}$ & $\operatorname{sign}(p_{0}-p_{1})$ & $\text{TV}(\mathbb{P}_{0},\mathbb{P}_{1})$ \\ 
\bottomrule
\end{tabular}
\begin{minipage}{\linewidth}
    \small \centering
    \vspace{2pt}
    $H^2$: Hellinger divergence; JS: Jensen-Shannon divergence; TV: Total variation distance. 
\end{minipage}
\end{table*}

Given sets of training samples $\{ \boldsymbol{x}_{k}^{(1)}, \cdots, \boldsymbol{x}_{k}^{(n_k)} \}$ for all $k \in \mathbb{K}$, denote
the corresponding empirical distributions as $\widehat{\mathbb{P}}_k = \frac{1}{n_k} \sum_{i=1}^{n_k}\delta_{x_k^{(i)}}$, where $n_k$ denotes the number of samples from $\mathbb{P}_k \in \mathcal{P}_k$. 
The following definition introduces the Sinkhorn discrepancy, an entropy-regularized variant of the Wasserstein distance that quantifies the distance between distributions~\cite{wang2025sinkhorn, wang2022data}. Throughout the paper, we use the quadratic transport cost $c(x,z):=\frac{1}{2}\|x-z\|_2^2$.
\begin{definition}[Sinkhorn Discrepancy]\label{Def-Sinkhorn}
Consider any two distributions $\mathbb{P}, \mathbb{Q} \in \mathscr{P}(\Omega)$ and reference measure $\nu \in \mathcal{M}(\Omega)$ such that $\mathbb{Q} \ll \nu$.
For regularization parameter $\Reg > 0$, the Sinkhorn discrepancy between two distributions $\mathbb{P}$ and $\mathbb{Q}$ is defined as
\begin{equation}\nonumber
    \mathcal{S}_{\Reg}(\mathbb{P}, \mathbb{Q}) = \inf_{\gamma\in\Gamma(\mathbb{P}, \mathbb{Q})}~
\left\{
\mathbb{E}_{(\boldsymbol{x},\boldsymbol{z}) \sim \gamma}[c(\boldsymbol{x},\boldsymbol{z})] + \Reg \cdot \mathcal{D}_{\mathrm{KL}}(\gamma \| \mathbb{P} \otimes \nu)
\right\},
\end{equation}
where $c(\boldsymbol{x},\boldsymbol{z})$ represents the transport cost function,  $\Gamma(\mathbb{P}, \mathbb{Q})$ denotes the set of couplings of $\mathbb{P}$ and $\mathbb{Q}$, respectively, and $\mathcal{D}_{\mathrm{KL}}(\gamma \| \mathbb{P} \otimes \nu)$ denotes the relative entropy of distribution $\gamma$ with respect to measure $\mathbb{P} \otimes \nu$, i.e., 
\begin{equation}
\tag*{  $\Diamond$ }
    \mathcal{D}_{\mathrm{KL}}(\gamma \| \mathbb{P} \otimes \nu) = \int_{\Omega \times \Omega}
\log\frac{\diff\gamma(\boldsymbol{x},\boldsymbol{z})}{\diff \mathbb{P}(\boldsymbol{x})\diff \nu(\boldsymbol{z})} \diff\gamma(\boldsymbol{x},\boldsymbol{z}). 
\end{equation}  
\end{definition}

Let the space of detectors be $\mathcal{F}$. 
We consider the following SDRHT problem, optimizing a detector to minimize the least-favorable risk in ambiguity sets: 
\begin{equation}\label{Eq:hard:constrained:SDRO}
    \inf_{\phi \in \mathcal{F}}\sup_{\mathbb{P}_0\in \mathbb{B}_{\rho_0}(\widehat{\mathbb{P}}_0), \mathbb{P}_1\in \mathbb{B}_{\rho_1}(\widehat{\mathbb{P}}_1)}~\Phi(\phi;\mathbb{P}_0,\mathbb{P}_1), \tag{SDRHT}
\end{equation}
where the ambiguity sets are constructed by the Sinkhorn discrepancy $\mathcal{S}_{\Reg}(\widehat{\mathbb{P}}_k,\mathbb{P})$, offering a flexible, non-parametric framework to characterize the least-favorable distributions
\begin{equation}\label{Eq:hard:constrained:SDRO-ambiguity}
    \mathbb{B}_{\rho_k}(\widehat{\mathbb{P}}_k) := \{ \mathbb{P}: \mathcal{S}_{\Reg}(\widehat{\mathbb{P}}_k,\mathbb{P})\le \rho_k \},\quad \forall k \in \mathbb{K}. 
\end{equation} 

We also consider the following soft-constrained counterpart of problem~\eqref{Eq:hard:constrained:SDRO}:
\begin{equation}\label{Eq:soft:constrained:SDRO}
    \inf_{\phi \in \mathcal{F}}~\sup_{\mathbb{P}_0, \mathbb{P}_1 \in \mathscr{P}(\Omega)} \quad \Phi(\phi;\mathbb{P}_0,\mathbb{P}_1) - \sum_{k \in \mathbb{K}} \lambda_k \mathcal{S}_{\Reg}(\widehat{\mathbb{P}}_k,\mathbb{P}_k) , 
    \tag{soft-SDRHT}
\end{equation}  
where $\lambda_k > 0$ for all $k \in \mathbb{K}$ are penalty parameters.  

\subsection{Strong Duality} \label{sec-model-duality} 

In this subsection, since the optimal detector admits a closed form for fixed distributions, we prove the strong duality of problems~\eqref{Eq:hard:constrained:SDRO} and~\eqref{Eq:soft:constrained:SDRO} by interchanging the minimization over detectors and the maximization over distributions. 
Note that the functional space $\mathcal{F}$ is infinite-dimensional and naturally convex, but is generally not compact. 
Hence, the classical Sion's minimax theorem cannot be used to verify strong duality. 
Instead, we prove strong duality primarily using the lopsided minimax theorem~\citep{kuhn2025distributionally}, which does not require $\mathcal{F}$ to be compact. 

For the transport cost and Sinkhorn discrepancy, we introduce the following assumptions. 
\begin{assumption} \label{assumption-tcost}
    \begin{enumerate}
        \item The transport cost function $c(\boldsymbol{x},\boldsymbol{z})$ is $\widehat{\mathbb{P}} \otimes \nu$-measurable and satisfies $\nu\{z: 0 \le c(\boldsymbol{x}, \boldsymbol{z}) < \infty \}=1$ for $\widehat{\mathbb{P}}$-almost every $x$. \label{assumption-tcost-bounded}
        \item $\mathbb{E}_{\boldsymbol{z}\sim \nu} [e^{-c(\boldsymbol{x}, \boldsymbol{z})/\epsilon}] < \infty$ for $\widehat{\mathbb{P}}$-almost every $x$. \label{assumption-tcost-ex-bounded}
        \item The cost function $c(\boldsymbol{x}, \boldsymbol{z})$ is lower semicontinuous in $(\boldsymbol{x}, \boldsymbol{z})$. \label{assumption-tcost-lower-sc}
        \item For every joint distribution $\gamma$ on $\Omega \times \Omega$ with first marginal distribution $\mathbb{P}$, it has a regular conditional distribution $\gamma_{\boldsymbol{x}}$ given the value of the first marginal equals $\boldsymbol{x}$. \label{assumption-marginal-distribution} 
    \end{enumerate}
\end{assumption}

According to~\cite{wang2025sinkhorn}, Assumptions~\ref{assumption-tcost}\ref{assumption-tcost-bounded} and \ref{assumption-tcost-ex-bounded} ensure that both the Sinkhorn discrepancy and the equivalent reformulation in the following Proposition~\ref{prop-kl-connection} are well-defined. 
Assumption~\ref{assumption-tcost}\ref{assumption-tcost-lower-sc} ensures that the Sinkhorn discrepancy is weakly lower semicontinuous. 
Assumption~\ref{assumption-tcost}\ref{assumption-marginal-distribution} ensures that the joint distribution $\gamma$ can be written as $\diff \gamma (\boldsymbol{x}, \boldsymbol{z}) = \diff \widehat{\mathbb{P}} (\boldsymbol{x}) \diff \gamma_{\boldsymbol{x}}$ and the law of total expectation holds. 
Our choice of transport cost $c(x,z)=\frac{1}{2}\|x-z\|_2^2$ naturally satisfies Assumption~\ref{assumption-tcost}\ref{assumption-tcost-lower-sc}. 

Define the non-negative loss function
\begin{equation}\label{loss-lk}
    L_{k}(\phi, \boldsymbol{\omega}) := \ell\Big((-1)^{k+1}\cdot \phi(\boldsymbol{\omega})\Big), \quad \forall k \in \mathbb{K}. 
\end{equation} 
In this paper, we focus on continuous detectors and consider the following assumptions for the loss function. 
\begin{assumption}[Loss Function]\label{assumption-2}
    \begin{enumerate} 
        \item For any $k \in \mathbb{K}$, $\phi \in \mathcal{F}$, and $\beta_k > 0$, the reference distribution $\mathbb{P}_k^{\epsilon}$ in Lemma 1 satisfies $\mathbb{E}_{\boldsymbol{\omega} \sim \mathbb{P}_k^{\epsilon}} [\exp(\beta_k \cdot L_{k}(\phi, \boldsymbol{\omega}) )] < \infty$ .   \label{assumption-2-el-bound}
        \item There exists $\mathbb{Q} \in \mathscr{P}(\Omega)$ with $\inf_{\phi \in \mathcal{F}} \, \mathbb{E}_{\boldsymbol{\omega} \sim \mathbb{Q}}\left[ L_{k}(\phi, \boldsymbol{\omega}) - \lambda_k \mathcal{S}_{\Reg}(\widehat{\mathbb{P}}_k,\mathbb{Q})\right] >-\infty$ for any $k \in \mathbb{K}$. \label{assumption-2-properness}
    \end{enumerate}
\end{assumption}

Assumption~\ref{assumption-2}\ref{assumption-2-el-bound} leads to the following Lemma~\ref{lemma-wusc}. 
Assumption~\ref{assumption-2}\ref{assumption-2-properness} ensures that the optimal value of the dual of problem~\eqref{Eq:soft:constrained:SDRO} is
strictly larger than $-\infty$. 

The following proposition provides an equivalent conditional-KL-divergence representation of the Sinkhorn-based ambiguity set. 

\begin{proposition}[Conditional-KL Representation of the Sinkhorn Ambiguity Set] \label{prop-kl-connection}
    The ambiguity sets based on Sinkhorn discrepancy with $\epsilon > 0$ are equivalent to 
    \begin{equation} \label{KL-ambiguity}
        \mathbb{B}_{\bar{\rho}_k}(\mathbb{P}_k^{\epsilon}) := \left\{ \mathbb{Q}: \inf_{\gamma\in\Gamma(\widehat{ \mathbb{P}}_k, \mathbb{Q})} \, \,  \mathbb{E}_{\boldsymbol{x} \sim \widehat{\mathbb{P}}_k} \Big [ \mathcal{D}_{\mathrm{KL}} (\gamma_{\boldsymbol{x}}\| \mathcal{K}_{\boldsymbol{x}, \epsilon})\Big ] \le \bar{\rho}_k/\Reg \right\}, \quad \forall k \in \mathbb{K}, 
    \end{equation}
    where $\bar{\rho}_k := \rho_k +\Reg \mathbb{E}_{\boldsymbol{x}\sim \widehat{\mathbb{P}}_k}\Big [ \log\, \int_{\mathbb{R}^d}e^{-c(\boldsymbol{x}, \boldsymbol{z})/ \Reg} \diff \boldsymbol{z} \Big]$, probability density $\mathrm{d} \mathbb{P}_k^{\epsilon}(\boldsymbol{z}) := \mathbb{E}_{\boldsymbol{x} \sim \widehat{\mathbb{P}}_k}\left[ \diff \mathcal{K}_{\boldsymbol{x}, \epsilon}(\boldsymbol{z}) \right]$, 
    \begin{equation*}
        \diff \mathcal{K}_{\boldsymbol{x}, \epsilon}(\boldsymbol{z}) := \frac{e^{-c(\boldsymbol{x}, \boldsymbol{z}) / \epsilon}  }{\int_{\mathbb{R}^d} e^{-c(\boldsymbol{x}, \boldsymbol{u}) / \epsilon} \cdot \diff \boldsymbol{u}} \cdot \mathrm{d} v(\boldsymbol{z}). 
    \end{equation*}
\end{proposition}

The proof of Proposition~\ref{prop-kl-connection} is provided in Appendix~\ref{app-sec-prop-kl-connection}. 
Proposition~\ref{prop-kl-connection} shows the connection between the ambiguity set based on Sinkhorn discrepancy and those based on the KL divergence~\citep{ben2013robust}. 
Since the reference measure $\nu \in \mathcal{M}(\Omega)$, $\mathcal{K}_{\boldsymbol{x}, \epsilon}$ is a probability kernel, and its corresponding continuous reference distributions in Eq.~\eqref{KL-ambiguity} are probability measures, i.e., 
\begin{equation}\label{Eq-pk-epsilon}
    \mathbb{P}_k^{\epsilon} = \frac{1}{n_k} \sum_{i=1}^{n_k} \mathcal{K}_{\boldsymbol{x}_k^{(i)}, \epsilon} \in \mathscr{P}(\Omega), \quad \quad \forall k \in \mathbb{K}. 
\end{equation}
When $\Omega = \mathbb{R}^{d}$, $\nu$ is a Lebesgue measure, and the transport cost is defined as $\frac{1}{2} \| \boldsymbol{x} - \boldsymbol{z} \|_2^2$, distribution $\mathbb{P}_k^{\epsilon}$ is a Gaussian mixture because $\mathcal{K}_{\boldsymbol{x}_k^{(i)}, \epsilon} = \mathcal{N}(\boldsymbol{x}_k^{(i)}, \Reg\cdot\boldsymbol{I}_d)$ for each $k \in \mathbb{K}$ and $i \in [n_k]$. 

Based on Assumptions~\ref{assumption-tcost} and~\ref{assumption-2} and Proposition~\ref{prop-kl-connection}, the following strong duality results hold. 
\begin{theorem}[Strong Duality]\label{theo-strong-duality}
For the convex risk function $\Phi$ and Sinkhorn discrepancy $\mathcal{S}_{\Reg}$, the following results hold:
    \begin{enumerate}
        \item If Assumptions~\ref{assumption-tcost} and~\ref{assumption-2}\ref{assumption-2-el-bound} hold, then
        \begin{equation}\nonumber
            \textnormal{Optimal Value of~\eqref{Eq:hard:constrained:SDRO} = }\max_{\mathbb{P}_0\in \mathbb{B}_{\rho_0}(\widehat{\mathbb{P}}_0), \mathbb{P}_1\in \mathbb{B}_{\rho_1} (\widehat{\mathbb{P}}_1)}\,  \inf_{\phi \in \mathcal{F}} \quad \Phi(\phi;\mathbb{P}_0,\mathbb{P}_1);
        \end{equation} \label{theo-strong-duality-hard}
        \item If Assumptions~\ref{assumption-tcost} and~\ref{assumption-2} hold, then 
        \begin{equation}\nonumber
            \textnormal{Optimal Value of~\eqref{Eq:soft:constrained:SDRO} = } \max_{\mathbb{P}_0, \mathbb{P}_1 \in \mathscr{P}(\Omega)} \, \inf_{\phi \in \mathcal{F}} \quad  \Phi(\phi;\mathbb{P}_0,\mathbb{P}_1) - \sum_{k \in \mathbb{K}} \lambda_k \mathcal{S}_{\Reg}(\widehat{\mathbb{P}}_k,\mathbb{P}_k), 
        \end{equation} \label{theo-strong-duality-soft}
        where $\lambda_k > 0$ for any $k \in \mathbb{K}$. 
    \end{enumerate}
\end{theorem}

The proof of Theorem~\ref{theo-strong-duality} is based on the lopsided minimax theorem~\citep[Theorem~5.15]{kuhn2025distributionally}, and is provided in Appendix~\ref{app-sec-theo-strong-duality}. 
We prove this by showing that the Sinkhorn-based ambiguity sets are weakly compact and the expected loss is weakly upper semicontinuous. 

Theorem~\ref{theo-strong-duality} allows us to eliminate the detector minimization before optimizing over the least-favorable distributions. 
Thus, for generating function $\ell(t) = \exp(t)$, problem~\eqref{Eq:soft:constrained:SDRO} is equivalent to 
\begin{equation} \label{Eq-SDRO-phi-star}
    \max_{\mathbb{P}_0, \mathbb{P}_1 \in \mathscr{P}(\Omega)} \quad   \Phi(\phi_{\text{opt}};\mathbb{P}_0,\mathbb{P}_1) - \sum_{k \in \mathbb{K}} \lambda_k \mathcal{S}_{\Reg}(\widehat{\mathbb{P}}_k,\mathbb{P}_k). 
\end{equation}
where 
\begin{equation} \label{Eq-Phi-risk}
    \Phi(\phi_{\text{opt}};\mathbb{P}_0,\mathbb{P}_1) = \sum_{k\in \mathbb{K}} \mathbb{E}_{\boldsymbol{\omega} \sim \mathbb{P}_k} \Big [\exp \Big [ (-1)^{k+1}\cdot \frac{1}{2} \Big(\log \, p_{0} (\boldsymbol{\omega})-\log \, p_{1} (\boldsymbol{\omega})\Big)\Big] \Big]. 
\end{equation}
For a testing sample $\boldsymbol{\omega} \in \Omega$, it suffices to evaluate the log-density of least-favorable distributions $\mathbb{P}_0$ and $\mathbb{P}_1$ at $\boldsymbol{\omega}$, and then make a decision by our explicitly optimal detector. 
Thus, the robust detector retains the likelihood-ratio form, but its densities are those of the least-favorable distributions. 
Problem~\eqref{Eq-SDRO-phi-star} nevertheless remains infinite-dimensional because the least-favorable distributions are still optimization variables.
In the next subsection, we characterize these distributions using parameterized convex-gradient transport maps. 

\subsection{Reformulation via Convex-Gradient Transport Maps}\label{sec-model-convex-potential}

By setting the transport cost to the quadratic function $c(\boldsymbol{x},\boldsymbol{z}) = \frac{1}{2}\|x-z\|_2^2$, we can leverage the following fundamental result from optimal transport~\citep{chewi2025statistical} to simplify Problem~\eqref{Eq-SDRO-phi-star}.

\begin{theorem}[Brenier's Theorem] 
\label{Thm:Fundamental:OT}
Let $\mu,\nu$ be two finite-second-moment probability measures on $\mathbb{R}^d$, and assume that $\mu$ is absolutely continuous with respect to the Lebesgue measure. Then there exists a proper convex function $\varphi:~\mathbb{R}^d\to\mathbb{R}$ such that $(\nabla\varphi)_{\#}\mu=\nu$.
In addition, $\nabla\varphi$ is the unique solution to
\[
\inf_{T:~T_{\#}\mu=\nu}\mathbb{E}_{\mu}\|\boldsymbol{x} - T(\boldsymbol{x})\|_2^2,
\]
and $\gamma_{\text{opt}}\sim (\mathrm{Id},\nabla\varphi)_{\#}\mu$ is the unique solution to 
\[
\inf_{\gamma\in\Gamma(\mu,\nu)}~\mathbb{E}_{\gamma}\|\boldsymbol{x}-\boldsymbol{y}\|_2^2.
\]
\end{theorem}

Theorem~\ref{Thm:Fundamental:OT} shows that the gradient of a convex potential can characterize the transport map between continuous distributions. 
Proposition~\ref{prop-kl-connection} and Theorem~\ref{Thm:Fundamental:OT} enable us to reformulate~\eqref{Eq-SDRO-phi-star} as an equivalent optimization problem over convex gradient transport maps, i.e., 
\begin{equation}\label{Eq-SDRO-convex-functional}
\max_{\varphi_0,\varphi_1\in\mathcal{F}_{\mathrm{CVX}}}~\left\{ 
\Phi(\phi_{\text{opt}};\mathbb{P}_0,\mathbb{P}_1) - \sum_{k \in \mathbb{K}} \lambda_k \mathcal{S}_{\Reg}(\widehat{\mathbb{P}}_k,\mathbb{P}_k):~
\mathbb{P}_k = (\nabla\varphi_k)_{\#}\mathbb{P}_k^{\epsilon}
\right\}, 
\end{equation}
where $\mathcal{F}_{\mathrm{CVX}}$ is the set of convex functions $\varphi:\mathbb R^d\to\mathbb R\cup\{+\infty\}$ and the distribution $\mathbb{P}_k^{\epsilon}$ for each $k \in \mathbb{K}$ is defined in Eq.~\eqref{Eq-pk-epsilon}. 
Note that we use continuous distributions $\mathbb{P}_k^{\epsilon}$ as source distributions of the transport maps for all $k \in \mathbb{K}$, thanks to the connection between Sinkhorn discrepancy and KL divergence shown in Proposition~\ref{prop-kl-connection}. 

For computation, we restrict the convex function $\varphi_k\in C^2(\mathbb R^d)$ to be strongly convex, and $\nabla\varphi_k$ to be a global diffeomorphism for each $k\in\mathbb{K}$. 
In formulation~\eqref{Eq-SDRO-convex-functional}, given strongly convex functions $\varphi_0$ and $\varphi_1$, the log-density terms in risk function $\Phi(\phi_{\text{opt}};\mathbb{P}_0,\mathbb{P}_1)$ can be explicitly computed by the log-density form of the change-of-variable formula~\citep{lai2025principles}. Specifically, for any sample $\boldsymbol{\omega}_k \sim \mathbb{P}_k$, we have 
\begin{equation}\label{Eq-change-of-variable-1}
    \log \, p_{k} (\boldsymbol{\omega}_k) = \log \, p_{k}^{\Reg} (\boldsymbol{z}_k) - \log \, \Big | \det \nabla^2 \varphi_k (\boldsymbol{z}_k) \Big |, \quad \forall k \in \mathbb{K}, 
\end{equation} 
where $\boldsymbol{z}_k = \nabla \varphi_{k}^{-1} (\boldsymbol{\omega}_k)$, and $p_{k}^{\Reg}$ represents the probability density function of $\mathbb{P}_k^{\Reg}$ for each $k \in \mathbb{K}$. 
Based on the Fenchel–Young inequality, the gradient of a differentiable strongly convex function $\varphi$ is the inverse of the gradient of its convex conjugate, i.e., $\nabla \varphi^{-1} = \nabla \varphi^{\ast}$.  
In this case, we have $\boldsymbol{z}_k = \arg \max_{\boldsymbol{z} \in \mathbb{R}^d} \, \, \{ \boldsymbol{\omega}_k^{\top}\boldsymbol{z} - \varphi_{k}(\boldsymbol{z}) \}$. 
Similarly, to evaluate the log-density of sample $\boldsymbol{\omega}_k \sim \mathbb{P}_k$ at the other distribution $\mathbb{P}_{|1-k|}^{\epsilon}$, i.e., $\log \, p_{|1-k|} (\boldsymbol{\omega}_k)$, we trace $\boldsymbol{\omega}_k$ back to its corresponding source sample $\bar{\boldsymbol{z}}_k$ in $\mathbb{P}_{|1-k|}^{\epsilon}$:  
\begin{equation}\label{Eq-trace-source-sample}
    \bar{\boldsymbol{z}}_k = \nabla \varphi_{|1-k|}^{-1} (\boldsymbol{\omega}_k) = \arg \max_{\boldsymbol{z} \in \mathbb{R}^d} \, \, \Big \{ \boldsymbol{\omega}_k^{\top}\boldsymbol{z} - \varphi_{|1-k|}(\boldsymbol{z}) \Big \}, \quad \forall k \in \mathbb{K}, 
\end{equation}
and then we have 
\begin{equation}\label{Eq-change-of-variable-2}
    \log \, p_{|1-k|} (\boldsymbol{\omega}_k) = \log \, p_{|1-k|}^{\Reg} (\bar{\boldsymbol{z}}_k) - \log \, \Big | \det \nabla^2 \varphi_{|1-k|} (\bar{\boldsymbol{z}}_k) \Big |, \quad \forall k \in \mathbb{K}. 
\end{equation}
Since the potential $\varphi_k$ in~\eqref{Eq-SDRO-convex-functional} is assumed to be twice differentiable and strongly convex, its Hessian determinant is strictly positive, i.e., $\Big | \det \nabla^2 \varphi_k (\boldsymbol{z}) \Big | = \det \nabla^2 \varphi_k (\boldsymbol{z})$, and the log-determinant operators in Eqs.~\eqref{Eq-change-of-variable-1} and~\eqref{Eq-change-of-variable-2} are well defined.
In high-dimensional applications, the exact values of the log-determinant-Hessian operators are computationally expensive, but can be estimated efficiently by the stochastic Lanczos quadrature (SLQ) in~\cite{ubaru2017fast, huang2020convex}. 

Although formulation~\eqref{Eq-SDRO-convex-functional} exposes the transport structure of the least-favorable distributions, it remains an infinite-dimensional optimization. 
Moreover, the induced densities, inverse maps, and log-determinant terms are generally unavailable in closed form. 
However, this formulation offers a new perspective on generating least-favorable distributions by characterizing the gradients of convex potentials. 
In the next section, we propose a generative method to solve~\eqref{Eq-SDRO-convex-functional}. 

\section{A HyCNN-Based Generative Method}\label{sec-algo}

Flow-based generative models represent a target distribution as the pushforward of a simple reference distribution through an invertible, differentiable map, allowing efficient density evaluation and sample generation~\citep{lai2025principles}. 
In flow-based DRO, these models can also generate least-favorable distributions when guided by proper loss functions~\citep{xu2024flow}.  
For problem~\eqref{Eq-SDRO-convex-functional}, a natural approach is to parameterize the convex potentials by the Input Convex Neural Network (ICNN)~\citep{amos2017input, huang2020convex}.  
In this section, we develop a generative framework that parametrizes convex potentials by a variant of ICNN, i.e., \textit{Hyper Input Convex Neural Network (HyCNN)}, extending the classical ICNN with more flexible architectures~\citep{hundrieser2026hyper}. 

\subsection{The Architecture of HyCNN}\label{sec-algo-hycnn}

In this subsection, we first introduce the HyCNN architecture. 
Let the number of layers of HyCNN be $L$, and the output of HyCNN given an input $\boldsymbol{x} \in \mathbb{R}^d$ be $\mathsf{HyCNN}(\boldsymbol{x}) \in \mathbb{R}$. 
Figure~\ref{fig:hycnn} illustrates the architecture of HyCNN, including the detailed structure of each block. The forward propagation process of HyCNN can be reformulated as follows
\begin{subequations}\label{Eq-HyCNN}
    \begin{equation} \label{Eq-HyCNN-a}
        \mathsf{HyCNN}(\boldsymbol{x}) := y_{L+1} = \boldsymbol{V}_L \boldsymbol{y}_{L} + \boldsymbol{W}_L \boldsymbol{x} + b_{L},
    \end{equation}
    \begin{equation}
        \boldsymbol{y}_{l+1} := \bar{\sigma} \Big( \boldsymbol{V}_l^{(1)} \boldsymbol{y}_{l} + \boldsymbol{W}_{l}^{(1)} \boldsymbol{x} + \boldsymbol{b}_{l}^{(1)}, \boldsymbol{V}_l^{(2)} \boldsymbol{y}_{l} + \boldsymbol{W}_{l}^{(2)} \boldsymbol{x} + \boldsymbol{b}_{l}^{(2)} \Big), \quad \forall l \in \{0, \cdots, L-1\}. 
    \end{equation}
\end{subequations}
where $\{ \boldsymbol{V}, \boldsymbol{W}, \boldsymbol{b}\}$ are trainable parameters, $\boldsymbol{y}_{0}:=\boldsymbol{0}$, and $\bar{\sigma}: \mathbb{R}^2 \to \mathbb{R}$ is a scalar activation function that is convex and component-wise non-decreasing, applied element-wise to vector inputs. 

\begin{figure}[!htb]
    \centering
    \includegraphics[width=\linewidth]{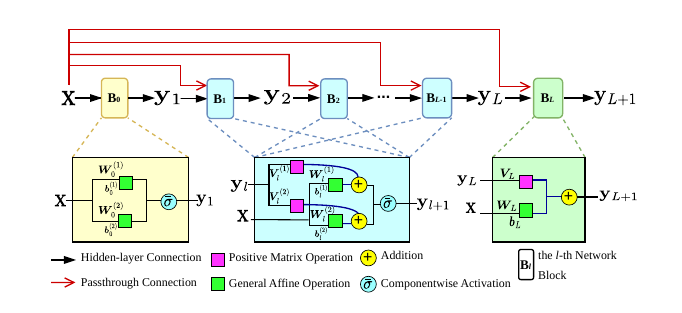}
    \caption{The architecture of HyCNN}
    \label{fig:hycnn}
\end{figure}

As shown in Figure~\ref{fig:hycnn}, in each hidden layer $l \ge 1$ of HyCNN, a ``passthrough'' layer (red arrow) directly connects the input vector with the hidden units. 
Since the nonnegative sums of convex functions are also convex and the composition of a convex and convex non-decreasing function is also convex, according to Proposition 2.1 in~\cite{hundrieser2026hyper}, this construction ensures the convexity of HyCNNs. 
\begin{proposition}[Convexity of HyCNN~\cite{hundrieser2026hyper}]
    Suppose that $\bar{\sigma}: \mathbb{R}^2 \to \mathbb{R}$ is convex and component-wise non-decreasing, and weight matrices $\{ \boldsymbol{V}_l^{(1)}, \boldsymbol{V}_l^{(2)} \}_{l=0}^{L-1}$ and $\boldsymbol{V}_L$ are component-wise nonnegative. Then, $\mathsf{HyCNN}(\boldsymbol{x})$ in Eq.~\eqref{Eq-HyCNN-a} is convex in $\boldsymbol{x}$.
\end{proposition}

In this paper, the non-negativity of weight matrices $\{ \boldsymbol{V}_l^{(1)}, \boldsymbol{V}_l^{(2)} \}_{l=0}^{L-1}$ and $\boldsymbol{V}_L$ is enforced by softplus operators. 
Note that when we choose $\bar{\sigma}(a_1, a_2) = \sigma(a_1)$, where $\sigma: \mathbb{R} \to \mathbb{R}$ is an activation function that is convex and non-decreasing (e.g., ReLU), the HyCNN reduces to the standard ICNN architecture~\cite{amos2017input}. 
Thus, HyCNNs can be interpreted as a generalization of ICNNs. 
When we choose a maxout unit~\cite{goodfellow2013maxout} as our activation function, i.e., $\bar{\sigma}(a_1, a_2) = \max(a_1, a_2)$, the resulting HyCNN is piecewise-affine but is nondifferentiable at ties.
Therefore, to ensure the trained HyCNN is differentiable, we use a smooth log-sum-exp approximation of the maxout activation function with parameter $\tau>0$:
\begin{equation} \label{Eq-smooth-maxout}
    \bar{\sigma}_{\tau}(a_1, a_2) := \tau \cdot \log \, \Big(e^{a_1/\tau}+e^{a_2/\tau}\Big), 
\end{equation}
which is referred to as the \textit{smooth} maxout activation function below. 

\subsection{Approximation for Convex Potentials}\label{sec-algo-framework}

In this subsection, we present the procedure for training HyCNNs. Given sample $\boldsymbol{z} \in \Omega$, denote the two HyCNNs we use for the problem~\eqref{Eq-SDRO-convex-functional} by $\mathsf{HyCNN}_0(\boldsymbol{z})$ and $\mathsf{HyCNN}_1(\boldsymbol{z})$, respectively. 
Denote the approximated convex function by $\widetilde{\varphi}_k$ for each $k \in \mathbb{K}$. 
To ensure the gradient of convex potentials is invertible, as the smooth maxout activation makes $\mathrm{HyCNN}_k$ twice continuously differentiable, we further add a positive quadratic term to ensure strong convexity, i.e., $\widetilde{\varphi}_k(\boldsymbol{z}) = \mathsf{HyCNN}_k(\boldsymbol{z}) + \frac{\alpha}{2} \| \boldsymbol{z}\|_2^2$ where $\alpha$ is a sufficient small positive number, such that the Hessian matrix $H_{\widetilde{\varphi}_k} \succeq \alpha \boldsymbol{I} \succ \boldsymbol{0}$. 
Unlike block-wise progressive methods that train a series of maps in~\cite{xu2024flow}, we directly learn one convex-gradient map for each hypothesis, since a composition of convex-gradient maps may not be the gradient of a convex potential.

The training process is given by the following Algorithm~\ref{algorithm-HyperCNN}. 
For the gradient of the objective function, in Step 5, we propose two $N$-sample estimators to approximate the gradient of the risk function and the Sinkhorn discrepancy, respectively; see Algorithms~\ref{algorithm-risk-estimator} and~\ref{algorithm-sinkhorn} in the next subsection for details. 
As shown, to generate the samples for estimators, we first draw $N$ samples from the continuous reference distribution $\mathbb{P}_k^{\Reg}$ (Step 2), and then map them to the target distribution space by the gradient of trained convex potentials (Step 3). 
In Step 6, the stochastic gradient ascent algorithm is used to update the network parameters. 
\begin{algorithm}[!ht] 
\caption{ Convex Potentials Approximation based on HyCNNs. }
\label{algorithm-HyperCNN}
\textbf{Require:} Number of epochs $T$, training samples $\{\boldsymbol{x}_k^{(i)}\}_{i=1}^{n_k}$, reference distributions $\mathbb{P}_{k,0} = \mathbb{P}_k^{\Reg}$. \\ 
\textbf{Output:} Convex potential approximations $\widetilde{\varphi}_{k, T}$ for all $k \in \mathbb{K}$. \\
    1: Initialize strongly convex potential $\widetilde{\varphi}_{k, 0}$ with parameters $\boldsymbol{\theta}_{k, 0} \in \Theta$ for each $k \in \mathbb{K}$;  \\ 
    2: \textbf{for } $t = 0, \cdots, T-1$ \textbf{ do} \\ 
    3: \quad Draw $N$ i.i.d. samples $\{\boldsymbol{z}_{k}^{(i)}\}_{i=1}^N$ from initial distribution $\mathbb{P}_{k, 0}$ for each $k \in \mathbb{K}$; \\ 
    4: \quad Generate samples $\{ \boldsymbol{\omega}_{k, t}^{(i)}\}_{i=1}^N \leftarrow  \{ \nabla \widetilde{\varphi}_{k, t}( \boldsymbol{z}_k^{(i)})\}_{i=1}^N$ in target distribution $\mathbb{P}_{k, t}$ for each $k \in \mathbb{K}$; \\ 
    5: \quad Estimate the gradient of objective function of~\eqref{Eq-SDRO-convex-functional} by Algorithms~\ref{algorithm-risk-estimator} and~\ref{algorithm-sinkhorn}: 
    \begin{equation}\nonumber
    \begin{aligned}
         \nabla_{\boldsymbol{\theta}_{k,t}} F(\boldsymbol{\theta}_{0,t}, \boldsymbol{\theta}_{1,t}) \leftarrow \  
         & \mathsf{Risk\_Gradient\_Estimator}(\{\boldsymbol{\omega}_0^{(i)}\}_{i=1}^N, \{\boldsymbol{\omega}_1^{(i)}\}_{i=1}^N, \widetilde{\varphi}_{0, t}, \widetilde{\varphi}_{1, t}, \boldsymbol{\theta}_{k,t}) \\ 
         & \quad - \nabla_{\boldsymbol{\theta}_{k,t}} \sum_{k \in \mathbb{K}} \lambda_k  \cdot \mathsf{Sinkhorn\_Estimator} (\{\boldsymbol{x}_k^{(i)}\}_{i=1}^{n_k}, \{ \boldsymbol{\omega}_{k, t}^{(i)}\}_{i=1}^N), \quad \forall k \in \mathbb{K}; 
    \end{aligned}
    \end{equation}
    6: \quad Update $\boldsymbol{\theta}_{k, t+1} \leftarrow  \boldsymbol{\theta}_{k, t} + \eta_t\cdot \nabla_{\boldsymbol{\theta}_{k,t}}F(\boldsymbol{\theta}_{0,t}, \boldsymbol{\theta}_{1,t})$, where $\eta_t$ is the learning rate; \\
    7: \textbf{end for} 
\end{algorithm}

\subsection{Stochastic Gradient Estimation} \label{sec-algo-estimator}

Next, we develop stochastic estimators required by Algorithm~\ref{algorithm-HyperCNN} to solve Problem~\eqref{Eq-SDRO-convex-functional}. 
Let $\{\boldsymbol{z}_0^{(i)}\}_{i=1}^N$ and $\{ \boldsymbol{z}_1^{(i)}\}_{i=1}^N$ be $N$ independent and identically distributed (i.i.d.) samples from reference distributions $\mathbb{P}_0^{\epsilon}$ and $\mathbb{P}_1^{\epsilon}$ (see Eq.~\ref{Eq-pk-epsilon}), respectively.  
These points are subsequently transformed into samples from the target distributions $\mathbb{P}_0$ and $\mathbb{P}_1$ via transport maps, i.e., $\boldsymbol{\omega}_0= \nabla\varphi_0 (\boldsymbol{z}_0)$ and $\boldsymbol{\omega}_1 = \nabla\varphi_1 (\boldsymbol{z}_1)$, where the underlying distributions satisfy $\mathbb{P}_0 = [\nabla\varphi_0]_{\#} \mathbb{P}_0^{\epsilon}$ and $\mathbb{P}_1 = [\nabla\varphi_1]_{\#} \mathbb{P}_1^{\epsilon}$. 

\subsubsection{Risk Gradient Estimator}

In this part, we present the gradient estimator of the risk function with respect to the parameters of HyCNNs in Algorithm~\ref{algorithm-HyperCNN}. 
For every $k \in \mathbb{K}$ and $\boldsymbol{\omega}_k \sim \mathbb{P}_k$, we define the basic element of the risk function as \[
h_k(\boldsymbol{\omega}_k, \boldsymbol{\theta}_0, \boldsymbol{\theta}_1) := \exp \Big [ \frac{1}{2} (\log \, p_{|1-k|} (\boldsymbol{\omega}_k) - \log \, p_{k} (\boldsymbol{\omega}_k) ) \Big ], \quad \forall k \in \mathbb{K}. \]
As an example, given convex functions $\widetilde{\varphi}_0$ and $\widetilde{\varphi}_1$ with parameters $\boldsymbol{\theta}_0$ and $\boldsymbol{\theta}_1$, respectively, the gradients of $h_0(\boldsymbol{\omega}_0, \boldsymbol{\theta}_0, \boldsymbol{\theta}_1)$ with respect to $\boldsymbol{\theta}_0$ and $\boldsymbol{\theta}_1$ are given by 
\begin{subequations}\label{Eq-risk-gradient}
    \begin{equation}\label{Eq-risk-gradient-0}
        \frac{\partial h_0(\boldsymbol{\omega}_0, \boldsymbol{\theta}_0, \boldsymbol{\theta}_1)}{\partial \boldsymbol{\theta}_0}  = \frac{h_0(\boldsymbol{\omega}_0, \boldsymbol{\theta}_0, \boldsymbol{\theta}_1)}{2} \cdot \Big [ \left( \frac{\partial \boldsymbol{\omega}_0}{\partial \boldsymbol{\theta}_0} \right)^\top [\nabla^2 \widetilde{\varphi}_1(\bar{\boldsymbol{z}}_0)]^{-1} \boldsymbol{g} + \frac{\partial }{\partial \boldsymbol{\theta}_0}  \log \, \det \nabla^2 \widetilde{\varphi}_0 (\boldsymbol{z}_0)\Big], 
    \end{equation}
    \begin{equation}\label{Eq-risk-gradient-1}
        \frac{\partial h_0(\boldsymbol{\omega}_0, \boldsymbol{\theta}_0, \boldsymbol{\theta}_1)}{\partial \boldsymbol{\theta}_1} = - \frac{h_0(\boldsymbol{\omega}_0, \boldsymbol{\theta}_0, \boldsymbol{\theta}_1)}{2} \cdot \Big[ \Big( \frac{\partial \nabla \widetilde{\varphi}_1 (\bar{\boldsymbol{z}}_0)}{\partial \boldsymbol{\theta}_1}\Big )^{\top} [ \nabla^2 \widetilde{\varphi}_1(\bar{\boldsymbol{z}}_0) ]^{-1} \boldsymbol{g} + \frac{\partial }{\partial \boldsymbol{\theta}_1} \log \, \det \nabla^2 \widetilde{\varphi}_{1} (\bar{\boldsymbol{z}}_0) \Big], 
    \end{equation}
\end{subequations}
where $\boldsymbol{g} := \frac{\partial \log \, p_{1} (\boldsymbol{\omega}_0)}{\partial \bar{\boldsymbol{z}}_0} =\frac{\partial  [ \log \, p_{1}^{\Reg} (\bar{\boldsymbol{z}}_0) - \log \, \det \nabla^2 \widetilde{\varphi}_{1} (\bar{\boldsymbol{z}}_0) ]}{\partial \bar{\boldsymbol{z}}_0}$, and $\bar{\boldsymbol{z}}_k$ is defined in Eq.~\eqref{Eq-trace-source-sample} for each $k \in \mathbb{K}$. 
For brevity, the detailed gradient derivation can be found in Appendix~\ref{app-sec-hypercnn-details-risk}. 
Eqs.~\eqref{Eq-risk-gradient-0} and \eqref{Eq-risk-gradient-0} indicate that directly differentiating the risk requires backpropagation through inverse maps and the log-determinant-Hessian (LDH), which is computationally expensive. 
Therefore, in the following, we develop a matrix-free stochastic gradient estimator for the LDH in Algorithm~\ref{algorithm-log-determinant-estimator}, and then present the estimator for the risk function in Algorithm~\ref{algorithm-risk-estimator}.

For the gradient of LDH, we first introduce an $\mathcal{O}(1)$-memory estimator for the gradient of the log-determinant~\citep{huang2020convex}. 
Specifically, for any invertible matrix $H$ with parameters $\boldsymbol{\theta}$: 
\begin{equation} \label{Eq-log-det-trace}
    \frac{\partial }{\partial \boldsymbol{\theta}}  \log \, \det \, H = \text{tr} (H^{-1} \frac{\partial H}{\partial \boldsymbol{\theta}}) = \mathbb{E}_{\boldsymbol{v}} [\boldsymbol{v}^{\top} H^{-1} \frac{\partial H}{\partial \boldsymbol{\theta}} \boldsymbol{v}], 
\end{equation}
where the last equality uses the Hutchinson trace estimator~\citep{hutchinson1989stochastic} with an isotropic random vector (e.g., Rademacher or standard Gaussian) $\boldsymbol{v} \in \mathbb{R}^{d}$ satisfying $\mathbb{E}[\boldsymbol{v}] = 0$ and $\mathbb{E}[\boldsymbol{v}\boldsymbol{v}^{\top}] = \boldsymbol{I}$. 
With exact linear solves, the Hutchinson construction offers an unbiased estimator of the gradient of the log-determinant. 
In this problem, $H$ is a symmetric positive definite Hessian matrix. Thus, $\boldsymbol{v}^{\top} H^{-1}$ in Eq.~\eqref{Eq-log-det-trace} can be efficiently solved by the following quadratic optimization problem 
\begin{equation}\label{Eq-quadratic-optimization}
    \min_{\boldsymbol{u}} \quad \frac{1}{2} \boldsymbol{u}^{\top} H \boldsymbol{u} - \boldsymbol{v}^{\top} \boldsymbol{u},  
\end{equation}
which has the unique minimizer $\boldsymbol{u}_{\text{opt}} = H^{-1}\boldsymbol{v} = (\boldsymbol{v}^{\top} H^{-1})^{\top}$. 
Then, the gradient of the LDH in our problem can be estimated as $\mathbb{E}_{\boldsymbol{v}} [\frac{\partial }{\partial \boldsymbol{\theta}} ( \boldsymbol{u}_{\text{opt}}^{\top} H \boldsymbol{v})]$. 
In this paper, we summarize this procedure as Algorithm~\ref{algorithm-log-determinant-estimator}. 

\begin{algorithm}[!ht] 
\caption{ $\mathsf{GLDH\_Estimator}$: Estimator for the Gradient of LDH. }
\label{algorithm-log-determinant-estimator}
\textbf{Require:} Symmetric positive definite matrix $H$, parameter $\boldsymbol{\theta}$, and sample size $M$. \\
    1: Initialize estimator $\boldsymbol{E} = \boldsymbol{0}$ with the same dimension as $\boldsymbol{\theta}$; \\
    2: \textbf{for } $m = 1, \cdots, M$ \textbf{ do} \\
    3: \quad Sample Rademacher random vector $\boldsymbol{v}^{(m)}$ with $d$ dimensions; \\  
    4: \quad $\boldsymbol{u}^{(m)} \leftarrow  \mathsf{ConjugateGradient}(H, \boldsymbol{v}^{(m)})$; \\  
    5: \quad $\overline{\boldsymbol{u}}^{(m)}
    \leftarrow
    \mathsf{StopGradient}
    \bigl(\boldsymbol{u}^{(m)}\bigr)$; \\
    6: \quad $\boldsymbol{E} \leftarrow  \boldsymbol{E}+ \nabla_{\boldsymbol{\theta}} \Big [ \left(\overline{\boldsymbol{u}}^{(m)}\right)^{\top} H \boldsymbol{v}^{(m)} \Big ]$; \\
    7: \textbf{end for} \\
    8: \textbf{Return }$\boldsymbol{E}/M$.  
\end{algorithm}

In Step 4 of Algorithm~\ref{algorithm-log-determinant-estimator}, the function $\mathsf{ConjugateGradient}(H, \boldsymbol{v}):= \arg\min_{\boldsymbol{u}} \, \frac{1}{2} \boldsymbol{u}^{\top} H \boldsymbol{u} - \boldsymbol{v}^{\top} \boldsymbol{u}$ represents that we use the conjugate gradient method to solve the unconstrained optimization problem~\eqref{Eq-quadratic-optimization}. 
In Step 5, $\mathsf{StopGradient}(\cdot)$ treats its argument as a constant
during backpropagation. Thus, in Step 6, only $H(\boldsymbol{\theta})$ is
differentiated, and the term $\nabla_{\boldsymbol{\theta}} [ \left(\overline{\boldsymbol{u}}^{(m)}\right)^{\top} H \boldsymbol{v}^{(m)} ]$ can be efficiently computed by Hessian-vector and vector-Jacobian products.

Then, the gradient of the risk function~\eqref{Eq-Phi-risk} can be estimated by Algorithm~\ref{algorithm-risk-estimator}.
Both functions $\mathsf{ConjugateGradient}$ in Step 5 and $\mathsf{GLDH\_Estimator}$ have been defined in Algorithm~\ref{algorithm-log-determinant-estimator}. 
In Step 6, we compute the value of the log-determinant in $h_k(\boldsymbol{\omega}_k, \boldsymbol{\theta}_0, \boldsymbol{\theta}_1)$ exactly when $d \le 32$ and estimate it by the stochastic Lanczos quadrature (SLQ)~\citep{ubaru2017fast, huang2020convex} when $d>32$.

\begin{algorithm}[!ht] 
\caption{ $\mathsf{Risk\_Gradient\_Estimator}$: Gradient Estimator for the Risk Function. }
\label{algorithm-risk-estimator}
\textbf{Require:} Generated samples $\{\boldsymbol{\omega}_k^{(i)}\}_{i=1}^N$ and convex potentials $\widetilde{\varphi}_{k}$ for all $k \in \mathbb{K}$, parameter $\boldsymbol{\theta}$. \\ 
    1: Initialize estimator $\boldsymbol{E} = \boldsymbol{0}$ with the same dimension as $\boldsymbol{\theta}$; \\
    2: \textbf{for } $k = 0, 1$ \textbf{ do} \\
    3: \quad \textbf{for } $i = 1, \cdots, N$ \textbf{ do} \\
    4: \quad \quad Compute $\boldsymbol{z}_k^{(i)} = \nabla \widetilde{\varphi}_k^{-1} (\boldsymbol{\omega}_k^{(i)})$ and $\bar{\boldsymbol{z}}_k^{(i)} = \nabla \widetilde{\varphi}_{|1-k|}^{-1} (\boldsymbol{\omega}_k^{(i)})$ by Eq.~\eqref{Eq-trace-source-sample}; \\
    5: \quad \quad Compute $\boldsymbol{u}_k^{(i)} \leftarrow \mathsf{ConjugateGradient} \Big(\nabla^2 \widetilde{\varphi}_{|1-k|}(\bar{\boldsymbol{z}}_k^{(i)}), \frac{\partial \log \, p_{|1-k|} (\boldsymbol{\omega}_k^{(i)})}{\partial \bar{\boldsymbol{z}}_k^{(i)}} \Big)$; \\ 
    6: \quad \quad Update the gradient estimator \begin{equation}\nonumber
        \begin{aligned}
            \boldsymbol{E} \leftarrow  \boldsymbol{E} & + \frac{h_k(\boldsymbol{\omega}_k^{(i)}, \boldsymbol{\theta}_0, \boldsymbol{\theta}_1)}{2} \cdot \Bigg \{ \Big[ \frac{\partial \Big(\boldsymbol{\omega}_k^{(i)} - \nabla \widetilde{\boldsymbol{\varphi}}_{|1-k|}(\bar{\boldsymbol{z}}_k^{(i)})\Big)}{\partial \boldsymbol{\theta}} \Big]^{\top} \boldsymbol{u}_k^{(i)} \\ & \quad \quad + \mathsf{GLDH\_Estimator}\Big(\nabla^2 \widetilde{\varphi}_{k}(\boldsymbol{z}_k^{(i)}), \boldsymbol{\theta} \Big)  - \mathsf{GLDH\_Estimator} \Big(\nabla^2 \widetilde{\varphi}_{|1-k|}(\bar{\boldsymbol{z}}_k^{(i)}), \boldsymbol{\theta} \Big) \Bigg \};
        \end{aligned}
    \end{equation}
    7: \quad \textbf{end for} \\
    8: \textbf{end for} \\
    9: \textbf{Return }$\boldsymbol{E}/N$. 
\end{algorithm}

\subsubsection{Sinkhorn Discrepancy Estimator}

Computing the Sinkhorn discrepancy term $\mathcal{S}_{\Reg}(\widehat{\mathbb{P}}_k,\mathbb{P}_k)$ exactly (see Definition~\ref{Def-Sinkhorn}) in the objective function is challenging. 
In this part, we present a differentiable finite-sample estimator for the Sinkhorn discrepancy. 
The gradient of this estimator can be computed efficiently by backpropagation and is used in Algorithm~\ref{algorithm-HyperCNN}. 

We first introduce the following proposition regarding the dual formulation of Sinkhorn discrepancy~\citep{chewi2025statistical, pooladian2021entropic}. 
\begin{proposition} \label{prop-Sinkhorn-duality}
Let $\mathbb{P}$ and $\mathbb{Q}$ be probability measures with finite second moments. The strong dual formulation of Sinkhorn discrepancy $\mathcal{S}_{\Reg}(\mathbb{P}, \mathbb{Q})$ is given by 
\begin{equation} \nonumber
    \sup_{f \in L^1(\mathbb{P}), g \in L^1(\mathbb{Q})} \quad \Big\{ \int f(\boldsymbol{x}) \ \diff \mathbb{P}(\boldsymbol{x}) + \int g(\boldsymbol{z}) \ \diff \mathbb{Q} (\boldsymbol{z}) - \Reg \iint \left[\exp \Big (\frac{f(\boldsymbol{x})+g(\boldsymbol{z}) -  c(\boldsymbol{x}, \boldsymbol{z})}{\Reg} \Big)-1 \right] \diff \mathbb{P}(\boldsymbol{x}) \diff \mathbb{Q}(\boldsymbol{z}) \Big\}, 
\end{equation}
Let $f_{\Reg}$ and $g_{\Reg}$ be the solutions of the dual formulation, which is also known as optimal entropic potentials. Then the transport plan 
\begin{equation} \nonumber
    \gamma_{\Reg}(\diff \boldsymbol{x}, \diff \boldsymbol{z}) = \exp\left(\frac{f_{\Reg}(\boldsymbol{x})+g_{\Reg}(\boldsymbol{z}) -  c(\boldsymbol{x}, \boldsymbol{z})}{\Reg}\right) \diff \mathbb{P}(\boldsymbol{x}) \diff \mathbb{Q}(\boldsymbol{z})
\end{equation}
is the unique solution to the Sinkhorn discrepancy, and the optimal entropic potentials satisfy
\begin{subequations}\label{Eq-dual-marginal}
    \begin{equation}
        \int \exp\left(\frac{f_{\Reg}(\boldsymbol{x})+g_{\Reg}(\boldsymbol{z}) - c(\boldsymbol{x}, \boldsymbol{z})}{\Reg}\right) \diff \mathbb{P}(\boldsymbol{x}) = 1, \quad \forall \boldsymbol{z} \in \Omega, 
    \end{equation}
    \begin{equation}
        \int \exp\left(\frac{f_{\Reg}(\boldsymbol{x})+g_{\Reg}(\boldsymbol{z}) -  c(\boldsymbol{x}, \boldsymbol{z})}{\Reg}\right) \diff \mathbb{Q}(\boldsymbol{z}) = 1, \quad \forall \boldsymbol{x} \in \Omega, 
    \end{equation}
\end{subequations}
for $\mathbb{Q}$-almost every $\boldsymbol{z}$ and $\mathbb{P}$-almost every $\boldsymbol{x}$, respectively. 
\end{proposition} 

Since the target continuous distribution $\mathbb{P}_k$ is computationally intractable, we approximate it as empirical distributions formed by the generated samples, denoted by $\widetilde{\mathbb{P}}_0 = \frac{1}{N} \sum_{i=1}^{N} \delta_{\boldsymbol{\omega}_0^{(i)}}$ and $\widetilde{\mathbb{P}}_1 = \frac{1}{N} \sum_{i=1}^{N} \delta_{\boldsymbol{\omega}_1^{(i)}}$.  
Therefore, it suffices to compute the value $\mathcal{S}_{\Reg, (n_k, N)}(\widehat{\mathbb{P}}_k, \widetilde{\mathbb{P}}_k)$. 
The optimal entropic potentials between discrete distributions $\widehat{\mathbb{P}}_k$ and $\widetilde{\mathbb{P}}_k$, denoted by $f_{\Reg, k, (n_k, N)}$ and $g_{\Reg, k, (n_k, N)}$, can be obtained by Sinkhorn's algorithm~\citep{peyre2019computational, pooladian2021entropic}, leading to the following estimator for the barycentric projection of the optimal entropic coupling: 
\begin{equation}\nonumber
    T_{\Reg, k, (n_k, N)} (\boldsymbol{x}) = \frac{\frac{1}{N} \sum_{i=1}^{N} \boldsymbol{\omega}_k^{(i)} e^{\Big (g_{\Reg, k, (n_k, N)}(\boldsymbol{\omega}_k^{(i)}) - c(\boldsymbol{x}, \boldsymbol{\omega}_k^{(i)}) \Big )/\Reg}}{\frac{1}{N} \sum_{i=1}^{N} e^{\Big (g_{\Reg, k, (n_k, N)}(\boldsymbol{\omega}_k^{(i)}) -  c(\boldsymbol{x}, \boldsymbol{\omega}_k^{(i)}) \Big )/\Reg}}, \quad \forall k \in \mathbb{K}. 
\end{equation}

However, the discrete dual potential $g_{\Reg, k, (n_k, N)}$ is initially defined only on the generated support points and is thus not differentiable for all samples in $\Omega$. 
To make the optimized empirical Sinkhorn discrepancy differentiable with respect to the generated sample locations, we extend $g$ to the general continuous sample space for fixed $f$: 
\begin{equation}\label{Eq-sinkhorn-potential-extension}
    g_{\epsilon, k}^{c}(\boldsymbol{\omega}) = -\epsilon \log \left( \frac{1}{n_k} \sum_{j=1}^{n_k} \exp \Big( \frac{f_k(\boldsymbol{x}_k^{(j)}) -  c(\boldsymbol{x}_k^{(j)}, \boldsymbol{\omega}) }{\epsilon} \Big) \right), \quad \forall \boldsymbol{\omega} \in \Omega. 
\end{equation}
This extension agrees with the original potential $g$ on the empirical support at convergence while making it differentiable. 
Therefore, our differentiable estimator for the Sinkhorn discrepancy is given by 
\begin{equation}\label{Eq-estimator-sinkhorn}
    \widehat{\mathcal{S}}_{\Reg, (n_k, N)}(\widehat{\mathbb{P}}_k,\mathbb{P}_k) := \mathcal{S}_{\Reg, (n_k, N)}(\widehat{\mathbb{P}}_k, \widetilde{\mathbb{P}}_k) = \frac{1}{n_k} \sum_{j=1}^{n_k}f_{\Reg, k, (n_k, N)} (\boldsymbol{x}_k^{(j)}) + \frac{1}{N} \sum_{i=1}^{N}  g_{\epsilon, k}^{c}(\boldsymbol{\omega}_k^{(i)}), 
\end{equation}
and its gradient can be computed by backpropagation (see Appendix~\ref{app-sec-hypercnn-details-sinkhorn}). 

Therefore, the pseudo-code for estimating the differentiable Sinkhorn discrepancy is shown in Algorithm~\ref{algorithm-sinkhorn}. 
Steps 1-6 in Algorithm~\ref{algorithm-sinkhorn} are essentially Sinkhorn's algorithm~\citep{pooladian2021entropic}, and thus the computational complexity of each iteration is $\mathcal{O}(n_k \cdot N)$. 
This alternating algorithm computes the entropic potentials by iteratively updating $f$ and $g$ while satisfying one of the two optimal marginal conditions in Eq.~\eqref{Eq-dual-marginal}. 

\begin{algorithm}[!ht]
\caption{$\mathsf{Sinkhorn\_Estimator}$: Estimator for the Sinkhorn Discrepancy. }
\label{algorithm-sinkhorn}
\textbf{Require:} Training samples $\{\boldsymbol{x}_k^{(i)}\}_{i=1}^{n_k}$ from $\widehat{\mathbb{P}}_k$ and generated samples $\{ \boldsymbol{\omega}_{k}^{(i)}\}_{i=1}^N$. \\ 
    1: Initialize $i = 0$ and $f_{\epsilon, k}^{(0)} = 0$ for each $k$;  \\
    2: \textbf{while } not converged \textbf{ do} \\
    3: \quad update $g_{\epsilon, k}^{(i+1)}(\boldsymbol{\omega}_k) \leftarrow  -\Reg \log \, \frac{1}{n_k} \sum_{j=1}^{n_k} \exp \Big( \frac{f_{\epsilon, k}^{(i)}(\boldsymbol{x}_k^{(j)}) - c(\boldsymbol{x}_k^{(j)}, \boldsymbol{\omega}_k) }{\Reg} \Big)$ for each $k\in \mathbb{K}$; \\ 
    4: \quad update $f_{\epsilon, k}^{(i+1)}(\boldsymbol{x}_k) \leftarrow  -\Reg \log \, \frac{1}{N} \sum_{j=1}^{N} \exp \Big( \frac{g_{\epsilon, k}^{(i+1)}(\boldsymbol{\omega}_k^{(j)}) - c(\boldsymbol{x}_k, \boldsymbol{\omega}_k^{(j)})}{\Reg} \Big)$ for each $k\in \mathbb{K}$; \\ 
    5: \quad $i \leftarrow i+ 1$; \\
    6: \textbf{end while} \\ 
    7: Extend the entropic potential $g_{\epsilon, k}^{(i+1)}(\boldsymbol{\omega}_k)$ to continuous space by Eq.~\eqref{Eq-sinkhorn-potential-extension}; \\
    8: \textbf{Return } $\widetilde{\mathcal{S}}_{\Reg, (n_k, N)}(\widehat{\mathbb{P}}_k, \mathbb{P}_{k, t})$ computed by Eq.~\eqref{Eq-estimator-sinkhorn}. 
\end{algorithm}





\section{Representation Power of HyCNNs}\label{sec-theory-hycnn} 

This section provides theoretical analyses of the representation capacity of HyCNNs. 
Specifically, we focus on their representation power for both convex functions and their gradients, as well as the distributional universality of their induced transport maps. 

We refer to HyCNNs equipped with \textit{smooth} maxout activation functions $\bar{\sigma}_{\tau}$ (see Eq.~\eqref{Eq-smooth-maxout}) as \textit{smooth HyCNNs}. 
We first show the representation power of both standard and smooth HyCNNs. 
Let $L$ and $h$ denote the number of layers and the width at each hidden layer of HyCNNs, respectively. 
Since we take the strongly convex coefficient $\alpha \to 0$, these results hold in our strongly convex setting. 

In the following results, we take the strongly convex coefficient $\alpha \to 0$. 

\begin{theorem}[Representation Power of (Smooth) HyCNNs]\label{theorem-hycnn-rp}
    For any $L_f$-Lipschitz convex function $f: \mathbb{R}^d  \to \mathbb{R}$, compact convex set $\mathbb{D} \subset \mathbb{R}^d$, and $\varepsilon > 0$, assume that $\mathbb{D}$ has a diameter $D\in(0,\infty)$ such that $\|\boldsymbol{x}_1 - \boldsymbol{x}_2\| \le D$ for any $\boldsymbol{x}_1, \boldsymbol{x}_2 \in \mathbb{D}$, and then: 
    \begin{enumerate}
        \item There exists a hyper input convex neural network $\widehat{\mathsf{HyCNN}}$ with maxout activation functions satisfying $\sup_{\boldsymbol{x} \in \mathbb{D}} | \widehat{\mathsf{HyCNN}}(\boldsymbol{x}) - f(\boldsymbol{x}) | \le \varepsilon$ and $L \cdot h = \mathcal{O}((\frac{DL_f} {\varepsilon})^d)$. \label{theorem-hycnn-rp-maxout}
        \item There exists a smooth HyCNN with a sufficiently small parameter $\tau>0$, denoted by $\widehat{\mathsf{HyCNN}}^{\tau}$, satisfying $\sup_{\boldsymbol{x} \in \mathbb{D}} | \widehat{\mathsf{HyCNN}}^{\tau}(\boldsymbol{x}) - f(\boldsymbol{x}) | \le \varepsilon$ and $L \cdot h = \mathcal{O}((\frac{DL_f} {\varepsilon})^d)$. \label{theorem-hycnn-rp-smooth-1}
        \item There exists a sequence of smooth HyCNNs, denoted by $\{\widehat{\mathsf{HyCNN}}_n^{\tau}\}_{n=1}^{\infty}$, satisfying  $\widehat{\mathsf{HyCNN}}_n^{\tau}\to f$ uniformly on $\mathbb{D}$, 
        when $\tau$ is sufficiently small and $L \cdot h = \mathcal{O}((\frac{DL_f} {\varepsilon})^d)$. 
        \label{theorem-hycnn-rp-smooth-2}
    \end{enumerate}
\end{theorem}
The proof of Theorem~\ref{theorem-hycnn-rp} is provided in Appendix~\ref{app-sec-theo-hycnn-rp}.  
Compared to existing qualitative representation power theorems of ICNNs~\citep{chen2018optimal, huang2020convex}, Theorem~\ref{theorem-hycnn-rp} extends this result to HyCNNs and provides quantitative guidelines for hyperparameter selection. 

Although these properties of (smooth) HyCNNs show that they can uniformly approximate convex potentials on a compact set, convergence of potentials does not generally imply convergence of the gradient fields~\citep{huang2020convex}.
Theorem~\ref{theo-hycnn-gradient-convergence} shows that the smooth HyCNNs approximate the gradient of convex functions well. 

\begin{theorem} [Gradient Convergence of Smooth HyCNNs] \label{theo-hycnn-gradient-convergence}
    Let $f:\mathbb{R}^d\to\mathbb{R}$ be a Lipschitz continuous convex function, and $\{\widehat{\mathsf{HyCNN}}_n^{\tau}\}_{n=1}^{\infty}$ be a sequence of smooth HyCNNs with hyperparameters $L_n \cdot h_n = \mathcal{O}({\varepsilon_n}^{-d})$ for any $\varepsilon_n>0$, where $\widehat{\mathsf{HyCNN}}_n^{\tau}: \mathbb{R}^d\to\mathbb{R}$. 
    Assume that $ \widehat{\mathsf{HyCNN}}_n^{\tau} \to f $. Then 
    \begin{enumerate}
        \item For almost every $\boldsymbol{x} \in \mathbb{R}^d$, $\nabla \widehat{\mathsf{HyCNN}}_n^{\tau} (\boldsymbol{x}) \to \nabla f(\boldsymbol{x})$.  \label{theo-hycnn-gradient-convergence-1}
        \item If $f$ is Lipschitz smooth, the gradient sequence $\nabla \widehat{\mathsf{HyCNN}}_n$ converges to $\nabla f$ uniformly at a rate of $\mathcal{O}\left(\sqrt{\varepsilon_n} \right)$. \label{theo-hycnn-gradient-convergence-2}
    \end{enumerate} 
\end{theorem}
The proof of Theorem~\ref{theo-hycnn-gradient-convergence} is provided in Appendix~\ref{appsec-theo-hycnn-gradient-convergence}. 
Generally, Theorem~\ref{theo-hycnn-gradient-convergence} holds for any differentiable convex function sequence. 
Note that the uniform convergence assumption in Theorem~\ref{theo-hycnn-gradient-convergence} can be verified using the representation power of smooth HyCNNs on compact sets (see Theorem~\ref{theorem-hycnn-rp}\ref{theorem-hycnn-rp-smooth-2}). 

Based on the results above and Brenier's theorem, the following Theorem~\ref{theo-hycnn-universality} shows that smooth HyCNNs are distributionally universal in our problem. 
\begin{theorem}[Distributional Universality of Smooth HyCNNs] \label{theo-hycnn-universality}
    Let $\mu,\nu\in\mathscr{P}(\Omega)$ be regarded as probability measures on $\mathbb{R}^d$ supported on $\Omega$. 
    Assume that $\mu$ is absolutely continuous with respect to the $d$-dimensional Lebesgue measure on $\mathbb{R}^d$, and that $\mu$ and $\nu$ have finite second moments. 
    Then there exists a sequence of smooth HyCNNs, denoted by $\{\widehat{\mathsf{HyCNN}}_n^{\tau}\}_{n=1}^{\infty}$, such that $(\nabla\widehat{\mathsf{HyCNN}}_n^{\tau})_{\#} \mu \Rightarrow \nu$, where $\Rightarrow$ denotes weak convergence of probability measures on $\mathbb{R}^d$. 
\end{theorem}
The proof of Theorem~\ref{theo-hycnn-universality} is provided in Appendix~\ref{app-sec-theo-hycnn-universality}.
It demonstrates that smooth HyCNN gradient maps are distributionally universal for approximating target distributions.
For a reference distribution $\mathbb{P}^{\epsilon} \in \mathscr{P}(\Omega)$ and any candidate least-favorable distribution $\mathbb{P}^{\star} \in \mathscr{P}(\Omega)$ in our problem, assume that both $\mathbb{P}^{\epsilon}$ and $\mathbb{P}^{\star}$ have finite second moments, and $\mathbb{P}^{\epsilon}$ is absolutely continuous with respect to the Lebesgue measure on $\mathbb{R}^d$. 
Then, the gradient map of smooth HyCNNs is distributionally universal for approximating any candidate least-favorable distribution. 

\section{Experiments}\label{Sec-results}

In this section, we conduct experiments to assess the effectiveness of the proposed method on datasets of varying training sample sizes and dimensions. 
Specifically, we evaluate its performance on small sample sizes using a synthetic Gaussian mixture dataset (Section~\ref{Sec-results-Gaussian}) and the MNIST dataset (Section~\ref{Sec-results-MNIST}), and on large sample sizes using the Higgs dataset (Section~\ref{Sec-results-Higgs}). 
We also provide a qualitative illustration of our method on a high-dimensional human face generation task (Section~\ref{Sec-results-FFHQ}). 
All experiments were conducted on a laptop with an
NVIDIA GeForce RTX 4060 GPU using Python. 
Compared to the existing methods, our approach can handle hypothesis testing problems across different sample sizes and dimensions, achieving higher accuracy and robustness. 
Codes and experimental results are available at \href{https://github.com/Arthas819/SDRHT}{https://github.com/Arthas819/SDRHT}. 

\subsection{Synthetic Gaussian Mixture Dataset}\label{Sec-results-Gaussian}

\begin{figure}[!htb]
  \centering

  \begin{subfigure}{0.32\linewidth}
    \centering
    \includegraphics[width=\linewidth,height=0.25\textheight,keepaspectratio]{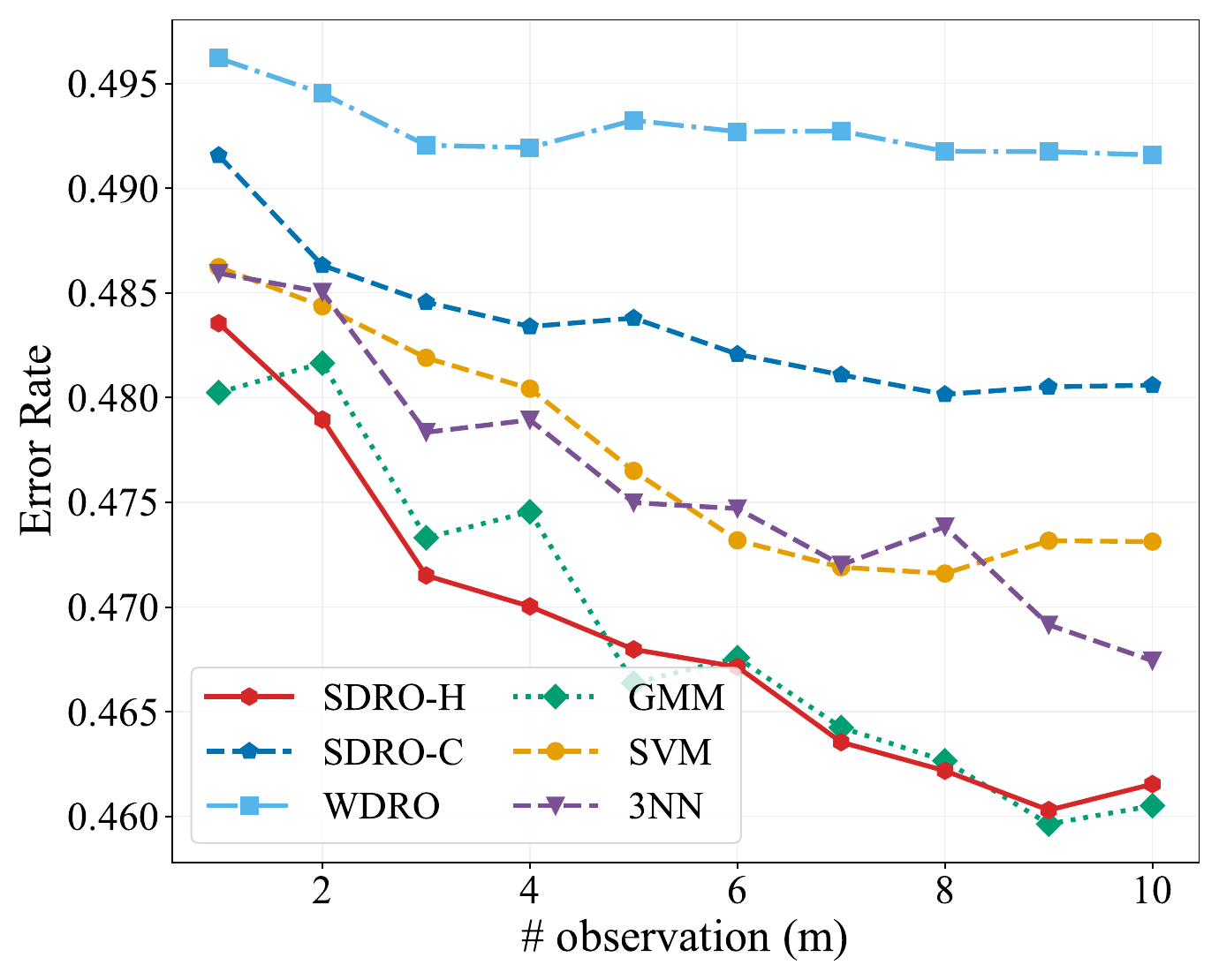}
    \caption{$d=4$}
  \end{subfigure}
  \hfill
  \begin{subfigure}{0.32\linewidth}
    \centering
    \includegraphics[width=\linewidth,height=0.25\textheight,keepaspectratio]{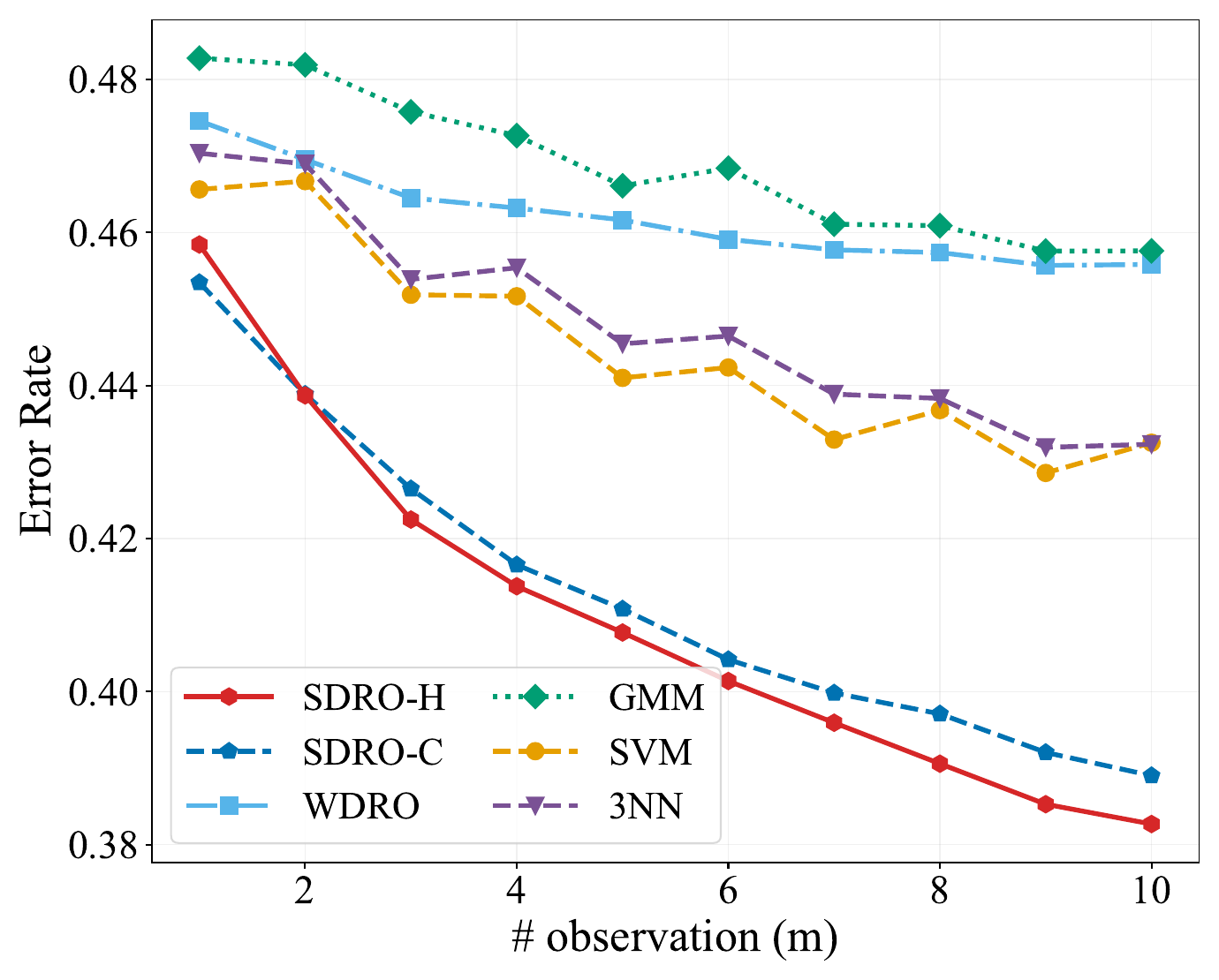}
    \caption{$d=10$}
  \end{subfigure}
  \hfill
  \begin{subfigure}{0.32\linewidth}
    \centering
    \includegraphics[width=\linewidth,height=0.25\textheight,keepaspectratio]{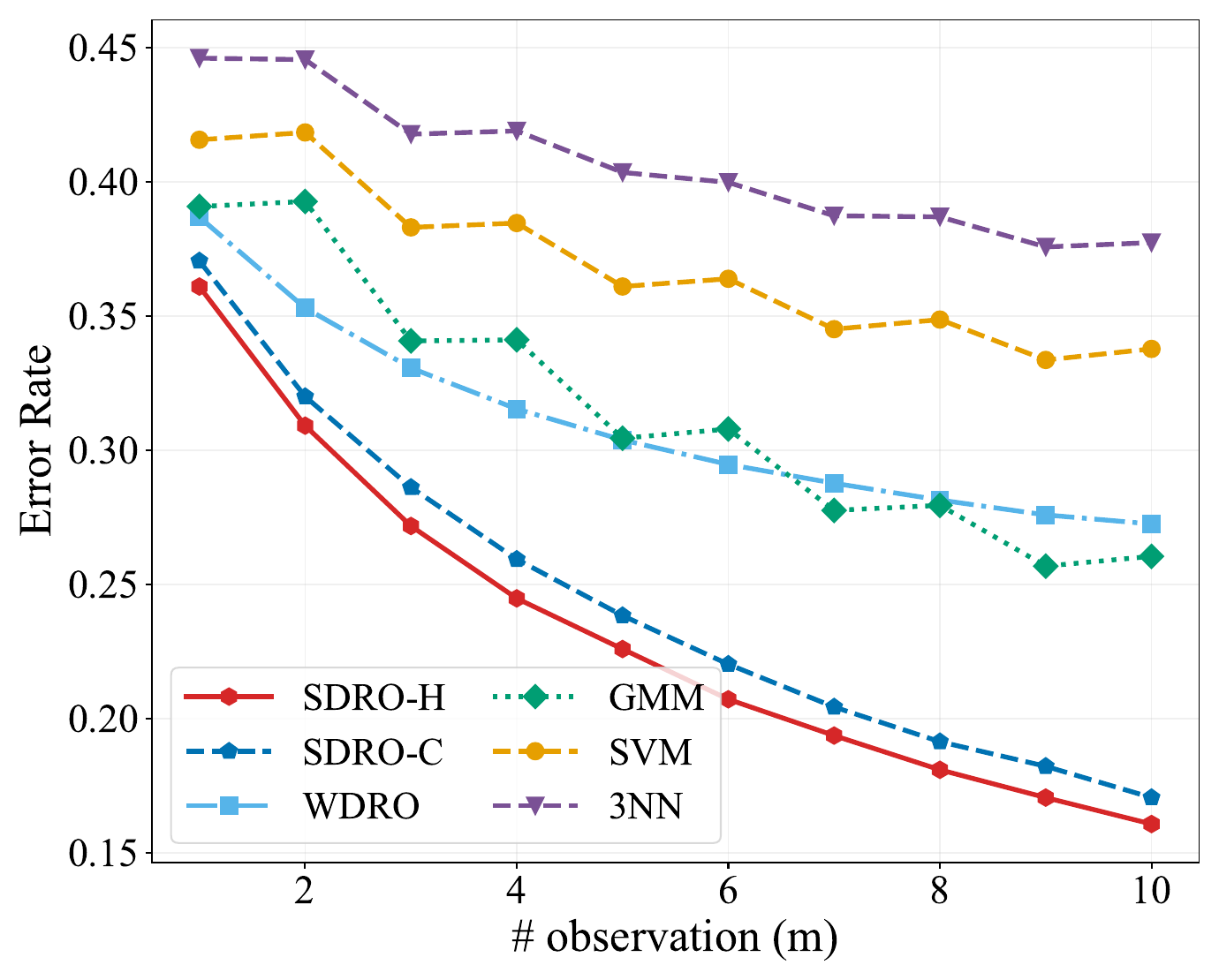}
    \caption{$d=30$}
  \end{subfigure}

  \vspace{0.5em}

  \begin{subfigure}{0.32\linewidth}
    \centering
    \includegraphics[width=\linewidth,height=0.25\textheight,keepaspectratio]{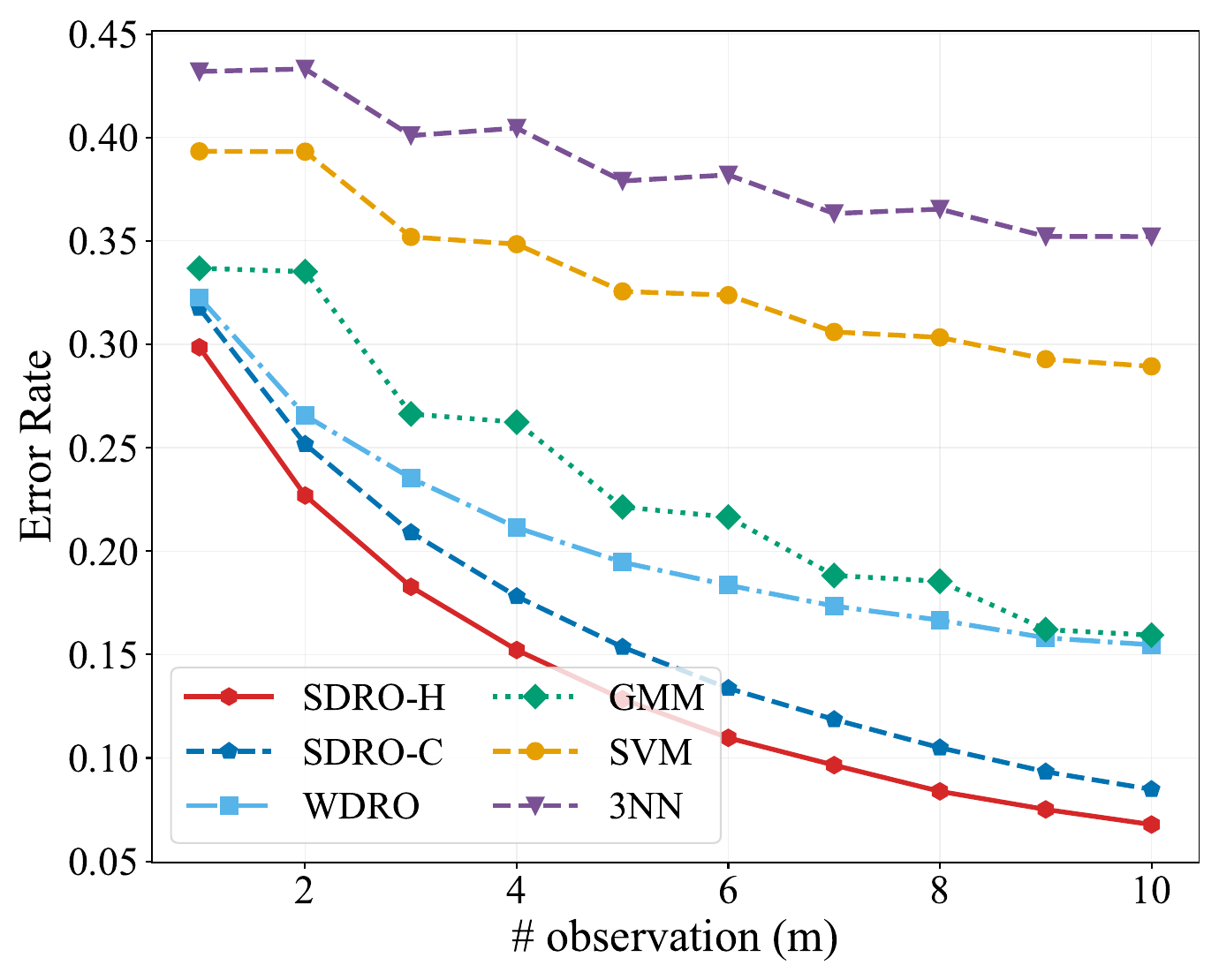}
    \caption{$d=50$}
  \end{subfigure}
  \hfill
  \begin{subfigure}{0.32\linewidth}
    \centering
    \includegraphics[width=\linewidth,height=0.25\textheight,keepaspectratio]{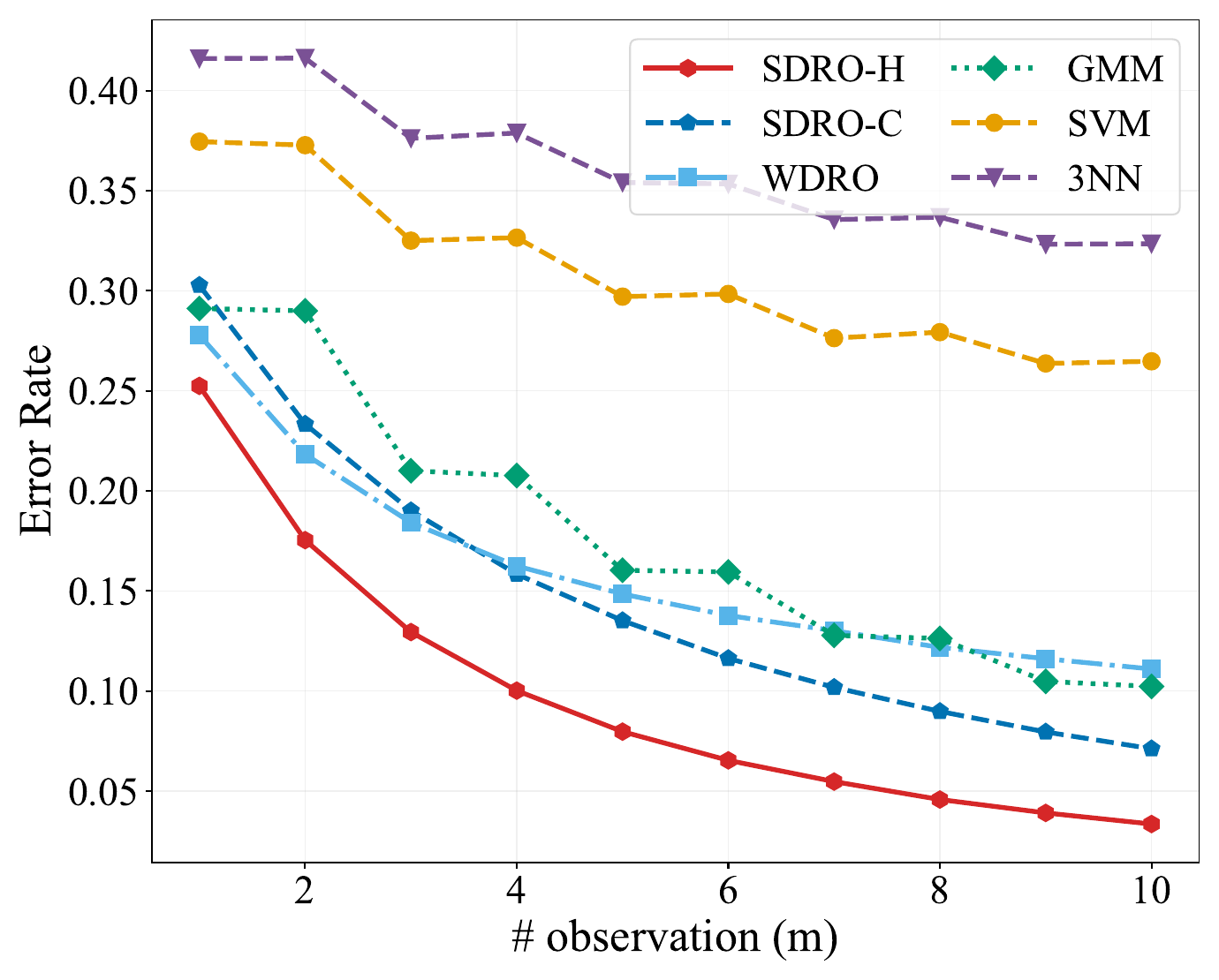}
    \caption{$d=70$}
  \end{subfigure}
  \hfill
  \begin{subfigure}{0.32\linewidth}
    \centering
    \includegraphics[width=\linewidth,height=0.25\textheight,keepaspectratio]{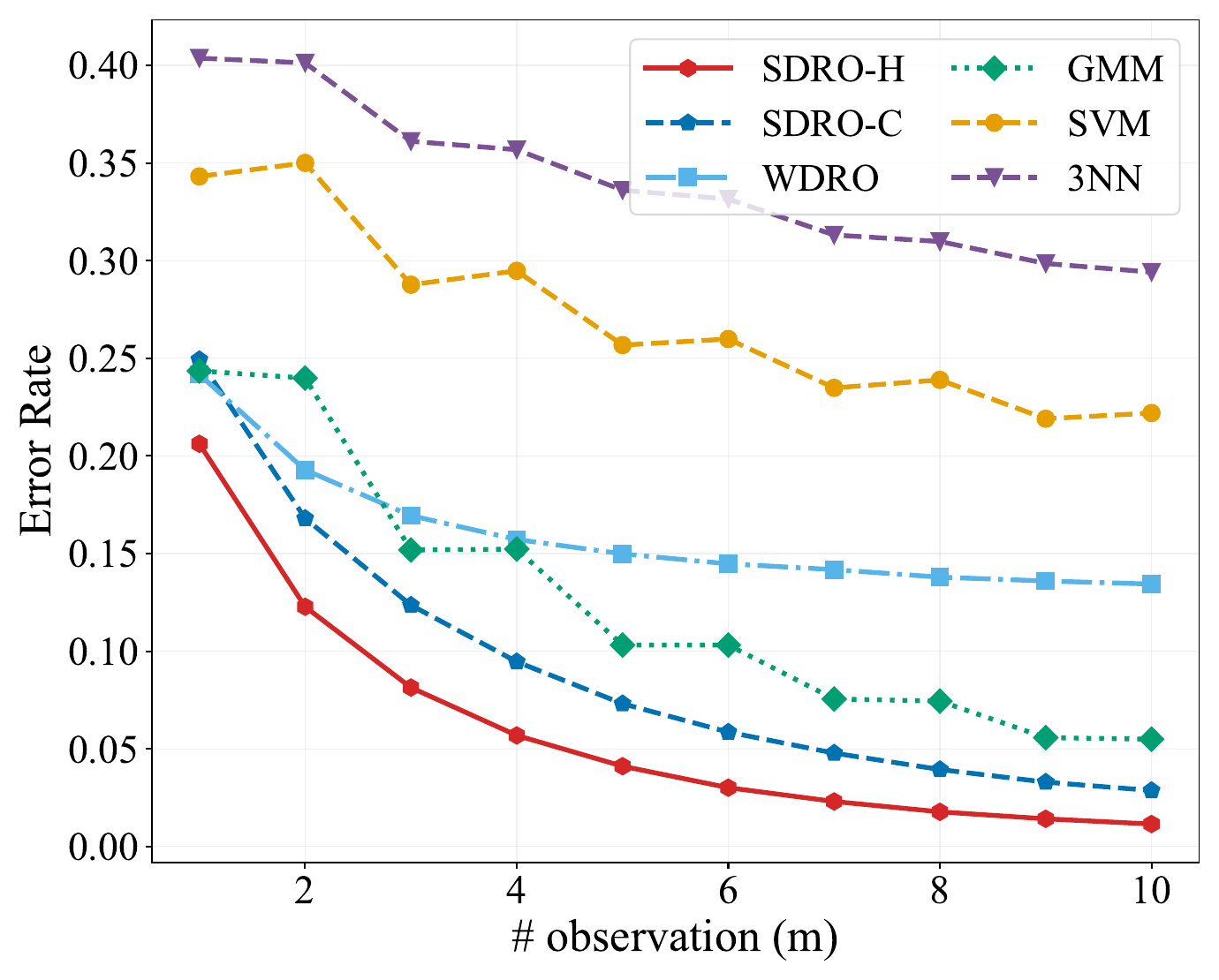}
    \caption{$d=100$}
  \end{subfigure}
  \caption{Error rate comparison on GMM dataset, averaged over 100 trials. \textbf{SDRO-H}: SDRHT method with HyCNN; \textbf{SDRO-C}: SDRHT method with standard ICNN. }
  \label{fig:Gaussian-error-rate}
\end{figure}

In this subsection, we consider a synthetic Gaussian mixture dataset based on~\cite{xie2021robust}. 
Let the samples under the two hypotheses be generated from distributions $0.5\mathcal{N}(0.4\boldsymbol{e}, I_{d}) + 0.5\mathcal{N}(-0.4\boldsymbol{e}, I_{d})$ and $0.5\mathcal{N}(0.4\boldsymbol{f}, I_{d}) + 0.5\mathcal{N}(-0.4\boldsymbol{f}, I_{d})$, respectively. Here $\boldsymbol{e} \in \mathbb{R}^{d}$ is a vector with all entries equal to 1, and $\boldsymbol{f} \in \mathbb{R}^{d}$ is a vector with the first $\left \lceil d/2 \right \rceil $ entries equal to 1 and the remaining entries equal to $-1$. Consider a setting with a small number of training samples $n_0 = n_1 = 10$, and then test on 1000 new samples from each mixture model. 
We compare our method against Wasserstein DRO-based hypothesis testing (WDRO)~\cite{xie2021robust}, Gaussian Mixture Model (GMM), kernel support vector machine (SVM) with radial basis function (RBF) kernel, and a three-layer neural network (3NN) with hidden layer dimensions $(32, 32, 32)$. 
The details of the WDRO method are provided in Appendix~\ref{app-sec-wdro}. 
All hyperparameters, including $\lambda$ and $\epsilon$ in SDRO-H (SDRHT with HyCNN) and SDRO-C (SDRHT with standard ICNN), ambiguity set radii, and kernel bandwidth in WDRO, are determined by cross-validation. 

\begin{table}[!htb]
  \centering
  \caption{Computation time (s) of methods on Gaussian mixture dataset, averaged over 100 trials.}
    \begin{tabular}{rrrrrrr}
    \toprule
    \multicolumn{1}{c}{\multirow{2}[4]{*}{$d$}} & \multicolumn{6}{c}{\textbf{Computation Time (s)}} \\
\cmidrule{2-7}          & \multicolumn{1}{c}{\textbf{SDRO-H}} & \multicolumn{1}{c}{\textbf{SDRO-C}} & \multicolumn{1}{c}{\textbf{WDRO}} & \multicolumn{1}{c}{\textbf{GMM}} & \multicolumn{1}{c}{\textbf{SVM}} & \multicolumn{1}{c}{\textbf{3NN}} \\
    \midrule
    4     & 1.906  & 1.749  & 0.283  & 0.075  & $<$0.000  & 0.052  \\
    10    & 1.507  & 1.802  & 0.275  & 0.002  & 0.001  & 0.040  \\
    30    & 1.292  & 1.900  & 0.271  & 0.002  & $<$0.000  & 0.034  \\
    50    & 1.753  & 1.897  & 0.285  & 0.002  & 0.001  & 0.033  \\
    70    & 1.614  & 1.907  & 0.290  & 0.002  & $<$0.000  & 0.026  \\
    100   & 1.725  & 1.844  & 0.327  & 0.002  & $<$0.000  & 0.028  \\
    \bottomrule
    \end{tabular}%
  \label{tab:Gaussian-com-time}%
\end{table}%

Figure~\ref{fig:Gaussian-error-rate} compares the classification error rate of selected methods over varying dimensions $d \in \{4, 10, 30, 50, 70, 100\}$. 
The horizontal axis in Figure~\ref{fig:Gaussian-error-rate} represents that there are $m$ independent observations in a testing sample, and the decision is made by majority rule~\citep{xie2021robust}.  
As illustrated, HyCNN is more suitable for approximating convex potentials in SDRO than standard ICNN (though the performance of SDRO-C remains competitive in Figures~\ref{fig:Gaussian-error-rate}b-\ref{fig:Gaussian-error-rate}f), and the SDRO-H method achieves the lowest classification error rate in almost all dimensions.  
These results demonstrate the superior performance of our method on hypothesis-testing tasks with small sample sizes compared to selected baselines. 
Table~\ref{tab:Gaussian-com-time} shows the average computation time of all methods. 
Although the training time of SDRO is higher than that of other baselines, it remains under 2s and remains stable as $d$ increases. 

Figure~\ref{fig:Gaussian_CV} reports the accuracy of SDRO-H and WDRO across different hyperparameter combinations when $n_0=n_1=10$ and $d=100$. As illustrated, the accuracy under the best parameter combination of WDRO is inferior to that of SDRO-H. 
The result of SDRO-H indicates that one may improve accuracy by moderately increasing the penalty parameter $\lambda$ and decreasing the entropic regularization parameter $\epsilon$. 
This is because a small $\lambda$ leads to excessive conservatism, and a large $\epsilon$ induces stronger smoothing and may dilute the correlation between samples and classification labels. 
The result of WDRO demonstrates that under an appropriate bandwidth selection, the classification accuracy is not sensitive to $\lambda$. 

\begin{figure}[!htb]
  \centering

  \begin{subfigure}{0.49\linewidth}
    \centering
    \includegraphics[width=\linewidth,height=\textheight,keepaspectratio]{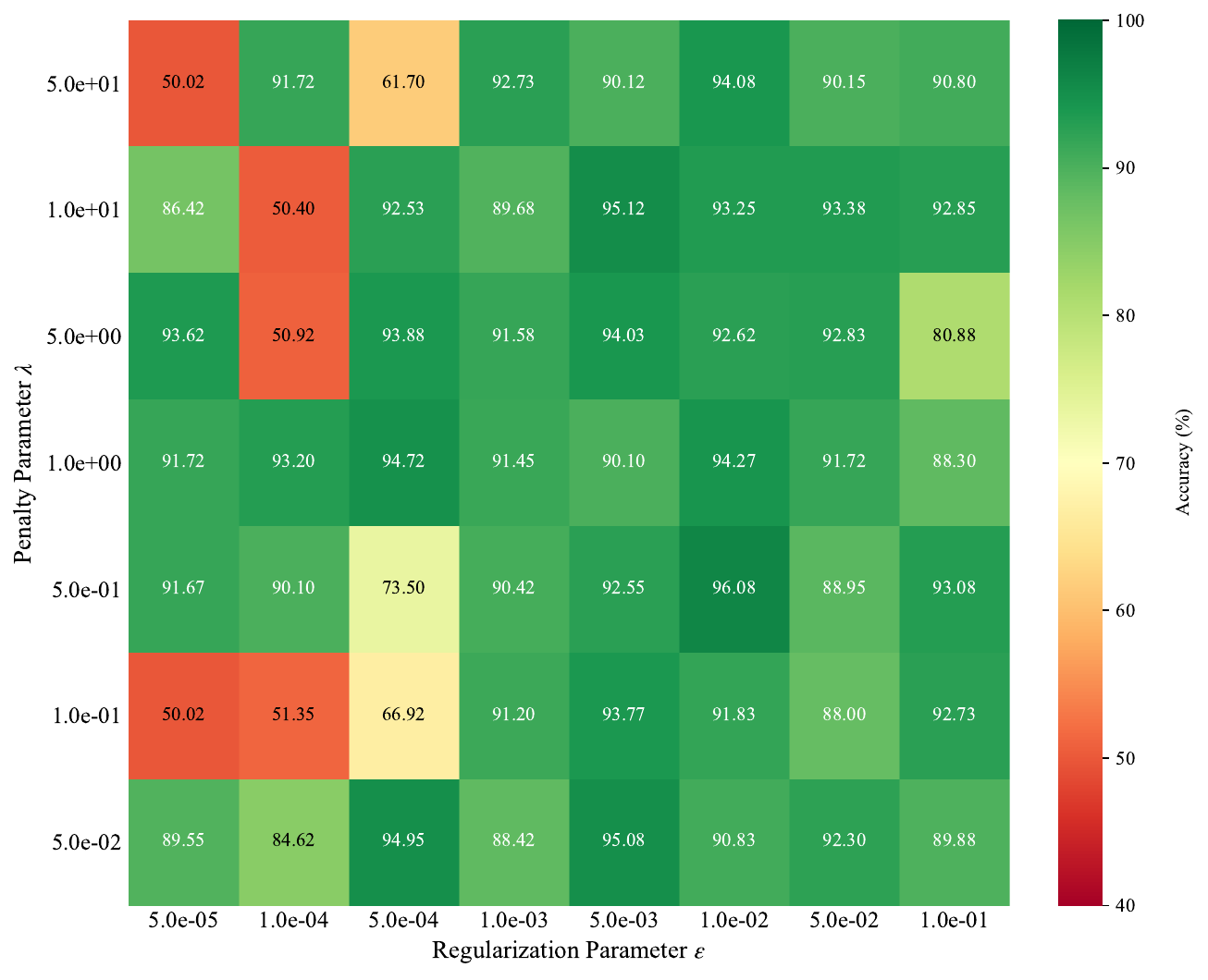}
    \caption{SDRO-H}
  \end{subfigure}
  \begin{subfigure}{0.49\linewidth}
    \centering
    \includegraphics[width=\linewidth,height=\textheight,keepaspectratio]{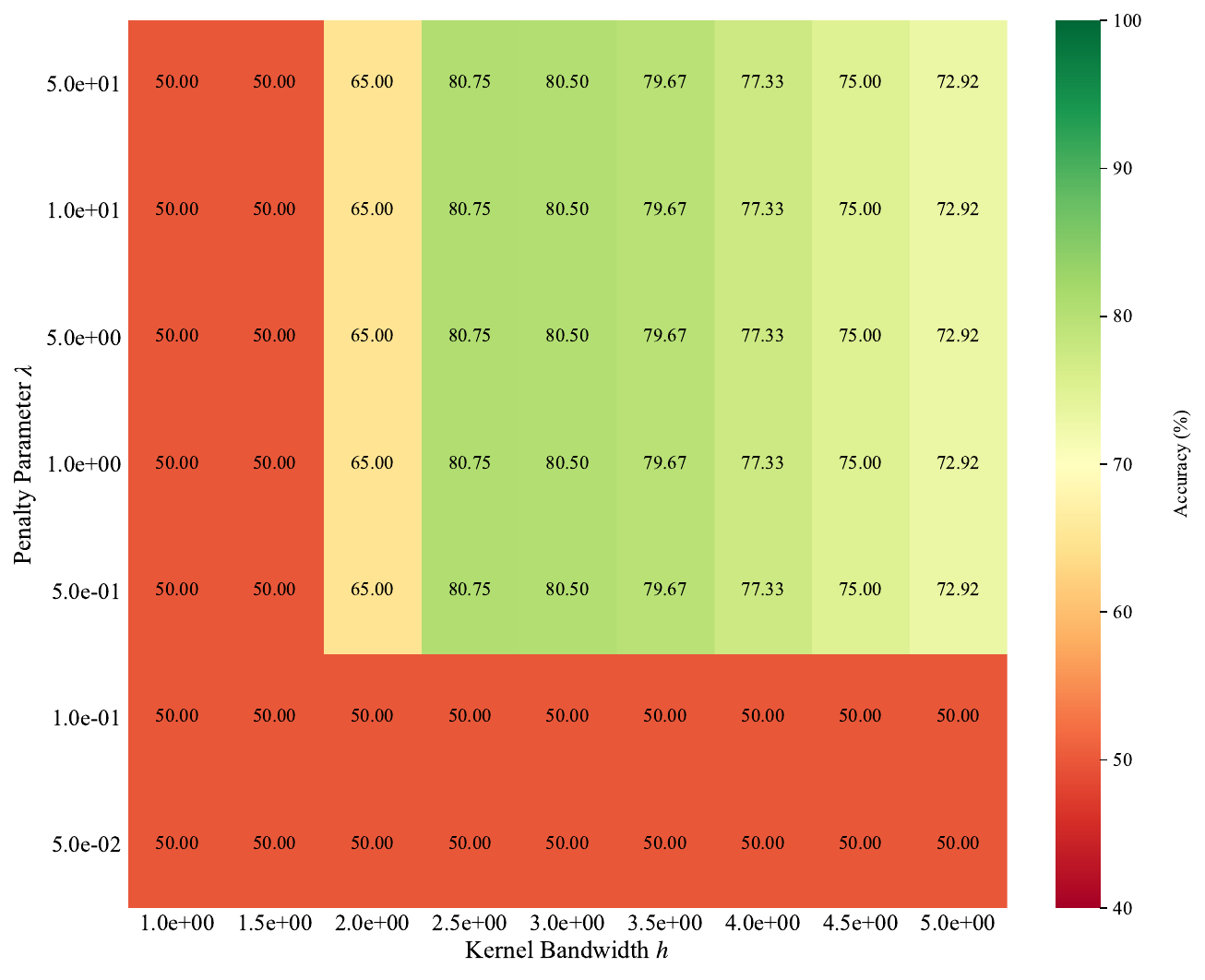}
    \caption{WDRO}
  \end{subfigure} 
  \caption{Average accuracy (\%) of SDRO-H and WDRO on the GMM dataset with different hyperparameters ($n_0=n_1=10, d=100$). }
  \label{fig:Gaussian_CV}
\end{figure}

\begin{table}[!htb]
  \centering
  \caption{Robustness comparison on the GMM dataset after PGD attack, averaged over $m\in [10]$ and 5 trials (accuracy \%, $n_0=n_1=10, d=50$). Bold represents the highest accuracy in each column. }
    \begin{tabular}{rrrrrrrrr}
    \toprule
    \multirow{2}[4]{*}{\textbf{Method}} & \multicolumn{4}{c}{\textbf{PGD-$L_2$ Attack}} & \multicolumn{4}{c}{\textbf{PGD-$L_\infty$ Attack}} \\
\cmidrule(lr){2-5} \cmidrule(lr){6-9}          & \multicolumn{1}{c}{\textbf{$\Delta$ = 0.025}} & \multicolumn{1}{c}{\textbf{$\Delta$ = 0.05}} & \multicolumn{1}{c}{\textbf{$\Delta$ = 0.075}} & \multicolumn{1}{c}{\textbf{$\Delta$ = 0.1}} & \multicolumn{1}{c}{\textbf{$\Delta$ = 0.05}} & \multicolumn{1}{c}{\textbf{$\Delta$ = 0.1}} & \multicolumn{1}{c}{\textbf{$\Delta$ = 0.15}} & \multicolumn{1}{c}{\textbf{$\Delta$ = 0.2}} \\
    \midrule
    \textbf{SDRO-H} & \textbf{85.801} & \textbf{85.769} & \textbf{85.736} & \textbf{85.737} & \textbf{85.467} & \textbf{85.165} & \textbf{84.906} & \textbf{84.811} \\
    \textbf{SDRO-C} & 82.333 & 82.347 & 82.332 & 82.309 & 82.175 & 82.018 & 81.867 & 81.774 \\
    \textbf{WDRO} & 69.693 & 68.902 & 68.108 & 67.924 & 62.440 & 54.711 & 48.096 & 44.256 \\
    \textbf{GMM} & 73.401 & 72.392 & 71.367 & 71.121 & 63.100  & 51.579 & 41.804 & 36.414 \\
    \textbf{SVM} & 64.477 & 63.628 & 62.774 & 62.561 & 55.957 & 47.171 & 40.166 & 36.166 \\
    \textbf{3NN} & 58.187 & 57.021 & 55.906 & 55.638 & 47.635 & 37.184 & 29.483 & 25.509 \\
    \bottomrule
    \end{tabular}%
  \label{tab:attack_Gaussian}%
\end{table}%

To evaluate the adversarial robustness of our model on adversarial examples, we apply adversarial perturbations to the test data using the projected gradient descent (PGD) attack~\cite{madry2017towards, xu2024flow}. 
Specifically, for a testing sample $\boldsymbol{\omega} \sim \mathbb{P}_k$, the attacked sample for fixed detector $\phi_{\text{opt}}$, defined by a point-wise perturbation problem, is given by 
\begin{equation} \label{Eq-attack}
    \boldsymbol{\omega}_{\text{attacked}} := \boldsymbol{\omega} + \boldsymbol{\delta}_{\boldsymbol{\omega}}^{*}, \quad \boldsymbol{\delta}_{\boldsymbol{\omega}}^{*} := \arg \max_{\| \boldsymbol{\delta}_{\boldsymbol{\omega}} \|_p \le \Delta} \quad L_k(\phi_{\text{opt}}, \, \boldsymbol{\omega} + \boldsymbol{\delta}_{\boldsymbol{\omega}}). 
\end{equation}
In the following, we assess robustness using $L_2$ and $L_{\infty}$ norm constraints, i.e., $p=2$ and $p\to \infty$ in Eq.~\eqref{Eq-attack}. 
In Table~\ref{tab:attack_Gaussian}, we compare the classification accuracy of all methods across different levels of $\Delta$. 
As shown in Table~\ref{tab:attack_Gaussian}, the accuracies of the proposed methods remain higher and more stable than those of the baselines as the attack intensity $\Delta$ increases under both $L_2$ and $L_\infty$ constraints, and SDRO-H consistently yields the best test accuracy across all values of $\Delta$. 

\subsection{MNIST Handwritten Digits Classification}\label{Sec-results-MNIST}

This subsection considers the MNIST handwritten digits dataset ($d$=784). 
In each trial, we randomly select 5 digits as class 0, and the remaining as class 1. 
We select $n^{\prime}$ samples for each digit of each class to form the two training sets, and therefore $n_0 = n_1 = 5n^{\prime}$. 
The testing data includes 1000 images from both classes. 
For this dataset, we further include logistic regression (LR) and Flow-based DRO (FDRO)~\cite{xu2024flow} as baselines. 
The FDRO method, which characterizes least-favorable distributions from a generative perspective, is suitable for high-dimensional robust hypothesis testing and requires training convolutional autoencoders. 
In the following, the label `SDRO' represents our SDRHT method with HyCNN. 
All hyperparameters are determined by cross-validation. 

\begin{figure}[!htb]
  \centering

  \begin{subfigure}{0.32\linewidth}
    \centering
    \includegraphics[width=\linewidth,height=0.25\textheight,keepaspectratio]{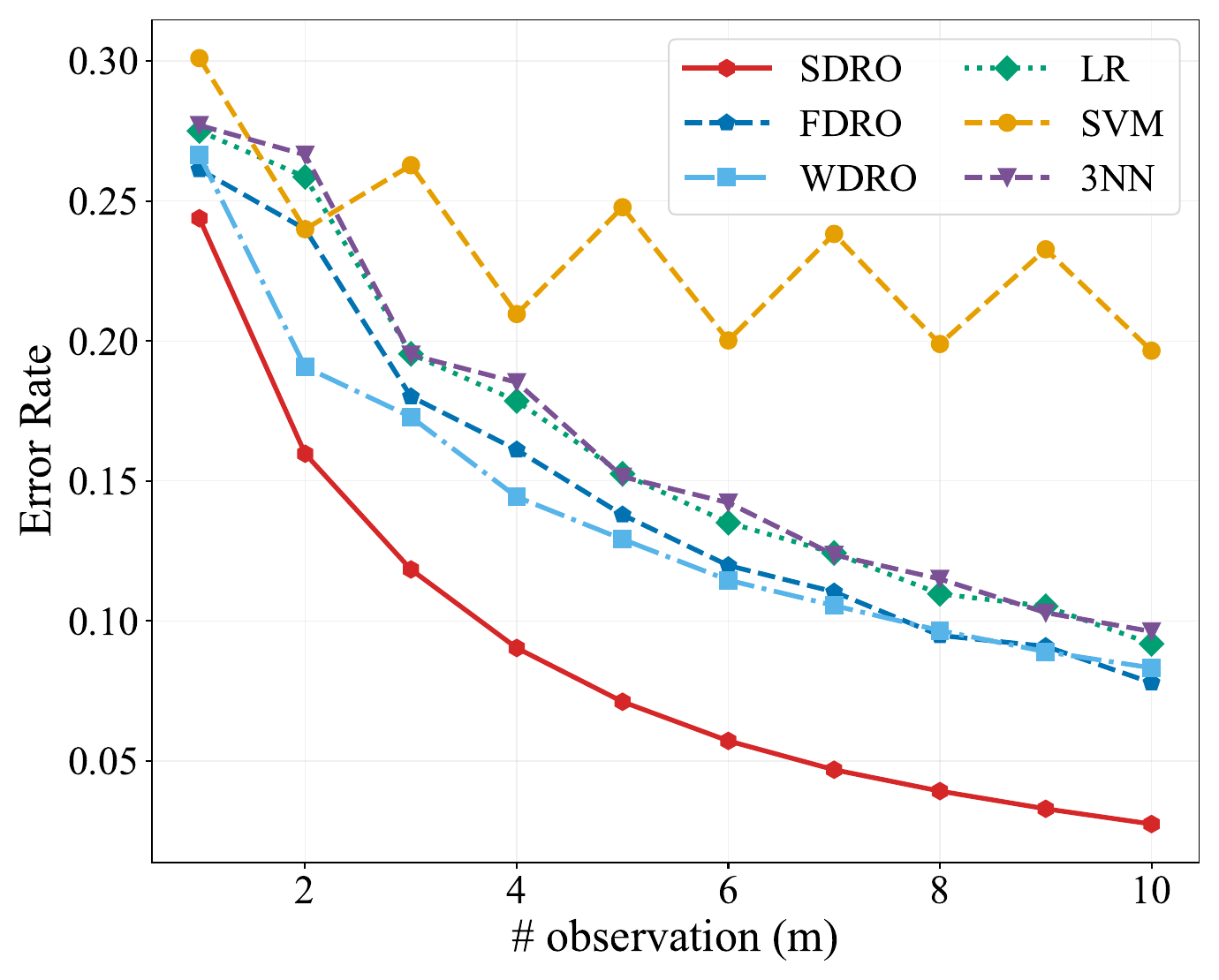}
    \caption{$n^{\prime}=2, n_0 = n_1=10$}
  \end{subfigure}
  \hfill
  \begin{subfigure}{0.32\linewidth}
    \centering
    \includegraphics[width=\linewidth,height=0.25\textheight,keepaspectratio]{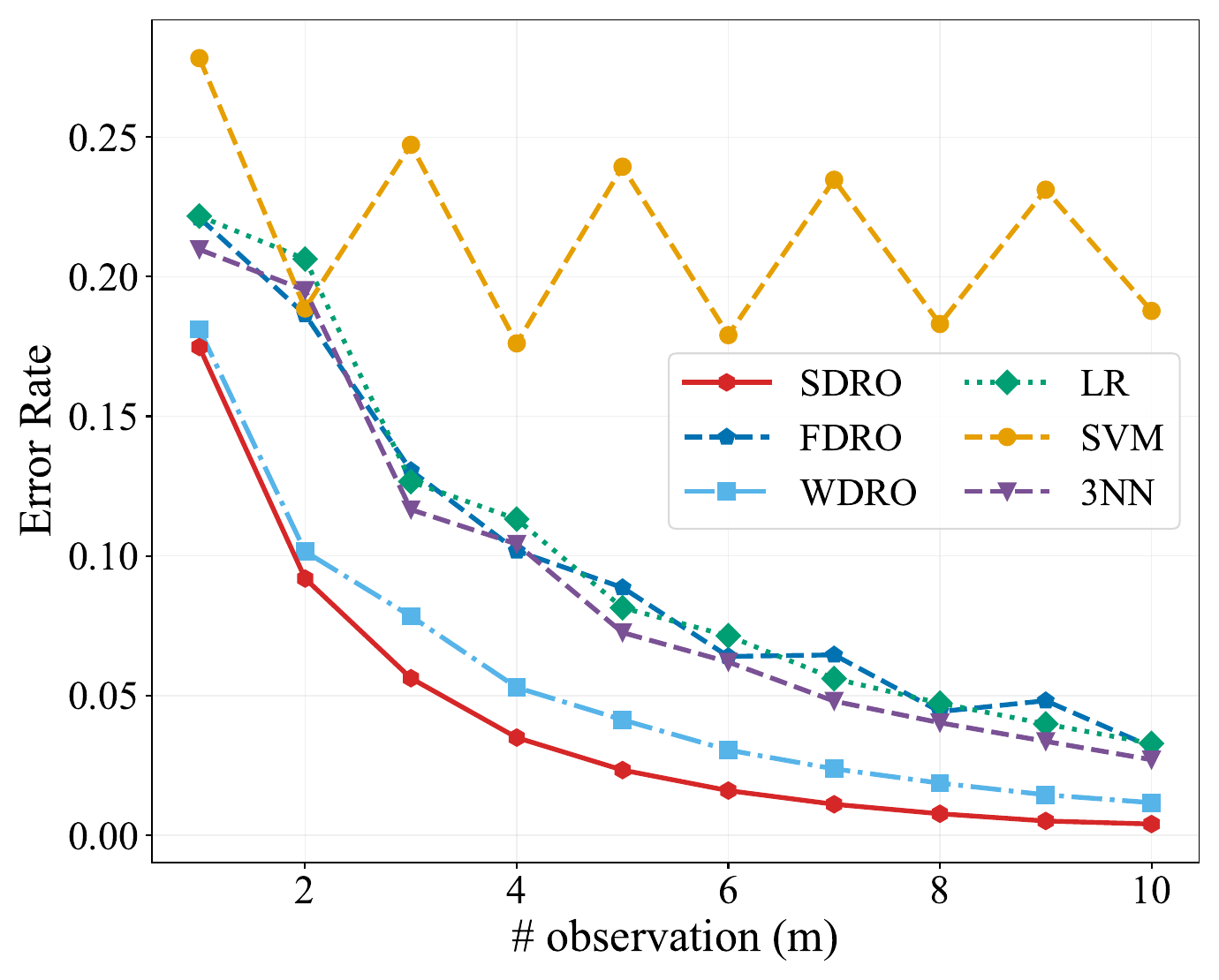}
    \caption{$n^{\prime}=5, n_0 = n_1=25$}
  \end{subfigure}
  \hfill
  \begin{subfigure}{0.32\linewidth}
    \centering
    \includegraphics[width=\linewidth,height=0.25\textheight,keepaspectratio]{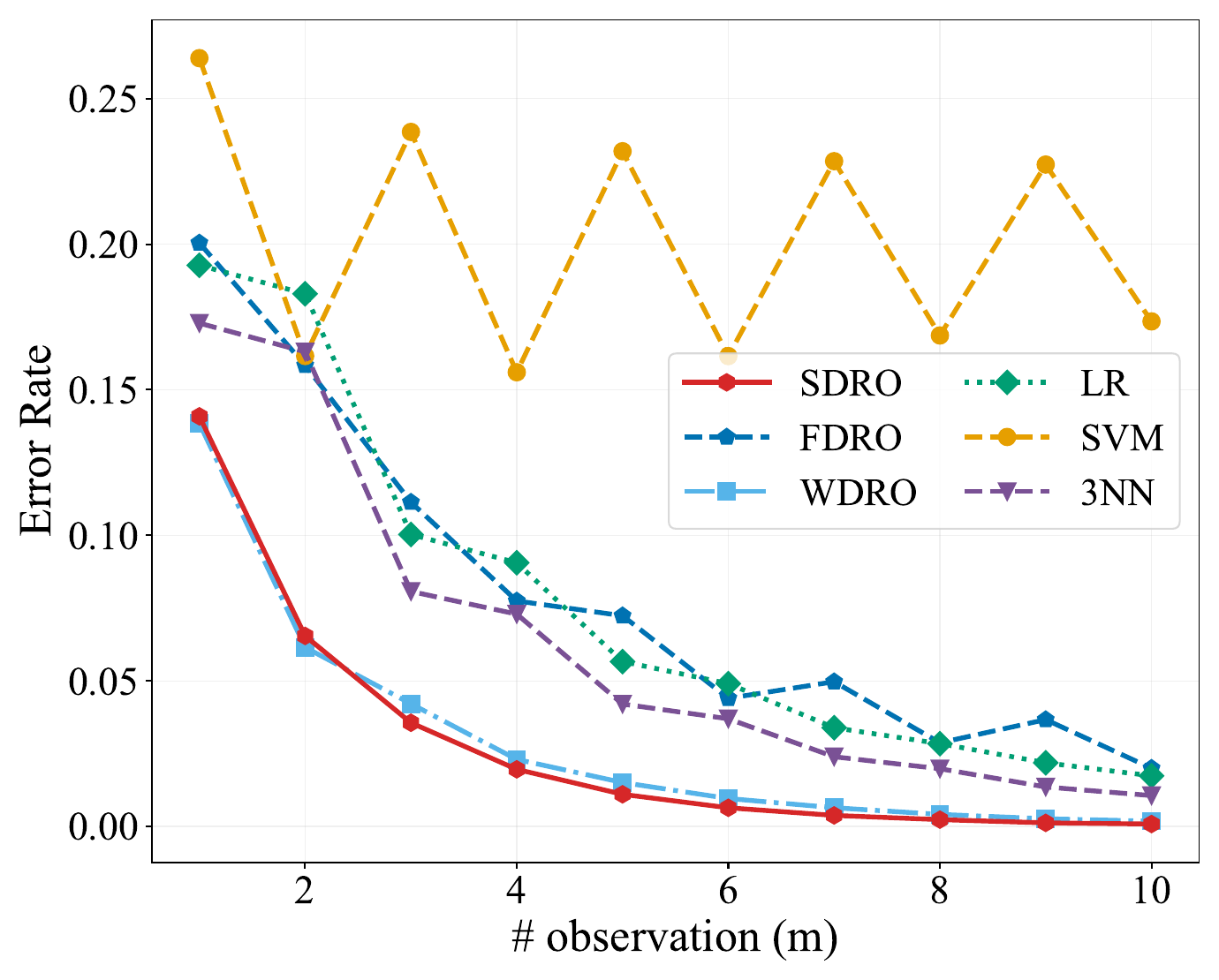}
    \caption{$n^{\prime}=10, n_0 = n_1=50$}
  \end{subfigure}
  \caption{Error rate comparison on the MNIST dataset, averaged over 100 trials. }
  \label{fig:MNIST-error-rate}
\end{figure}

Figure~\ref{fig:MNIST-error-rate} illustrates that the proposed method achieves a lower error rate than existing well-tuned DRO methods and common classifiers across different training set sizes. 
Table~\ref{tab:MNIST-com-time} shows the average computation time of all methods, illustrating that the training time of SDRO is acceptable and stable as $n^{\prime}$ increases. 
To evaluate robustness, we compare all methods under both $L_2$ and $L_\infty$ PGD attacks (see Eq.~\eqref{Eq-attack}) in Table~\ref{tab:attack_MNIST}.  
As shown, SDRO achieves the highest accuracy across all non-zero attack intensities under both norms, indicating the superior robustness of our method. 
Under the PGD-$L_\infty$ attack, the three DRO-based models demonstrate better robustness than other baselines. 

\begin{table}[!htb]
  \centering
  \caption{Computation time (s) of methods on the MNIST dataset, averaged over 100 trials. }
    \begin{tabular}{rrrrrrr}
    \toprule
    \multicolumn{1}{c}{\multirow{2}[0]{*}{$\boldsymbol{n}^{\prime}$}} & \multicolumn{6}{c}{\textbf{Computation Time (s)}} \\
     \cmidrule{2-7}
     & \multicolumn{1}{c}{\textbf{SDRO}} &  \multicolumn{1}{c}{\textbf{FDRO}} & \multicolumn{1}{c}{\textbf{WDRO}} & \multicolumn{1}{c}{\textbf{LR}} & \multicolumn{1}{c}{\textbf{SVM}} & \multicolumn{1}{c}{\textbf{3NN}} \\
    \midrule 
    2     & 1.943  & 4.821  & 0.022  & 0.008  & $<$0.001  & 0.167  \\
    5     & 1.477  & 8.999  & 0.045  & 0.011  & 0.001  & 0.231  \\
    10    & 1.465  & 17.153  & 0.135  & 0.018  & 0.001  & 0.269  \\
    \bottomrule
    \end{tabular}%
  \label{tab:MNIST-com-time}%
\end{table}%

\begin{table}[!htb]
  \centering
  \caption{Robustness comparison on MNIST dataset after PGD attack, averaged over $m\in [10]$ and 5 trials (accuracy \%, $n_0=n_1=25$).}
    \begin{tabular}{rrrrrrrrr}
    \toprule
    \multirow{2}[4]{*}{\textbf{Method}} & \multicolumn{4}{c}{\textbf{PGD-$L_2$ Attack}} & \multicolumn{4}{c}{\textbf{PGD-$L_\infty$ Attack}} \\
\cmidrule(lr){2-5} \cmidrule(lr){6-9}         & \multicolumn{1}{c}{\textbf{$\Delta$ = 0.025}} & \multicolumn{1}{c}{\textbf{$\Delta$ = 0.05}} & \multicolumn{1}{c}{\textbf{$\Delta$ = 0.075}} & \multicolumn{1}{c}{\textbf{$\Delta$ = 0.1}} & \multicolumn{1}{c}{\textbf{$\Delta$ = 0.05}} & \multicolumn{1}{c}{\textbf{$\Delta$ = 0.1}} & \multicolumn{1}{c}{\textbf{$\Delta$ = 0.15}} & \multicolumn{1}{c}{\textbf{$\Delta$ = 0.2}} \\
    \midrule
    \textbf{SDRO} & \textbf{95.75 } & \textbf{95.70 } & \textbf{95.67 } & \textbf{95.67 } & \textbf{93.53 } & \textbf{90.60 } & \textbf{87.46 } & \textbf{85.47 } \\
    \textbf{FDRO} & 91.04  & 90.93  & 90.87  & 90.87  & 86.51  & 80.92  & 75.28  & 71.91  \\
    \textbf{WDRO} & 94.54  & 94.47  & 94.42  & 94.42  & 91.18  & 85.63  & 79.42  & 75.56  \\
    \textbf{LR} & 91.15  & 90.98  & 90.89  & 90.89  & 84.79  & 75.80  & 67.27  & 62.34  \\
    \textbf{SVM} & 76.67  & 76.50  & 76.44  & 76.44  & 72.65  & 69.02  & 66.21  & 64.73  \\
    \textbf{3NN} & 91.39  & 91.21  & 91.11  & 91.11  & 82.10  & 68.34  & 55.00  & 48.09  \\
    \bottomrule
    \end{tabular}%
  \label{tab:attack_MNIST}%
\end{table}%

We next visualize how the trained convex functions induce adversarial distributional attacks on their opposite classes. 
In each trial, we select two digits that are easily confused in handwritten form and place them in separate classes. 
To generate a more intuitive and clear trajectory, in the training phase, we first train a variational autoencoder, including two parts, $\mathsf{Encoder}$ and $\mathsf{Decoder}$, to embed all images into a 32-dimensional latent space. 
Then, the convex potentials are trained on this latent space with a small training sample size. 
In the generation phase, we randomly select an image $\boldsymbol{z}_k$ from each class $k \in \mathbb{K}$, encode it into the latent space, denoted by $\widetilde{\boldsymbol{z}}_k:= \mathsf{Encoder}(\boldsymbol{z}_k)$, and then map $\widetilde{\boldsymbol{z}}_k$ by the gradient of the convex potential of the other class, i.e., $\nabla \widetilde{\varphi}_{|1-k|}(\widetilde{\boldsymbol{z}}_k)$, and finally draw the attacked image $\mathsf{Decoder} \Big( \nabla \widetilde{\varphi}_{|1-k|}(\widetilde{\boldsymbol{z}}_k) \Big)$. 
Figure~\ref{fig:MNIST-attack} illustrates the trajectory of injecting adversarial attacks on digits in both classes in 7 trials. 
As illustrated, our adversarial attack naturally achieves conversions of different classes of numbers through convex gradient mapping in a few steps.  

\begin{figure}[!htb]
  \centering
  \begin{subfigure}{0.48\linewidth}
    \centering
    \includegraphics[width=\linewidth,height=\textheight,keepaspectratio]{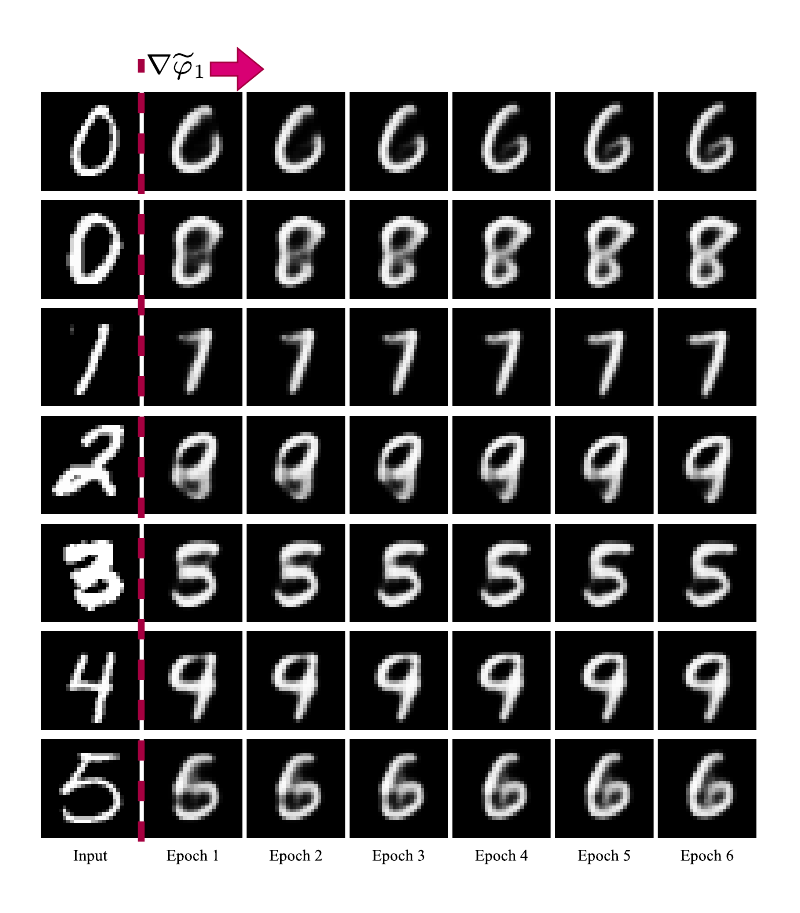}
    \caption{Adversarial attack for Class 0 by $\nabla \widetilde{\varphi}_{1}$}
  \end{subfigure}
  \begin{subfigure}{0.48\linewidth}
    \centering
    \includegraphics[width=\linewidth,height=\textheight,keepaspectratio]{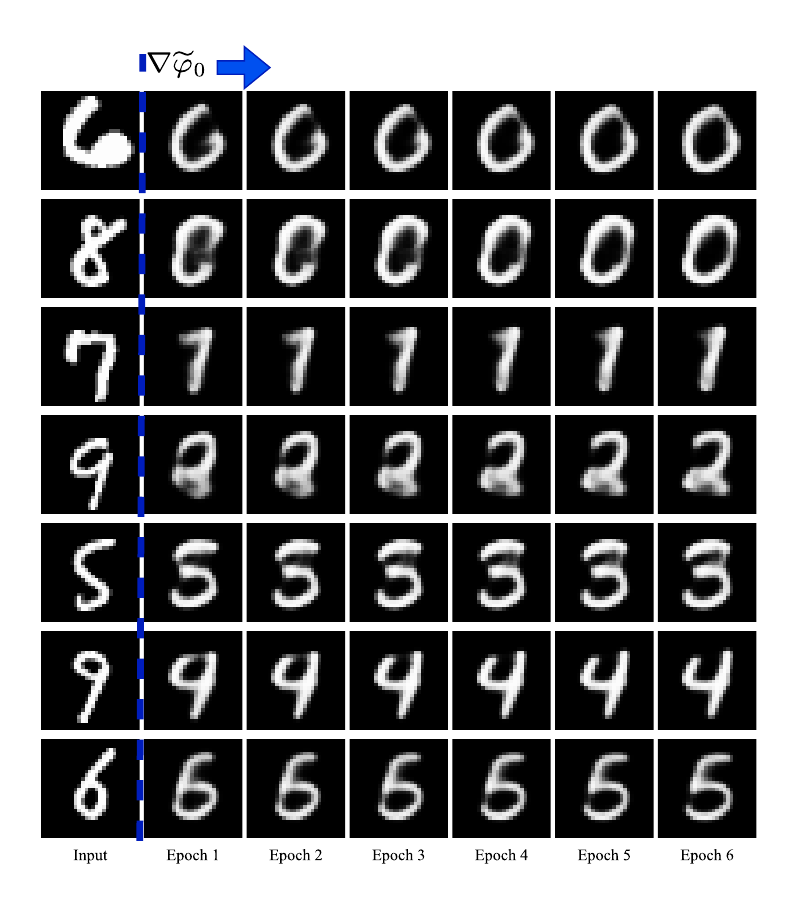}
    \caption{Adversarial attack for Class 1 by $\nabla \widetilde{\varphi}_{0}$}
  \end{subfigure} 
  \caption{Adversarial attacks on digits by trained convex gradients. The first column displays the original images, while the following six columns show the attacked images from Epochs 1 to 6. }
  \label{fig:MNIST-attack}
\end{figure}

\subsection{High-energy Physics Higgs Dataset} \label{Sec-results-Higgs}

In this subsection, we evaluate our method on the Higgs dataset, which distinguishes between a signal process where new theoretical Higgs bosons (HIGGS) are produced, and a background process with the identical decay products but distinct kinematics~\cite {baldi2014searching}. 
We consider 21 original features and large training sample sizes. 

Table~\ref{tab:higgs-accuracy} reports the accuracy and computation time of our method against other baselines over varying training sample sizes. 
As shown, the SDRO method achieves the highest accuracy in a relatively low computation time. 
Note that Table~\ref{tab:higgs-accuracy} does not include the result of WDRO, since when the sample size is large, WDRO did not produce an acceptable solution within the 1,800-second time limit. 
This is because WDRO requires solving a linear program with $\mathcal{O}(n_0 n_1)$ decision variables (see Appendix~\ref{app-sec-wdro}) and becomes computationally prohibitive for large sample sizes. 

\begin{table}[!htb]
  \centering
  \caption{Accuracy ($\%$) and computation time (s) of methods on the Higgs dataset, averaged over observations $m \in [10]$ and 10 trials. Bold represents the highest accuracy in each row.}
    \begin{tabular}{c|rrrrr|rrrrr}
    \toprule
    $n_0=n_1$ & \multicolumn{5}{c|}{\textbf{Accuracy (\%)}} & \multicolumn{5}{c}{\textbf{Computation Time (s)}} \\
    \cmidrule(lr){2-6} \cmidrule(lr){7-11}   
    ($\times 10^3$)  & \multicolumn{1}{c}{\textbf{SDRO}} & \multicolumn{1}{c}{\textbf{FDRO}} & \multicolumn{1}{c}{\textbf{LR}} & \multicolumn{1}{c}{\textbf{SVM}} & \multicolumn{1}{c|}{\textbf{3NN}} & \multicolumn{1}{c}{\textbf{SDRO}} & \multicolumn{1}{c}{\textbf{FDRO}} & \multicolumn{1}{c}{\textbf{LR}} & \multicolumn{1}{c}{\textbf{SVM}} & \multicolumn{1}{c}{\textbf{3NN}} \\
    \midrule
    \textbf{1} & \textbf{68.09 } & 62.07  & 59.73  & 64.13  & 57.92  & 1.11  & 2.09  & 0.02  & 0.18  & 5.09  \\
    \textbf{2} & \textbf{71.32 } & 65.56  & 60.45  & 66.67  & 58.70  & 1.06  & 2.06  & 0.01  & 0.58  & 11.06  \\
    \textbf{3} & \textbf{72.83 } & 67.20  & 60.68  & 67.87  & 58.97  & 1.13  & 2.69  & 0.01  & 0.74  & 9.86  \\
    \textbf{4} & \textbf{73.88 } & 68.03  & 60.95  & 68.88  & 60.17  & 1.64  & 3.65  & 0.02  & 1.51  & 11.21  \\
    \textbf{5} & \textbf{74.62 } & 68.77  & 61.01  & 69.56  & 61.00  & 1.79  & 3.56  & 0.02  & 2.20  & 11.16  \\
    \textbf{6} & \textbf{75.37 } & 69.15  & 61.07  & 70.27  & 61.54  & 1.87  & 3.30  & 0.02  & 3.01  & 13.74  \\
    \textbf{7} & \textbf{75.71 } & 69.21  & 61.09  & 70.67  & 62.86  & 2.37  & 3.35  & 0.03  & 4.46  & 15.90  \\
    \textbf{8} & \textbf{76.25 } & 69.60  & 61.22  & 71.13  & 63.66  & 2.47  & 3.30  & 0.03  & 5.66  & 16.71  \\
    \textbf{9} & \textbf{76.26 } & 69.71  & 61.18  & 71.14  & 63.85  & 2.54  & 3.73  & 0.02  & 7.33  & 18.53  \\
    \textbf{10} & \textbf{76.66 } & 69.91  & 61.29  & 71.59  & 64.89  & 2.96  & 3.52 & 0.03  & 8.65  & 21.41  \\
    \bottomrule
    \end{tabular} 
    \begin{minipage}{0.95\linewidth}
        \small Note: The WDRO method did not produce an acceptable solution within the 1,800-second time limit. 
    \end{minipage}
  \label{tab:higgs-accuracy}%
\end{table}%

In Table~\ref{tab:higgs-robustness}, we show that SDRO remains the highest accuracy over 70\% across all non-zero attack intensities under both norms. 
This result demonstrates that our method can still maintain its accuracy even under significant adversarial perturbations compared to other baselines.

\begin{table}[!htb]
  \centering
  \caption{Robustness comparison on Higgs dataset after PGD attack, averaged over $m\in [10]$ and 5 trials (accuracy \%, $n_0=n_1=2,000$).}
    \begin{tabular}{rrrrrrrrr}
    \toprule
    \multirow{2}[4]{*}{\textbf{Method}} & \multicolumn{4}{c}{\textbf{\textbf{PGD-$L_2$ Attack}}} & \multicolumn{4}{c}{\textbf{PGD-$L_\infty$ Attack}} \\
    \cmidrule(lr){2-5} \cmidrule(lr){6-9} 
    & \multicolumn{1}{c}{\textbf{$\Delta$ = 0.025}} & \multicolumn{1}{c}{\textbf{$\Delta$ = 0.05}} & \multicolumn{1}{c}{\textbf{$\Delta$ = 0.075}} & \multicolumn{1}{c}{\textbf{$\Delta$ = 0.1}} & \multicolumn{1}{c}{\textbf{$\Delta$ = 0.025}} & \multicolumn{1}{c}{\textbf{$\Delta$ = 0.05}} & \multicolumn{1}{c}{\textbf{$\Delta$ = 0.075}} & \multicolumn{1}{c}{\textbf{$\Delta$ = 0.1}} \\
    \midrule
    \textbf{SDRO} & \textbf{72.342}  &\textbf{71.998}  & \textbf{71.887}  & \textbf{71.891}  & \textbf{73.352}  & \textbf{72.758}  & \textbf{72.436}  & \textbf{72.284}  \\
    \textbf{FDRO} & 66.999  & 64.888  & 62.736  & 60.835  & 62.780  & 55.396  & 48.914  & 44.461  \\
    \textbf{LR} & 62.945  & 60.913  & 59.193  & 57.582  & 59.708  & 54.046  & 48.528  & 44.412  \\
    \textbf{SVM} & 67.669  & 65.849  & 64.191  & 62.716  & 64.524  & 58.350  & 53.250  & 49.614  \\
    \textbf{3NN} & 58.204  & 54.669  & 50.782  & 48.015  & 50.264  & 38.658  & 30.418  & 26.266  \\
    \bottomrule
    \end{tabular}%
  \label{tab:higgs-robustness}%
\end{table}%

\subsection{Human Face Generation} \label{Sec-results-FFHQ}

In this subsection, similar to~\cite{xu2025gradient}, we present a qualitative experiment on the human face dataset FFHQ~\cite{karras2019style} that illustrates how our method generates adversarial-attacked high-dimensional images. 

We take the gender attribute of the FFHQ dataset as the binary label in the following image generation experiment. 
We labeled Class 0 as \textit{Female} and Class 1 as \textit{Male}. 
In the training phase, we first map the original $1024 \times 1024$ images to a 512-dimensional latent space by a pre-trained adversarial latent autoencoder~\cite{pidhorskyi2020adversarial}. 
Then, we train two HyCNNs by Algorithm~\ref{algorithm-HyperCNN} on the latent space. 
In the generation phase, we randomly select three images in each class. These samples are first encoded into a latent space, then mapped by our trained convex gradient, and finally recovered in the original space by the pre-trained decoder. 

For all experiments, we set $n_0 = n_1 = 500, \lambda = 1, \epsilon = 0.1, T = 30$, and each HyCNN has a single hidden layer with 128 units. 
Figure~\ref{fig:FFHQ-gender} shows the trajectories of generating the attacked images of all selected samples in the gender classification task. 
In each subfigure, the first column shows the original image, and the other 12 columns represent the images generated using the convex gradient maps trained for 30 epochs, respectively.
As illustrated, the samples we generated underwent meaningful modifications to the original images, such that the generated images gradually acquired visual characteristics associated with the opposite dataset label after a small number of training epochs. 

\begin{figure}[!htb]
  \centering
  \begin{subfigure}{\linewidth}
    \centering
    \includegraphics[width=\linewidth,height=\textheight,keepaspectratio]{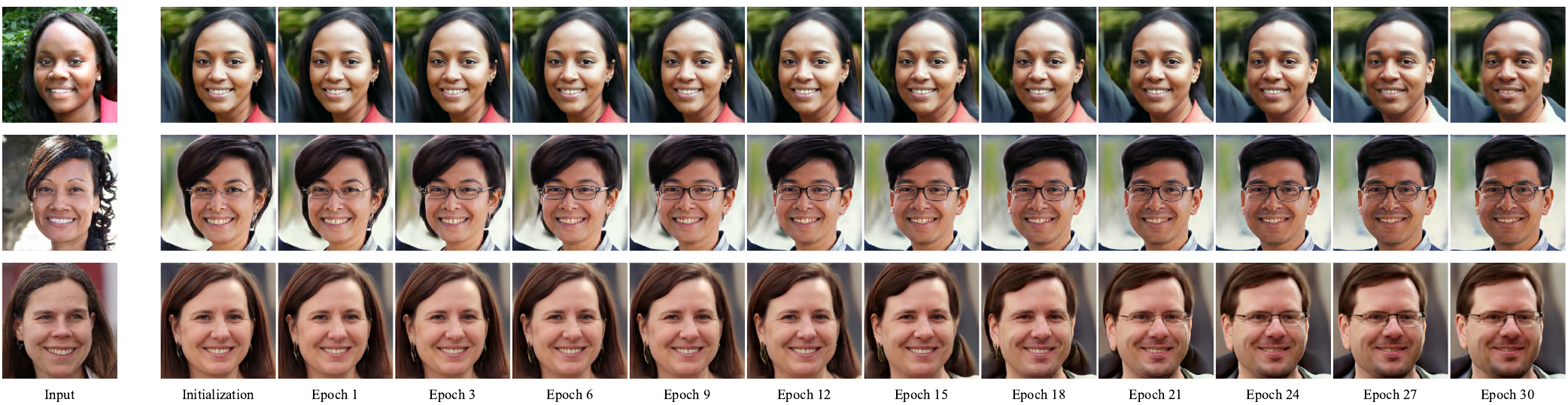}
    \caption{\textit{Female} to \textit{Male}}
  \end{subfigure}
  \begin{subfigure}{\linewidth}
    \centering
    \includegraphics[width=\linewidth,height=\textheight,keepaspectratio]{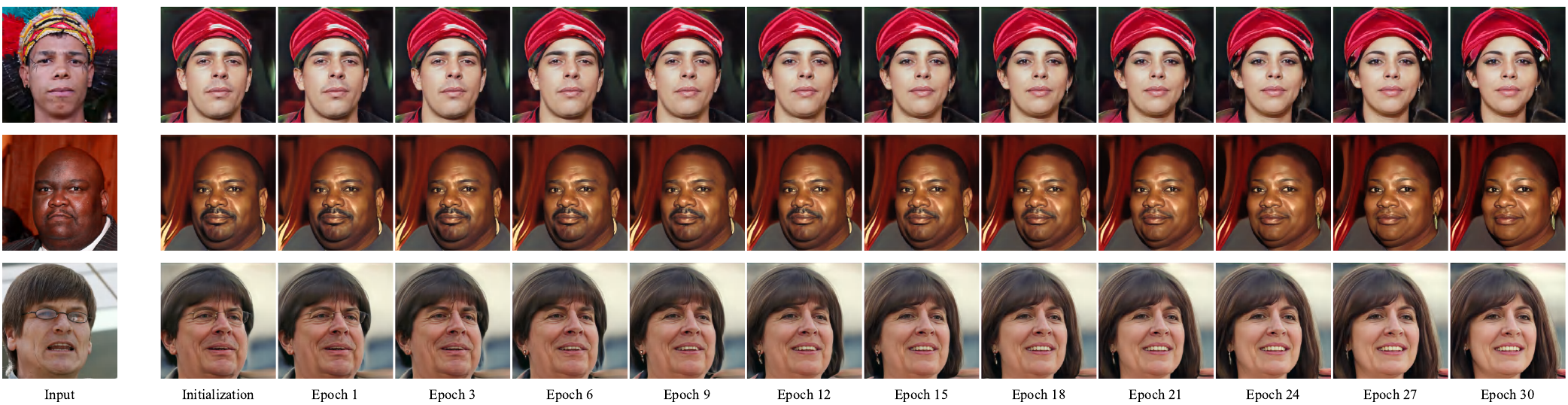}
    \caption{\textit{Male} to \textit{Female}}
  \end{subfigure} 
  \caption{The trajectory of image generation between \textit{Female} (Class 0) and \textit{Male} (Class 1). The first column displays the original images, while the following columns show the images generated during 30 epochs. }
  \label{fig:FFHQ-gender}
\end{figure}

\section{Conclusions and Discussions}\label{sec-conclusion}

As a fundamental method in statistics, robust hypothesis testing addresses noisy data, distributional shifts, and measurement errors in real-world, uncertain environments. 
This paper develops a generative framework for Sinkhorn distributionally robust hypothesis testing (SDRHT), overcoming the non-continuity of the least-favorable distributions of WDRO-based testing and the computational challenges of existing SDRHT methods. 
By exploiting a conditional-KL representation of Sinkhorn ambiguity sets and strong duality, this framework reformulates the infinite-dimensional problem as an optimization over parametrized convex potentials. 
We further establish approximation guarantees for the potentials and their gradients, as well as the distributional universality of the induced transport maps. 
Numerical results demonstrate the superior performance and robustness of our method compared to existing methods on tasks with different sample sizes and dimensions. 

There are several topics worth investigating for future work. 
First, it is interesting to extend our method to multi-class hypothesis testing, exploring its closed-form optimal detector and efficient algorithms. 
Then, since the transport map between distributions is the gradient of a convex potential, directly parameterizing the gradient instead of the potential may yield a better approximation of the transport map. 
Finally, for distributions with extremely high dimensions, e.g., high-resolution images, it is important to develop methods that efficiently solve transport maps while maintaining the convexity of potentials.



\bibliographystyle{IEEEtran}
\bibliography{references}

\newpage

\appendix

\newpage 

\section*{Appendix}

\section{Summary of Notations}

For integer $K \in \mathbb{Z}_+$, define $\left [ K \right ] := \left \{ 1,\cdots,K \right \}$, which is the set of positive running indices to $K$. 
The list of important sets, spaces, and notations used in this paper is shown in Table~\ref{T1-notations}. 
The list of important functions is shown in Table~\ref{T2-notations}. 

\begin{table}[htbp]
	\centering
	\caption {Notations List}
	\label{T1-notations}
  \renewcommand{\arraystretch}{1.0} 
	\begin{tabular}[c]{llc}
            \toprule
               \textbf{Spaces}& \textbf{Description} \\
                $\Omega$ & Closed sample space, $\Omega \subseteq \mathbb{R}^d$. \\
                $\mathcal{F}$ & 
                Functional space, which is not necessarily compact. \\ 
            \midrule
            \multicolumn{2}{l}{\textbf{Basic Sets}} \\  
                $\mathbb{K}$ & Indices set of hypotheses, $\mathbb{K} = \{0,1 \}$. \\
                $\mathscr{P}\left(\Omega\right)$ & The set of all probability distributions on $\Omega$. \\
                $\mathcal{M}\left(\Omega\right)$ & The set of measures on $\Omega$. \\ 
                $\Gamma(\mathbb{P}, \mathbb{Q})$ & Joint distribution set of marginal distribution $\mathbb{P}$ and $ \mathbb{Q}$. \\ 
                $\mathcal{P}_k $ &  The distribution set under hypotheses $H_k$, $\mathcal{P}_k \subset \mathscr{P}\left(\Omega\right)$ for each $k \in \mathbb{K}$. \\
                $\mathcal{F}_{\mathrm{CVX}}$ & Set of convex functions $\varphi: \mathbb{R}^d \to \mathbb{R} \cup\{+\infty\}$. \\
                \midrule
               \multicolumn{2}{l}{\textbf{Parameters}} \\
               $\Reg$ & Entropic regularization parameter in Sinkhorn discrepancy, $\Reg \ge 0$.  \\
               $\rho_k$ & Radius for Sinkhorn-based ambiguity set, for each $k \in \mathbb{K}$. \\ 
               $\lambda_k$ & Penalty parameter in soft-constraint problem~\eqref{Eq:soft:constrained:SDRO}, $\lambda_k \ge 0$ for each $k \in \mathbb{K}$. \\
                
                \midrule
               \multicolumn{2}{l}{\textbf{Distributions}} \\
                $\widehat{\mathbb{P}}_k$ & Empirical distribution with $n_k$ samples for each $k \in \mathbb{K}$. \\
                $\mathcal{K}_{\boldsymbol{x}_k^{(i)}, \epsilon}$ & Kernel extension for any sample point $i \in [n_k]$ and each $k \in \mathbb{K}$. \\
                $\mathbb{P}_k^{\epsilon}$ & Mixture of continuous distributions $\mathcal{K}_{\boldsymbol{x}_k^{(i)}, \epsilon}$ for all $i \in [n_k]$, see Eq.~\eqref{Eq-pk-epsilon}. \\ 
                $\widetilde{\mathbb{P}}_k$ & Discrete approximation of distribution $\mathbb{P}_k^{\epsilon}$ with $N$ samples, for each $k \in \mathbb{K}$. \\

                \midrule
               \multicolumn{2}{l}{\textbf{Discrepancies}} \\
                $\mathcal{S}_{\Reg}(\mathbb{P}, \mathbb{Q})$ & Sinkhorn discrepancy from distribution $\mathbb{P}$ to $ \mathbb{Q}$, see Definition~\ref{Def-Sinkhorn}. \\
                $\mathcal{D}_{\mathrm{KL}}(\mathbb{P}, \mathbb{Q})$ & Kullback–Leibler divergence (KL divergence) from distribution $\mathbb{P}$ to $ \mathbb{Q}$. \\
                
                \midrule
               \multicolumn{2}{l}{\textbf{Ambiguity Sets}} \\
                $\mathbb{B}_{\rho_k}(\widehat{\mathbb{P}}_k)$ & Sinkhorn discrepancy-based ambiguity set centered at $\widehat{\mathbb{P}}_k$ for each $k \in \mathbb{K}$, see Eq.~\eqref{Eq:hard:constrained:SDRO-ambiguity}. \\
                $\mathbb{B}_{\bar{\rho}_k}(\mathbb{P}_k^{\epsilon})$ & Equivalent KL-based ambiguity set of $\mathbb{B}_{\rho_k}(\widehat{\mathbb{P}}_k)$ centered at $\mathbb{P}_k^{\epsilon}$, see Proposition~\ref{prop-kl-connection}.  \\
            \bottomrule
        \end{tabular}
\end{table}

\begin{table}[htbp]
	\centering
	\caption {Functions List}
	\label{T2-notations}
  \renewcommand{\arraystretch}{1.0} 
	\begin{tabular}[c]{llc}
            \toprule
               \textbf{Functions}& \textbf{Description} \\
               \midrule
                $c: \Omega\times \Omega \to \mathbb{R}$ & Transport cost function satisfying Assumption~\ref{assumption-tcost}, we choose $c(\boldsymbol{x},\boldsymbol{z}) := \frac{1}{2}\|\boldsymbol{x} - \boldsymbol{z}\|_2^2$. \\
               
                $\ell: \mathbb{R} \to \mathbb{R}_+ \cup \{\infty\}$ & Generating function (see Definition~\ref{Def-generating}), we choose $\ell(t) = \exp(t)$. \\ 
                
                $\phi: \Omega \to \mathbb{R}$ & Detector for hypothesis testing, which is measurable and upper semicontinuous. \\

                $\phi_{\text{opt}}: \Omega \to \mathbb{R}$ & Optimal detector for generating function $\ell(t) = \exp(t)$, $\phi_{\text{opt}}  = \frac{1}{2}\log (p_{0}/p_{1} )$. \\

                $\varphi_1, \varphi_2: \mathbb{R}^d \to \mathbb{R}$ & Convex functions, decision variables in the problem~\eqref{Eq-SDRO-convex-functional}. \\

                $\varphi^{\ast}: \mathbb{R}^d \to \mathbb{R} \cup \{\infty\}$ & Convex conjugate of a convex function $\varphi : \mathbb{R}^d \to \mathbb{R} \cup \{\infty\}$. \\

                $\mathsf{HyCNN}_k: \mathbb{R}^d \to \mathbb{R} $ & Convex function formed by the HyCNN for each $k \in \mathbb{K}$. \\

                $\widetilde{\varphi}_k: \mathbb{R}^d \to \mathbb{R}$ & Strongly convex function, $\widetilde{\varphi}_k(\boldsymbol{z}) = \mathsf{HyCNN}_k(\boldsymbol{z}) + \frac{\alpha}{2} \| \boldsymbol{z}\|_2^2$ for each $k \in \mathbb{K}$ and $\alpha > 0$. \\ 

                $L_{k}: \Omega\to \mathbb{R}$ & Loss function of a given detector $\phi$ for sample $\boldsymbol{\omega}$, satisfying Assumption~\ref{assumption-2}. \\ 
                
                $\Phi: \mathscr{P}(\Omega)^2 \to \mathbb{R}$ & Convex surrogate risk of the detector $\phi$.  \\ 


                $g_{\epsilon, k}^{c}$ & Continuous extension for the optimal entropic potential $g_{\Reg, k, (n_k, N)}$ in Algorithm~\ref{algorithm-sinkhorn}. \\

                $\widehat{\mathcal{S}}_{\Reg, (n_k, N)} $ & Finite sample estimator for the Sinkhorn discrepancy. \\
            \bottomrule
        \end{tabular}
\end{table}

\section{Technical Analyses and Proofs} 

\subsection{Proof of Proposition~\ref{prop-kl-connection} } \label{app-sec-prop-kl-connection}

\begin{proof}[Proof of Proposition~\ref{prop-kl-connection}]

By the definition of Sinkhorn discrepancy, for distributions $\mathbb{P}, \mathbb{Q} \in \mathscr{P}(\Omega)$, it is equivalent to 
\begin{equation} \label{prop1-prove-sinkhorn}
    \begin{aligned}
        \mathcal{S}_{\Reg}(\mathbb{P}, \mathbb{Q})  = \inf_{\gamma\in\Gamma(\mathbb{P}, \mathbb{Q})} \quad \int_{\Omega \times \Omega} \Big(c(\boldsymbol{x},\boldsymbol{z}) + \Reg \cdot  \log\frac{\diff\gamma(\boldsymbol{x},\boldsymbol{z})}{\diff\mathbb{P}(\boldsymbol{x})\diff \nu(\boldsymbol{z})}\Big)\diff\gamma(\boldsymbol{x},\boldsymbol{z}) .
    \end{aligned}
\end{equation}
Taking the reference distribution as the empirical distribution $\widehat{\mathbb{P}} \in \mathscr{P}(\Omega)$, Eq.~\eqref{prop1-prove-sinkhorn} can be further reformulated to 
\begin{equation} \label{Eq-Sinkhorn-kl-inf}
    \begin{aligned}
        \mathcal{S}_{\Reg}(\widehat{\mathbb{P}}, \mathbb{Q}) & = \inf_{\gamma\in\Gamma(\widehat{\mathbb{P}}, \mathbb{Q})} \quad \mathbb{E}_{\boldsymbol{x} \sim \widehat{\mathbb{P}}} \Big [ \int_{\Omega} \Big( c(\boldsymbol{x},\boldsymbol{z}) + \Reg  \log\frac{\diff \gamma_{\boldsymbol{x}}(\boldsymbol{z})}{\diff \nu(\boldsymbol{z})} \Big)\diff \boldsymbol{z} \Big] \\
        & = \inf_{\gamma\in\Gamma(\widehat{\mathbb{P}}, \mathbb{Q})} \quad  \Reg \cdot \mathbb{E}_{\boldsymbol{x} \sim \widehat{\mathbb{P}}} \Big [ \int_{\Omega}  \log\frac{\diff\gamma_{\boldsymbol{x}}(\boldsymbol{z})}{e^{-c(\boldsymbol{x}, \boldsymbol{z})} \cdot \diff \nu(\boldsymbol{z})} \diff \boldsymbol{z} \Big] \\
        & = \inf_{\gamma\in\Gamma(\widehat{\mathbb{P}}, \mathbb{Q})} \quad \Reg \cdot \mathbb{E}_{\boldsymbol{x} \sim \widehat{\mathbb{P}}} \Big [\int_{\Omega}  \log\frac{\diff\gamma_{\boldsymbol{x}}(\boldsymbol{z})}{\diff \mathcal{K}_{\boldsymbol{x}, \epsilon}(\boldsymbol{z})}\diff \boldsymbol{z}  - \log\, \int_{\mathbb{R}^d}e^{-c(\boldsymbol{x}, \boldsymbol{z})/ \Reg} \diff \boldsymbol{z}\Big]  \\
        & = \inf_{\gamma\in\Gamma(\widehat{\mathbb{P}}, \mathbb{Q})} \quad \Reg \cdot \mathbb{E}_{\boldsymbol{x} \sim \widehat{\mathbb{P}}} \Big [ \mathcal{D}_{\mathrm{KL}} (\gamma_{\boldsymbol{x}}\| \mathcal{K}_{\boldsymbol{x}, \epsilon})\Big ] - \Reg \cdot \mathbb{E}_{\boldsymbol{x} \sim \widehat{\mathbb{P}}} \Big [ \log\, \int_{\mathbb{R}^d}e^{-c(\boldsymbol{x}, \boldsymbol{z})/ \Reg} \diff \boldsymbol{z}\Big], 
    \end{aligned}
\end{equation}
where 
\begin{equation}\nonumber
    \diff \mathcal{K}_{\boldsymbol{x}, \epsilon}(\boldsymbol{z}) := \frac{e^{-c(\boldsymbol{x}, \boldsymbol{z}) / \epsilon}  }{\int_{\mathbb{R}^d} e^{-c(\boldsymbol{x}, \boldsymbol{u}) / \epsilon} \cdot \diff \boldsymbol{u}} \cdot \mathrm{d} v(\boldsymbol{z}). 
\end{equation}
Thus, the ambiguity sets in Eq.~\eqref{Eq:hard:constrained:SDRO-ambiguity} is equivalent to 
\begin{equation}\nonumber
    \inf_{\gamma\in\Gamma(\widehat{\mathbb{P}}_k, \mathbb{Q})} \quad \Reg \cdot \mathbb{E}_{\boldsymbol{x} \sim \widehat{\mathbb{P}}_k} \Big [ \mathcal{D}_{\mathrm{KL}} (\gamma_{\boldsymbol{x}}\| \mathcal{K}_{\boldsymbol{x}, \epsilon})\Big ] \le \rho_k + \Reg \cdot \mathbb{E}_{\boldsymbol{x} \sim \widehat{\mathbb{P}}_k} \Big [ \log\, \int_{\mathbb{R}^d}e^{-c(\boldsymbol{x}, \boldsymbol{z})/ \Reg} \diff \boldsymbol{z}\Big], \quad \forall k \in \mathbb{K},
\end{equation}
that is, 
\begin{equation} \nonumber
    \mathbb{B}_{\bar{\rho}_k}(\mathbb{P}_k^{\epsilon}) := \{ \mathbb{Q}\in \mathscr{P}(\Omega): \inf_{\gamma\in\Gamma(\widehat{\mathbb{P}}_k, \mathbb{Q})} \, \,  \mathbb{E}_{\boldsymbol{x} \sim \widehat{\mathbb{P}}_k} \Big [ \mathcal{D}_{\mathrm{KL}} (\gamma_{\boldsymbol{x}}\| \mathcal{K}_{\boldsymbol{x}, \epsilon})\Big ] \le \bar{\rho}_k/\Reg \}, \quad \forall k \in \mathbb{K}, 
\end{equation}
where $\bar{\rho}_k := \rho_k +\Reg \mathbb{E}_{\boldsymbol{x}\sim \widehat{\mathbb{P}}_k}\Big [ \log\, \int_{\mathbb{R}^d}e^{-c(\boldsymbol{x}, \boldsymbol{z})/ \Reg} \diff \boldsymbol{z} \Big]$ and probability density $\mathrm{d} \mathbb{P}_k^{\epsilon}(\boldsymbol{z}) := \mathbb{E}_{\boldsymbol{x} \sim \widehat{\mathbb{P}}_k}\left[ \diff \mathcal{K}_{\boldsymbol{x}, \epsilon}(\boldsymbol{z}) \right]$. 
This completes the proof.     

\end{proof}

\subsection{Proof of Theorem~\ref{theo-strong-duality} }\label{app-sec-theo-strong-duality}

To prove Theorem~\ref{theo-strong-duality}, we first establish the weak compactness of the Sinkhorn ambiguity sets and the weak upper semicontinuity of the expected loss. 

Based on Eq.~\eqref{KL-ambiguity}, the following lemma holds. 
\begin{lemma} [Weak Compactness of the Ambiguity Sets] \label{lemma-weakly-compact}
    Under Assumption~\ref{assumption-tcost}, the ambiguity set defined in \eqref{KL-ambiguity} is weakly compact in $\mathcal P(\Omega)$ for any closed set $\Omega \subseteq \mathbb{R}^d$, $k \in \mathbb{K}$,  and $\bar{\rho}_k \ge 0$. 
\end{lemma} 
The proof of Lemma~\ref{lemma-weakly-compact} is provided in the following Appendix~\ref{app-sec-lemma-weakly-compact}. 
While~\cite{cescon2026sinkhorn} also confirms the weak compactness of Sinkhorn ambiguity sets based on the properties of Wasserstein balls, our proof follows directly from the KL representation in Proposition~\ref{prop-kl-connection} and holds under more general assumptions.  

\begin{lemma}[Weak Upper Semicontinuity of the Expected Loss]\label{lemma-wusc}
Under Assumption~\ref{assumption-2}\ref{assumption-2-el-bound}, $\mathbb{E}_{\boldsymbol{\omega} \sim \mathbb{Q}} [L_{k}(\phi, \boldsymbol{\omega})]$ is weakly upper semicontinuous in $\mathbb{Q} \in \mathbb{B}_{\bar{\rho}_k}(\mathbb{P}_k^{\epsilon})$ for any $k \in \mathbb{K}$ and $\phi \in \mathcal{F}$. 
\end{lemma}

The proof of Lemma~\ref{lemma-wusc} is provided in the following Appendix~\ref{app-sec-lemma-wusc}. 

\subsubsection{Proof of Lemma~\ref{lemma-weakly-compact} }\label{app-sec-lemma-weakly-compact}

\begin{proof}[Proof of Lemma~\ref{lemma-weakly-compact}]

To prove $\mathbb{B}_{\bar{\rho}_k}(\mathbb{P}_k^{\epsilon})$ is weakly compact, we first show its connection with KL-divergence-based ambiguity set. 
Since the KL-divergence is a convex function, by Jensen's inequality, we have 
\begin{equation}\nonumber
\begin{aligned}
        \mathcal{D}_{\mathrm{KL}} ( \mathbb{Q} \| \mathbb{P}_k^{\epsilon} ) & \le \inf_{\gamma\in\Gamma(\widehat{\mathbb{P}}_k, \mathbb{Q})} \, \,  \mathbb{E}_{\boldsymbol{x} \sim \widehat{\mathbb{P}}_k} \Big [ \mathcal{D}_{\mathrm{KL}} (\gamma_{\boldsymbol{x}}\| \mathcal{K}_{\boldsymbol{x}, \epsilon})\Big ] \\
        & \le \frac{\rho_k}{\Reg} + \mathbb{E}_{\boldsymbol{x} \sim \widehat{\mathbb{P}}_k} \Big [ \log\, \int_{\mathbb{R}^d}e^{-c(\boldsymbol{x}, \boldsymbol{z})/ \Reg} \diff \boldsymbol{z}\Big], \quad \forall k \in \mathbb{K},
\end{aligned}
\end{equation}
where probability distribution $\mathbb{Q} \in \mathscr{P}(\Omega)$. 
Let 
\begin{equation} \nonumber
    \mathbb{B}^{\prime}_{\bar{\rho}_k}(\mathbb{P}_k^{\epsilon}) := \{ \mathbb{Q}: \mathcal{D}_{\mathrm{KL}} ( \mathbb{Q} \| \mathbb{P}_k^{\epsilon} ) \le \bar{\rho}_k/\Reg \}, \quad \forall k \in \mathbb{K}. 
\end{equation}
According to Proposition 3.12 in~\cite{kuhn2025distributionally}, the KL-divergence-based ambiguity sets $\mathbb{B}^{\prime}_{\bar{\rho}_k}(\mathbb{P}_k^{\epsilon})$ for any $k \in \mathbb{K}$ are weakly compact, for any closed set $\Omega \subseteq \mathbb{R}^d$ and $k \in \mathbb{K}$, distribution $\mathbb{P}_k^{\epsilon} \in \mathscr{P}(\Omega)$ and $\bar{\rho}_k \ge 0$. 

Since $\mathbb{B}_{\bar{\rho}_k}(\mathbb{P}_k^{\epsilon}) \subseteq \mathbb{B}^{\prime}_{\bar{\rho}_k}(\mathbb{P}_k^{\epsilon})$ for any $k \in \mathbb{K}$, it suffices to prove that $\mathbb{B}_{\bar{\rho}_k}(\mathbb{P}_k^{\epsilon})$ is weakly closed, as weakly closed subsets of weakly compact sets are weakly compact~\citep[Lemma 1]{Shafiee25nash}. 
Define 
\begin{equation} \nonumber
    \diff M_{k, \Reg} ( \boldsymbol{x}, \boldsymbol{z}) := \diff \widehat{\mathbb{P}}_k(\boldsymbol{x}) \cdot 
    \diff \mathcal{K}_{\boldsymbol{x}, \epsilon}(\boldsymbol{z}) = \frac{e^{-c(\boldsymbol{x}, \boldsymbol{z}) / \epsilon}  }{\int_{\mathbb{R}^d} e^{-c(\boldsymbol{x}, \boldsymbol{u}) / \epsilon} \cdot \diff \boldsymbol{u}} \cdot \diff \widehat{\mathbb{P}}_k(\boldsymbol{x}) \cdot \mathrm{d} v(\boldsymbol{z}), \quad \forall k \in \mathbb{K},
\end{equation}
According to Eq.~\eqref{prop1-prove-sinkhorn}, we have 
\begin{equation} \nonumber
    \inf_{\gamma\in\Gamma(\widehat{\mathbb{P}}_k, \mathbb{Q})} \, \,  \mathbb{E}_{\boldsymbol{x} \sim \widehat{\mathbb{P}}_k} \Big [ \mathcal{D}_{\mathrm{KL}} (\gamma_{\boldsymbol{x}}\| \mathcal{K}_{\boldsymbol{x}, \epsilon})\Big ]  = \inf_{\gamma\in\Gamma(\widehat{\mathbb{P}}_k, \mathbb{Q})} \mathcal{D}_{\mathrm{KL}} (\gamma \| M_{k, \Reg} ), \quad \forall k \in \mathbb{K}. 
\end{equation}
According to Proposition 3.12 in~\cite{kuhn2025distributionally}, to prove the Sinkhorn-based ambiguity set is weakly closed, it suffices to prove that function $f_k(\mathbb{Q}):= \inf_{\gamma\in\Gamma(\widehat{\mathbb{P}}_k, \mathbb{Q})} \mathcal{D}_{\mathrm{KL}} (\gamma \| M_{k, \Reg} )$ is weakly lower semicontinuous in $\mathbb{Q}$ for each $k \in \mathbb{K}$, which implies that any sublevel set of $f(\mathbb{Q})$, i.e., $\{\mathbb{Q}: f(\mathbb{Q}) \le c\}$ for any $c \in \mathbb{R}$, is weakly closed. 
Let $\{ \mathbb{Q}_{j} \}_{j \in \mathbb{N}} \subset \mathscr{P}(\Omega)$ be any sequence of distributions that converges weakly to $\mathbb{Q}$. 
By the definition of infimum, there exists a sequence of couplings $\gamma_n \in \Gamma (\widehat{\mathbb{P}}_k, \mathbb{Q}_n)$ that converges weakly to $\gamma \in \Gamma (\widehat{\mathbb{P}}_k, \mathbb{Q})$ such that 
\begin{equation}\label{prove-KL-infimum}
    \mathcal{D}_{\mathrm{KL}} (\gamma_n \| M_{k, \Reg} ) \le f_k(\mathbb{Q}_n) + \frac{1}{n}, \quad \forall k \in \mathbb{K}. 
\end{equation}
Then, for each $k \in \mathbb{K}$, we have 
\begin{equation}\nonumber
    \begin{aligned}
        f_k(\mathbb{Q}) \le \mathcal{D}_{\mathrm{KL}} (\gamma \| M_{k, \Reg} ) \le \liminf_{j \in \mathbb{N}}\,\, \mathcal{D}_{\mathrm{KL}} (\gamma_j \| M_{k, \Reg} ) \le \liminf_{j \in \mathbb{N}}\,\, f_k(\mathbb{Q}_j), \quad \forall k \in \mathbb{K},
    \end{aligned}
\end{equation}
where the second inequality is due to the weak lower semi-continuity of $\mathcal{D}_{\mathrm{KL}} (\gamma \| M_{k, \Reg} )$ in $\gamma$~\citep[Proposition 3.12]{kuhn2025distributionally}, and the final inequality is due to Eq.~\eqref{prove-KL-infimum}. 
Therefore, for each $k \in \mathbb{K}$, the function $f_k(\mathbb{Q})$ is weakly lower semicontinuous in $\mathbb{Q}$ which implies that the ambiguity set $\mathbb{B}_{\bar{\rho}_k}(\mathbb{P}_k^{\epsilon})$ is weakly closed, and thereby weakly compact. 
This completes the proof. 
\end{proof}

\subsubsection{Proof of Lemma~\ref{lemma-wusc} } \label{app-sec-lemma-wusc}

\begin{proof}[Proof of Lemma~\ref{lemma-wusc}]

    We prove Lemma~\ref{lemma-wusc} in two steps. 
    Specifically, we prove that the loss function $L_k$ is uniformly integrable and $ \mathbb{E}_{\boldsymbol{\omega} \sim \mathbb{Q}}\Big[ L_{k}(\phi, \boldsymbol{\omega}) \Big]$ is weakly upper semicontinuous with respect to the ambiguity set $\mathbb{Q} \in \mathbb{B}_{\bar{\rho}_k}(\mathbb{P}_k^{\epsilon})$ for each $k \in \mathbb{K}$ and any fixed $\phi$, respectively. 
    
    \textbf{Step 1 of Lemma~\ref{lemma-wusc}:} Based on the Donsker-Varadhan variational representation of the KL divergence~\cite{donsker1983asymptotic}, for any measurable function $f: \Omega \to \mathbb{R}$ and probability distribution $\mathbb{Q}$, we have 
    \begin{equation} \label{eq-DV-relation}
         \mathbb{E}_{\boldsymbol{\omega} \sim \mathbb{Q}} [f] - \log \, \mathbb{E}_{\boldsymbol{\omega} \sim \mathbb{P}_k^{\epsilon}} [e^{f}] \le \mathcal{D}_{\mathrm{KL}}(\mathbb{Q}, \mathbb{P}_k^{\epsilon}) \le \inf_{\gamma\in\Gamma(\widehat{\mathbb{P}}_k, \mathbb{Q})} \, \,  \mathbb{E}_{\boldsymbol{x} \sim \widehat{\mathbb{P}}_k} \Big [ \mathcal{D}_{\mathrm{KL}} (\gamma_{\boldsymbol{x}}\| \mathcal{K}_{\boldsymbol{x}, \epsilon})\Big ] \le \frac{\bar{\rho}_k}{\Reg}, \quad \forall k \in \mathbb{K} . 
    \end{equation}
    According to the Donsker-Varadhan variational representation part in relation~\eqref{eq-DV-relation}, under Assumption~\ref{assumption-2}\ref{assumption-2-el-bound}, for $\beta_k > 0$, we have 
    \begin{equation} \label{Eq-E-Lk}
        \mathbb{E}_{\boldsymbol{\omega} \sim \mathbb{Q}} [L_{k}(\phi, \boldsymbol{\omega})] \le \frac{1}{\beta} \Big ( \mathcal{D}_{\mathrm{KL}}(\mathbb{Q}, \mathbb{P}_k^{\epsilon})  + \log \, \mathbb{E}_{\boldsymbol{\omega} \sim \mathbb{P}_k^{\epsilon}} [e^{\beta_k \cdot L_{k}(\phi, \boldsymbol{\omega})}] \Big ) < \infty,  \quad \forall k \in \mathbb{K}, 
    \end{equation}
    which implies that the inner problem of \eqref{Eq:hard:constrained:SDRO} is bounded. 
    For a truncation threshold $M > 0$, let $\mathds{1}_{\{L_{k}(\phi, \boldsymbol{\omega})>M\}}$ be an indicator function that equals to 1 if $L_{k}(\phi, \boldsymbol{\omega}) > M$ and 0 otherwise. Based on relation~\eqref{eq-DV-relation}, we also have 
    \begin{equation} \nonumber
        \mathbb{E}_{\boldsymbol{\omega} \sim \mathbb{Q}} [L_{k}(\phi, \boldsymbol{\omega})\cdot \mathds{1}_{\{L_{k}(\phi, \boldsymbol{\omega})>M\}}] \le \frac{1}{\beta} \Big ( \mathcal{D}_{\mathrm{KL}}(\mathbb{Q}, \mathbb{P}_k^{\epsilon})  + \log \, \mathbb{E}_{\boldsymbol{\omega} \sim \mathbb{P}_k^{\epsilon}} \Big[e^{\beta_k \cdot L_{k}(\phi, \boldsymbol{\omega})\cdot \mathds{1}_{\{L_{k}(\phi, \boldsymbol{\omega})>M\}}}\Big] \Big ), \quad \forall k \in \mathbb{K}, 
    \end{equation} 
    and it follows that
    \begin{equation}\nonumber
        \sup_{\mathbb{Q} \in \mathbb{B}_{\bar{\rho}_k}(\mathbb{P}_k^{\epsilon})} \,\, \mathbb{E}_{\boldsymbol{\omega} \sim \mathbb{Q}} [L_{k}(\phi, \boldsymbol{\omega})\cdot \mathds{1}_{\{L_{k}(\phi, \boldsymbol{\omega})>M\}}] \le \frac{\bar{\rho}_k}{\Reg \beta} + \frac{1}{\beta} \log \, \mathbb{E}_{\boldsymbol{\omega} \sim \mathbb{P}_k^{\epsilon}} \Big [e^{\beta_k \cdot L_{k}(\phi, \boldsymbol{\omega})\cdot \mathds{1}_{\{L_{k}(\phi, \boldsymbol{\omega})>M\}}}\Big ], \quad \forall k \in \mathbb{K}.
    \end{equation}

    When $M \to \infty$, we have $\mathds{1}_{\{L_k(\phi,\boldsymbol{\omega})>M\}} \to 0$
    for $\mathbb{P}_k^\epsilon$-almost every $\boldsymbol{\omega}$. Moreover, as both $\beta_k$ and $L_{k}(\phi, \boldsymbol{\omega})$ are non-negative, we have 
    \[
    e^{\beta_k \cdot L_k(\phi,\boldsymbol{\omega})
\mathds{1}_{\{L_k(\phi,\boldsymbol{\omega})>M\}}}
\le
e^{\beta_k \cdot L_k(\phi,\boldsymbol{\omega})}, \quad \forall k \in \mathbb{K},
    \]
    where $e^{\beta_k \cdot L_k(\phi,\boldsymbol{\omega})}$ is integrable under Assumption~\ref{assumption-2}\ref{assumption-2-el-bound}. Hence, by the dominated convergence theorem,
    \begin{equation}\nonumber
        \lim_{M \to \infty} \log\,\mathbb{E}_{\boldsymbol{\omega} \sim \mathbb{P}_k^{\epsilon}} \Big [e^{\beta_k \cdot L_{k}(\phi, \boldsymbol{\omega})\cdot \mathds{1}_{\{L_{k}(\phi, \boldsymbol{\omega})>M\}}}\Big ] = \ \log\, \mathbb{E}_{\boldsymbol{\omega} \sim \mathbb{P}_k^{\epsilon}} \Big [ e^0 \Big ] = 0, \quad \forall k \in \mathbb{K}. 
    \end{equation}
    Thus, we have 
    \begin{equation}
        \lim_{M \to \infty} \sup_{\mathbb{Q} \in \mathbb{B}_{\bar{\rho}_k}(\mathbb{P}_k^{\epsilon})} \,\, \mathbb{E}_{\boldsymbol{\omega} \sim \mathbb{Q}} \Big[L_{k}(\phi, \boldsymbol{\omega})\cdot \mathds{1}_{\{L_{k}(\phi, \boldsymbol{\omega})>M\}}\Big] \le \frac{\bar{\rho}_k}{\Reg \beta_k}, \quad \forall k \in \mathbb{K}. 
    \end{equation}
    Since Assumption~\ref{assumption-2}\ref{assumption-2-el-bound} guarantees that this property holds for any $\beta_k > 0$, let $\beta_k \to \infty$, and then we have 
    \begin{equation}\label{Eq-uniform-integrable}
        \lim_{M \to \infty} \sup_{\mathbb{Q} \in \mathbb{B}_{\bar{\rho}_k}(\mathbb{P}_k^{\epsilon})} \,\, \mathbb{E}_{\boldsymbol{\omega} \sim \mathbb{Q}} \Big[L_{k}(\phi, \boldsymbol{\omega})\cdot \mathds{1}_{\{L_{k}(\phi, \boldsymbol{\omega})>M\}}\Big]\le \inf_{\beta>0} \,\, \frac{\bar\rho_k}{\epsilon\beta_k} = 0, \quad \forall k \in \mathbb{K},
    \end{equation}
    which completes the proof of uniform integrability. 

    \textbf{Step 2 of Lemma~\ref{lemma-wusc}:}  We next prove the weak upper semicontinuity of $\mathbb{E}_{\boldsymbol{\omega} \sim \mathbb{Q}} [L_{k}(\phi, \boldsymbol{\omega})]$ in $\mathbb{Q} \in \mathbb{B}_{\bar{\rho}_k}(\mathbb{P}_k^{\epsilon})$. 
    We consider a continuous detector $\phi$ and a non-negative, non-decreasing, convex, and continuous generating function $\ell$. 
    Then, the loss functions $L_0$ and $L_1$ are non-negative and upper semicontinuous for all $k \in \mathbb{K}$. 
    By Lemma~\ref{lemma-weakly-compact}, the ambiguity set $\mathbb{B}_{\bar{\rho}_k}(\mathbb{P}_k^{\epsilon})$ is weakly closed. 
    For each $k \in \mathbb{K}$, let $\mathbb{Q}_{j} \in \mathbb{B}_{\bar{\rho}_k}(\mathbb{P}_k^{\epsilon})$ for all $j \in \mathbb{N}$ be any sequence of distributions that converges weakly to $\mathbb{Q}$, and then we have
    \begin{equation}\label{Eq-wusc-limsup}
        \begin{aligned}
            \limsup_{j \in N} \,  & \mathbb{E}_{\boldsymbol{\omega} \sim \mathbb{Q}_{j} }\Big [L_{k}(\phi, \boldsymbol{\omega})\Big]  \\ & \le \limsup_{j \in N} \, \mathbb{E}_{\boldsymbol{\omega} \sim \mathbb{Q}_{j} }\Big[ \min\{L_{k}(\phi, \boldsymbol{\omega}), M\}\Big] + \limsup_{j \in N} \, \mathbb{E}_{\boldsymbol{\omega} \sim \mathbb{Q}_j} \Big[(L_{k}(\phi, \boldsymbol{\omega})-M)\cdot \mathds{1}_{\{L_{k}(\phi, \boldsymbol{\omega})>M\}}\Big] \\
            & \le \mathbb{E}_{\boldsymbol{\omega} \sim \mathbb{Q} }\Big[ \min\{L_{k}(\phi, \boldsymbol{\omega}), M\}\Big] + \limsup_{j \in N} \, \mathbb{E}_{\boldsymbol{\omega} \sim \mathbb{Q}_j} \Big[(L_{k}(\phi, \boldsymbol{\omega})-M)\cdot \mathds{1}_{\{L_{k}(\phi, \boldsymbol{\omega})>M\}}\Big] \\
            & \le \mathbb{E}_{\boldsymbol{\omega} \sim \mathbb{Q} }\Big[ \min\{L_{k}(\phi, \boldsymbol{\omega}), M\}\Big] + \sup_{\mathbb{Q} \in \mathbb{B}_{\bar{\rho}_k}(\mathbb{P}_k^{\epsilon})} \,\, \mathbb{E}_{\boldsymbol{\omega} \sim \mathbb{Q}} [L_{k}(\phi, \boldsymbol{\omega})\cdot \mathds{1}_{\{L_{k}(\phi, \boldsymbol{\omega})>M\}}], \quad \forall k \in \mathbb{K},
        \end{aligned}
    \end{equation}
    where $M > 0$, the second inequality is due to Proposition 3.3 in~\cite{kuhn2025distributionally} for the upper semicontinuous and bounded function $\min\{L_{k}(\phi, \boldsymbol{\omega}), M\}$, and the third inequality follows because $\mathbb{Q}_j \in \mathbb{B}_{\bar{\rho}_k}(\mathbb{P}_k^{\epsilon})$ for any $j \in \mathbb{N}$ and $(L_{k}(\phi, \boldsymbol{\omega})-M)\cdot \mathds{1}_{\{L_{k}(\phi, \boldsymbol{\omega})>M\}} \le L_{k}(\phi, \boldsymbol{\omega})\cdot \mathds{1}_{\{L_{k}(\phi, \boldsymbol{\omega})>M\}}$. 
    Taking the limit as $M \to \infty$ on both sides of inequality \eqref{Eq-wusc-limsup}, we obtain
    \begin{equation} \nonumber
    \begin{aligned}
        \limsup_{j \in N} \, & \mathbb{E}_{\boldsymbol{\omega} \sim \mathbb{Q}_{j} } \Big [L_{k}(\phi, \boldsymbol{\omega})\Big] \\ 
        & \le \lim_{M \to \infty} \, \mathbb{E}_{\boldsymbol{\omega} \sim \mathbb{Q} }\Big[ \min\{L_{k}(\phi, \boldsymbol{\omega}), M\}\Big] + \lim_{M \to \infty} \sup_{\mathbb{Q} \in \mathbb{B}_{\bar{\rho}_k}(\mathbb{P}_k^{\epsilon})} \,\, \mathbb{E}_{\boldsymbol{\omega} \sim \mathbb{Q}} [L_{k}(\phi, \boldsymbol{\omega})\cdot \mathds{1}_{\{L_{k}(\phi, \boldsymbol{\omega})>M\}}] \\
        & = \lim_{M \to \infty} \, \mathbb{E}_{\boldsymbol{\omega} \sim \mathbb{Q} }\Big[ \min\{L_{k}(\phi, \boldsymbol{\omega}), M\}\Big] 
        = \mathbb{E}_{\boldsymbol{\omega} \sim \mathbb{Q}}\Big[ L_{k}(\phi, \boldsymbol{\omega}) \Big], \quad \forall k \in \mathbb{K} ,
    \end{aligned}
    \end{equation}
    where the first equality follows from Eq.~\eqref{Eq-uniform-integrable} and the second equality follows from the monotone convergence theorem. This completes the proof. 
\end{proof}

\subsubsection{Formal Proof of Theorem~\ref{theo-strong-duality}}

\begin{proof}[Proof of Theorem~\ref{theo-strong-duality}.] 

    For Theorem~\ref{theo-strong-duality}\ref{theo-strong-duality-hard}, the problem~\eqref{Eq:hard:constrained:SDRO} can be reformulated as 
    \begin{equation}\nonumber
        \inf_{\phi \in \mathcal{F}}\sup_{\mathbb{P}_k\in \mathbb{B}_{\rho_k}(\widehat{\mathbb{P}}_k), \forall k \in \mathbb{K}} \quad \sum_{k \in 
        \mathbb{K}} \mathbb{E}_{\boldsymbol{\omega} \sim \mathbb{P}_k } \Big[ L_{k}(\phi, \boldsymbol{\omega}) \Big]. 
    \end{equation}
    Note that $\mathbb{B}_{\bar{\rho}_k}(\mathbb{P}_k^{\epsilon})$ is weakly compact for any $\bar{\rho}_k \ge 0$ thanks to Proposition~\ref{prop-kl-connection} and Lemma~\ref{lemma-weakly-compact}, which applies because of Assumption~\ref{assumption-tcost}. In addition, $\mathbb{E}_{\boldsymbol{\omega} \sim \mathbb{Q}}\Big[ L_{k}(\phi, \boldsymbol{\omega}) \Big]$ is non-negative and weakly upper semicontinuous in $\mathbb{Q} \in \mathbb{B}_{\bar{\rho}_k}(\mathbb{P}_k^{\epsilon})$ for all $k \in \mathbb{K}$ and $\phi \in \mathcal{F}$, due to Definition~\ref{Def-generating} and Lemma~\ref{lemma-wusc}. 
    As $\Phi(\phi;\mathbb{P}_0,\mathbb{P}_1)$ is convex in $\phi$ and concave in distributions, Theorem 1(II) in~\cite{Shafiee25nash}, an extension of the lopsided minimax theorem~\citep[Theorem~5.15]{kuhn2025distributionally},  ensures that  
    \begin{equation}\nonumber
        \text{Optimal Value of~\eqref{Eq:hard:constrained:SDRO} = }\max_{\mathbb{P}_0\in \mathbb{B}_{\rho_0}(\widehat{\mathbb{P}}_0), \mathbb{P}_1\in \mathbb{B}_{\rho_1} (\widehat{\mathbb{P}}_1)}\,  \inf_{\phi \in \mathcal{F}} \, \Phi(\phi;\mathbb{P}_0,\mathbb{P}_1). 
    \end{equation} 

    For Theorem~\ref{theo-strong-duality}\ref{theo-strong-duality-soft}, the problem~\eqref{Eq:soft:constrained:SDRO} can be reformulated as
        \begin{equation}\nonumber
        \inf_{\phi \in \mathcal{F}}\sup_{\mathbb{P}_0, \mathbb{P}_1\in \mathscr{P}(\Omega)} \quad \sum_{k \in 
        \mathbb{K}} \mathbb{E}_{\boldsymbol{\omega} \sim \mathbb{P}_k } \Big[ L_{k}(\phi, \boldsymbol{\omega}) \Big] - \lambda_k \mathcal{S}_{\Reg}(\widehat{\mathbb{P}}_k,\mathbb{P}_k), 
    \end{equation}
    where $\lambda_k > 0$ for any $k \in \mathbb{K}$. 
    Define 
    \begin{equation}\nonumber
        H(\phi, \mathbb{P}_0, \mathbb{P}_1) := \sum_{k \in 
        \mathbb{K}} \mathbb{E}_{\boldsymbol{\omega} \sim \mathbb{P}_k } \Big[ L_{k}(\phi, \boldsymbol{\omega}) \Big] - \lambda_k \mathcal{S}_{\Reg}(\widehat{\mathbb{P}}_k,\mathbb{P}_k). 
    \end{equation}
    We next verify the level compactness of $H(\phi, \mathbb{P}_0, \mathbb{P}_1)$ required for the lopsided minimax theorem~\citep[Theorem~5.15]{kuhn2025distributionally}. 
    By Eqs.~\eqref{Eq-Sinkhorn-kl-inf}, \eqref{eq-DV-relation}, and~\eqref{Eq-E-Lk}, for any fixed $\phi$, we have 
    \begin{equation}\label{Eq-H-upper-bound}
    \begin{aligned}
        H(\phi, \mathbb{P}_0, \mathbb{P}_1) \le \sum_{k\in\mathbb{K}} \Big [ C_k+\Big (\frac{1}{\beta_k} - \lambda_k\epsilon\Big)\cdot \inf_{\gamma \in \Gamma(\widehat{\mathbb{P}}_k, \mathbb{P}_k)} \, \mathbb{E}_{\boldsymbol{x} \sim \widehat{\mathbb{P}}_k} [ \mathcal{D}_{\mathrm{KL}} (\gamma_{\boldsymbol{x}}\| \mathcal{K}_{\boldsymbol{x}, \epsilon}) ] \Big ], 
    \end{aligned}
    \end{equation}
    where $\beta_k >0$ is finite, $C_k := \frac{1}{\beta_k}\log \, \mathbb{E}_{\boldsymbol{\omega} \sim \mathbb{P}_k^{\epsilon}} [e^{\beta_k \cdot L_{k}(\phi, \boldsymbol{\omega})}] + \lambda_k \Reg \cdot \mathbb{E}_{\boldsymbol{x} \sim \widehat{\mathbb{P}}_k} [ \log\, \int_{\mathbb{R}^d}e^{-c(\boldsymbol{x}, \boldsymbol{z})/ \Reg} \diff \boldsymbol{z}]$ for any $k \in \mathbb{K}$ is finite by Assumptions~\ref{assumption-tcost}\ref{assumption-tcost-ex-bounded} and~\ref{assumption-2}\ref{assumption-2-el-bound}, and is independent of $\mathbb{P}_0$ and $\mathbb{P}_1$. 
    
    By choosing $\beta_{k} \in (\frac{1}{\lambda_{k} \epsilon}, +\infty)$ for each $k\in \mathbb{K}$, when $\inf_{\gamma \in \Gamma(\widehat{\mathbb{P}}_k, \mathbb{P}_k)} \, \mathbb{E}_{\boldsymbol{x} \sim \widehat{\mathbb{P}}_k} \Big [ \mathcal{D}_{\mathrm{KL}} (\gamma_{\boldsymbol{x}}\| \mathcal{K}_{\boldsymbol{x}, \epsilon})\Big ] \to \infty$, we have $H(\phi, \mathbb{P}_0, \mathbb{P}_1) \to -\infty$. 
    This implies that for any $\alpha \in \mathbb{R}$, there exist finite radii $r_0$ and $r_1$ such that the upper-level set of $H(\phi, \mathbb{P}_0, \mathbb{P}_1)$ satisfies
    \begin{equation} \nonumber
        \{(\mathbb{P}_0,\mathbb{P}_1): H(\phi, \mathbb{P}_0, \mathbb{P}_1) \ge \alpha\} \subseteq \mathbb{B}_{r_0}(\widehat{\mathbb{P}}_0) \times \mathbb{B}_{r_1}(\widehat{\mathbb{P}}_1). 
    \end{equation} 
    By Lemma~\ref{lemma-weakly-compact}, the set $\mathbb{B}_{r_0}(\widehat{\mathbb{P}}_0) \times \mathbb{B}_{r_1}(\widehat{\mathbb{P}}_1)$ is weakly compact. 
    For each $k \in \mathbb{K}$, the Sinkhorn discrepancy is convex and weakly lower semicontinuous (see the proof of Lemma~\ref{lemma-weakly-compact}), and $\mathbb{E}_{\boldsymbol{\omega} \sim \mathbb{P}_k } \Big[ L_{k}(\phi, \boldsymbol{\omega}) \Big]$ is weakly upper semicontinuous in $\mathbb{P}_k $ by Lemma~\ref{lemma-wusc}. 
    Hence, $H(\phi, \mathbb{P}_0, \mathbb{P}_1)$ is weakly upper semicontinuous and its upper-level set $\{(\mathbb{P}_0,\mathbb{P}_1): H(\phi, \mathbb{P}_0, \mathbb{P}_1) \ge \alpha\}$ is weakly closed. 
    Therefore, for every $\alpha \in \mathbb{R}$ and every fixed $\phi \in \mathcal{F}$, the upper-level set $\{(\mathbb{P}_0,\mathbb{P}_1): H(\phi, \mathbb{P}_0, \mathbb{P}_1) \ge \alpha\}$, a closed subset of a weakly compact set, is also weakly compact. 
    In addition, $-H(\phi, \mathbb{P}_0, \mathbb{P}_1)$ is convex, weakly lower semicontinuous, and thus weakly closed in $\mathbb{P}_0$ and $\mathbb{P}_1$. 
    
    Finally, $H(\cdot,\mathbb{P}_0,\mathbb{P}_1)$ is convex in $\phi$, and Assumption~\ref{assumption-2}\ref{assumption-2-properness}
    ensures that
    \[
        \sup_{\mathbb{P}_0,\mathbb{P}_1}
        \inf_{\phi\in\mathcal{F}} \quad
        H(\phi,\mathbb{P}_0,\mathbb{P}_1)>-\infty.
    \]
    Therefore, as all conditions of the lopsided minimax theorem are satisfied, the strong duality of the problem~\eqref{Eq:soft:constrained:SDRO} holds, i.e., 
    \begin{equation}\label{Eq-soft-strong-dual-sup}
        \text{Optimal Value of~\eqref{Eq:soft:constrained:SDRO} = } \sup_{\mathbb{P}_0, \mathbb{P}_1 \in \mathscr{P}(\Omega)} \, \inf_{\phi \in \mathcal{F}} \,  \Phi(\phi;\mathbb{P}_0,\mathbb{P}_1) - \sum_{k \in \mathbb{K}} \lambda_k \mathcal{S}_{\Reg}(\widehat{\mathbb{P}}_k,\mathbb{P}_k).   
    \end{equation}    
    Further, since $\inf_{\phi\in\mathcal{F}} \,\, H(\phi,\mathbb{P}_0,\mathbb{P}_1)$ is weakly upper semicontinuous in $\mathbb{P}_0$ and $\mathbb{P}_1$, has a finite upper bound by Eq.~\eqref{Eq-H-upper-bound} under proper $\beta$, and has weakly compact upper-level sets, the supremum in Eq.~\eqref{Eq-soft-strong-dual-sup} is attained by Weierstrass' theorem, i.e., 
    \begin{equation}\nonumber
        \text{Optimal Value of~\eqref{Eq:soft:constrained:SDRO} = } \max_{\mathbb{P}_0, \mathbb{P}_1 \in \mathscr{P}(\Omega)} \, \inf_{\phi \in \mathcal{F}} \,  \Phi(\phi;\mathbb{P}_0,\mathbb{P}_1) - \sum_{k \in \mathbb{K}} \lambda_k \mathcal{S}_{\Reg}(\widehat{\mathbb{P}}_k,\mathbb{P}_k).   
    \end{equation}    
    This completes the proof. 
\end{proof}

\subsection{Proof of Theorem~\ref{theorem-hycnn-rp} }\label{app-sec-theo-hycnn-rp}

\begin{proof}[Proof of Theorem~\ref{theorem-hycnn-rp}.] 

    As a preparation, when $\mathbb{D}$ has a finite diameter $D$, for any $\upsilon \in (0, 1)$, there exists a $\upsilon$-net $\mathbb{D}_{\upsilon} \subset \mathbb{D}$ containing $J$ points such that $\forall \boldsymbol{x} \in \mathbb{D}$, $\exists \boldsymbol{x}_j \in \mathbb{D}_{\upsilon}$, $\|\boldsymbol{x} -\boldsymbol{x}_j \| \le \upsilon$. 
    According to~\cite{shapiro2021lectures}, the size of the $\upsilon$-net is bounded by $J \le \mathcal{O}((D/\upsilon)^d)$. 

    We prove Theorem~\ref{theorem-hycnn-rp}\ref{theorem-hycnn-rp-maxout} in two steps.  
    According to Lemma 1 in~\cite{chen2018optimal}, all continuous Lipschitz convex functions over convex compact sets can be approximated using the maximum of affine functions. 
    Thus, in Step 1, we control the discretization error by adjusting the size $J$ of the $\upsilon$-net. 
    In Step 2, we prove that a HyCNN with maxout activation functions can exactly represent a maximum of affine functions and determine the quantitative relationship between $J$ and the network hyperparameters $(L, h)$. 

    \textbf{Step 1 of Theorem~\ref{theorem-hycnn-rp}\ref{theorem-hycnn-rp-maxout}: } For any $L_f$-Lipschitz function $f$, we approximate it by the maximum of supporting hyperplanes at points in the $\upsilon$-net. 
    The supporting hyperplane at $(\boldsymbol{x}_j, f(\boldsymbol{x}_j))$ is defined as
    \[
        s_j(\boldsymbol{x}) :=\boldsymbol{g}_j^{\top} (\boldsymbol{x}-\boldsymbol{x}_j) +f(\boldsymbol{x}_j). 
    \]
    where $\boldsymbol{g}_j \in \partial f(\boldsymbol{x}_j)$. 
    For any $\boldsymbol{x} \in \mathbb{D}$, let $\boldsymbol{x}_{j_0} \in \mathbb{D}_\upsilon$ be the nearest point to $\boldsymbol{x}$ in the $\upsilon$-net. The pointwise error satisfies
    \begin{equation}\nonumber 
       \begin{aligned}
        \sup_{\boldsymbol{x} \in \mathbb{D}} \left| f(\boldsymbol{x}) - \max_{j \in [J]} s_j(\boldsymbol{x}) \right| 
        & = \sup_{\boldsymbol{x} \in \mathbb{D}} \left( f(\boldsymbol{x}) - \max_{j \in [J]} s_j(\boldsymbol{x}) \right) \\
        & \le \sup_{\boldsymbol{x} \in \mathbb{D}} \Big( f(\boldsymbol{x}) - s_{j_0}(\boldsymbol{x}) \Big) \\
        & = \sup_{\boldsymbol{x} \in \mathbb{D}} \Big( f(\boldsymbol{x}) - f(\boldsymbol{x}_{j_0}) - \boldsymbol{g}_j^{\top} (\boldsymbol{x}-\boldsymbol{x}_{j_0}) \Big)  \\
        & \le \sup_{\boldsymbol{x} \in \mathbb{D}} \Big( |f(\boldsymbol{x}) - f(\boldsymbol{x}_{j_0})| + \|\boldsymbol{g}_j\| \cdot \|\boldsymbol{x}-\boldsymbol{x}_{j_0}\|\Big)  \\
        & \le 2 L_f \|\boldsymbol{x}-\boldsymbol{x}_{j_0}\| \le 2L_f \upsilon. 
    \end{aligned}
\end{equation}
    Here, the first equality holds since $f$ is convex, implying $s_j(\boldsymbol{x}) \le f(\boldsymbol{x})$ for all $j$. The first inequality follows because $\max_{j \in [J]} s_j(\boldsymbol{x}) \ge s_{j_0}(\boldsymbol{x})$. The second inequality uses the triangle inequality and Cauchy-Schwarz inequality. The third inequality applies the $L_f$-Lipschitz continuity of $f$, which bounds both $|f(\boldsymbol{x}) - f(\boldsymbol{x}_{j_0})| \le L_f \|\boldsymbol{x} - \boldsymbol{x}_{j_0}\|$ and $\|\boldsymbol{g}_j\| \le L_f$. The last inequality holds as $\|\boldsymbol{x} - \boldsymbol{x}_{j_0}\| \le \upsilon$. 
    To ensure $\sup_{\boldsymbol{x} \in \mathbb{D}} |f(\boldsymbol{x}) - \max_{j \in [J]} s_j(\boldsymbol{x})|  \le \varepsilon$ for any $\varepsilon >0$, we have $ \upsilon \le \min \{ \frac{\varepsilon}{2L_f},1 \}$. 
    Taking $\upsilon = \varepsilon/2L_f$. 
    By $J \le \mathcal{O}((\frac{D}{\upsilon})^d)$, the number of points in the $\upsilon$-net is given by $J \le \mathcal{O}((\frac{DL_f} {\varepsilon})^d)$. 

    \textbf{Step 2 of Theorem~\ref{theorem-hycnn-rp}\ref{theorem-hycnn-rp-maxout}: } Here, we show that the maximum of $J$ affine functions can be represented by a HyCNN with $L$ layers and width $h$. 
    As an intuitive example, we first consider a maximum of two affine functions and reformulate it as a HyCNN with two layers $\boldsymbol{y}_1$ and $\boldsymbol{y}_2$ (see Figure~\ref{fig:hycnn}) by introducing an extra affine term: 
    \begin{equation} \nonumber
        \begin{aligned}
            \max \Big \{ \boldsymbol{a}_1^\top \boldsymbol{x} + b_1, \boldsymbol{a}_2^\top \boldsymbol{x} + b_2 \Big \}  = \underbrace{(\boldsymbol{a}_3^\top \boldsymbol{x} + b_3) + \overbrace{\max \Big \{ (\boldsymbol{a}_1-\boldsymbol{a}_3)^\top \boldsymbol{x} + (b_1-b_3), (\boldsymbol{a}_2-\boldsymbol{a}_3)^\top \boldsymbol{x} + (b_2-b_3) \Big\}}^{\boldsymbol{y}_1}}_{\boldsymbol{y}_2 }
        \end{aligned}
    \end{equation}

    We extend this intuition to a maximum of $J$ affine functions $\max \{ \boldsymbol{a}_1^\top \boldsymbol{x} + b_1, \cdots, \boldsymbol{a}_{J}^\top \boldsymbol{x} + b_J \}$. 
    We consider a simple case in which the width of each hidden layer is $h=1$. 
    In the following, we first construct a HyCNN with maxout activation functions that use both positive and negative weights, representing a maximum over affine functions. 
    Then, all weights can be restricted to be nonnegative by the method in the proof of Theorem 1 in~\cite{chen2018optimal} without changing the width and depth of the network. 
    Let $s_{j} := \boldsymbol{a}_j^\top \boldsymbol{x} + b_j$ for any $j \in [2J]$, extra affine terms $r_j:= \boldsymbol{c}_j^\top \boldsymbol{x} + d_j$ for any $j\in [2,\cdots, J]$, $s_j^{\prime} = s_j - r_j$ for any $j \in [2,\cdots, J]$, and $\hat{\sigma}: \mathbb{R}^2 \to \mathbb{R}$ be a maxout activation function.
    Then 
    \begin{equation}\label{Eq-exact-represent}
        \begin{aligned}
            &\max \{s_1, \cdots, s_J\} \\
            = & \max \{ \max\{s_1, \cdots, s_{J-1}\}, s_J \} \\
            = & r_{J} + \hat{\sigma} ( \max\{s_1, \cdots, s_{J-1}\} - r_{J} , s_J^{\prime} ) \\ 
            = & r_{J} + \hat{\sigma} ( \max \{ \max\{s_1, \cdots, s_{J-2}\}, s_{J-1}\} - r_{J} , s_J^{\prime} ) \\ 
            = & r_{J} + \hat{\sigma} ( r_{J-1} + \hat{\sigma} ( \max\{s_1, \cdots, s_{J-2}\} -  r_{J-1} , s_{J-1}^{\prime} ) -  r_{J} , s_J^{\prime} ) \\
            = & \cdots\\
            = & r_{J} + \hat{\sigma} ( r_{J-1} + \hat{\sigma} ( r_{J-2} + \hat{\sigma} ( \cdots + r_3 + \hat{\sigma} ( r_{2} + \hat{\sigma} (s_1 - r_2, s_2^{\prime}) - r_3, s_3^{\prime}) - \cdots ) - r_{J-2}, s_{J-1}^{\prime}) - r_J, s_J^{\prime}). 
        \end{aligned}
    \end{equation}
    The last equation describes a $J$-layer HyCNN with width $h=1$, where the layers are
    \begin{equation}\nonumber
        \begin{aligned}
            y_1 &= \hat{\sigma} (s_1 - r_2, s_2 - r_2), \\
            y_2 &=  \hat{\sigma} ( 1 \cdot y_1  + r_{2} - r_3, 0 \cdot y_1 + s_3 - r_3), \\
            & \cdots \\
            y_{j} &= \hat{\sigma} (1 \cdot y_{j-1} + r_{j} - r_{j+1}, 0 \cdot y_{j-1} + s_{j+1} - r_{j+1}), \\
            & \cdots \\
            y_{J} &= r_{J} + y_{J-1}. 
        \end{aligned}
    \end{equation}
    To ensure the coefficients of previous layers in this neural network are nonnegative, following~\cite{chen2018optimal}, we expand the original input $\boldsymbol{x} \in \mathbb{R}^d$ in each layer to $\widehat{\boldsymbol{x}} \in \mathbb{R}^{2d}$ that includes $\boldsymbol{x}$ and $-\boldsymbol{x}$, i.e., $\widehat{\boldsymbol{x}} = [\boldsymbol{x}^{\top}, -\boldsymbol{x}^{\top}]^\top$. 
    When all coefficients of $\boldsymbol{x}$ are nonnegative, the coefficients of $ -\boldsymbol{x}$ are 0. 
    When $h_j < 0$, we set the coefficient of $\widehat{x}_{j+d}$ to $-h_j$ and that of $x_j$ to 0. 
    Therefore, without loss of generality, we can limit all weights between layers to be nonnegative, thereby ensuring that the neural network is input-convex.     

    The construction in Eq.~\eqref{Eq-exact-represent} represents $\max_{j\in[J]}s_j(\boldsymbol{x})$ exactly by a HyCNN with width $h=1$ and depth $L=\mathcal{O}(J)$. By Step 1, choosing $h=1$ and $L=\mathcal{O}(J)$ gives \[
    L \cdot h = \mathcal{O}((\frac{DL_f} {\varepsilon})^d), \]
    which proves Theorem~\ref{theorem-hycnn-rp}\ref{theorem-hycnn-rp-maxout}. 
    
    For Theorem~\ref{theorem-hycnn-rp}\ref{theorem-hycnn-rp-smooth-1}, we first focus on the property of the smooth maxout activation function. 
    Without loss of generality, suppose $a_1 \ge a_2$, we have $\bar{\sigma}(a_1, a_2) = a_1$ and 
    \[
        \bar{\sigma}_\tau(a_1, a_2) = \tau \log \left( e^{a_1/\tau} (1 + e^{(a_2 - a_1)/\tau}) \right) = a_1 + \tau \log \left( 1 + e^{(a_2 - a_1)/\tau} \right) \le a_1 + \tau \log 2,
    \]
    where $\tau>0$ and the last inequality is due to $e^{(a_2 - a_1)/\tau} \in (0, 1]$. 
    This is equivalent to 
    \[
        \bar{\sigma}(a_1, a_2) < \bar{\sigma}_\tau(a_1, a_2)  \le \bar{\sigma}(a_1, a_2) + \tau \log 2. 
    \]
    Consequently, for a HyCNN and its smooth version, we have 
    \[
        \sup_{\boldsymbol{x} \in \mathbb{D}} | \widehat{\mathsf{HyCNN}}^{\tau}(\boldsymbol{x}) - \widehat{\mathsf{HyCNN}}(\boldsymbol{x}) | \le C \cdot \tau \log 2
    \]
    where $C>0$ is a constant depending on the number of maxout activation functions and layers in the network. 
    By Theorem~\ref{theorem-hycnn-rp}\ref{theorem-hycnn-rp-maxout}, for any $\varepsilon > 0$, there exists a HyCNN such that 
    \[
        \sup_{\boldsymbol{x} \in \mathbb{D}} | \widehat{\mathsf{HyCNN}}(\boldsymbol{x}) - f(\boldsymbol{x}) | \le \frac{\varepsilon}{2}. 
    \]
    By selecting a sufficiently small $\tau$, e.g., $\tau \le \frac{\varepsilon}{2C\log2}$, we have 
    \[
        \sup_{\boldsymbol{x} \in \mathbb{D}} | \widehat{\mathsf{HyCNN}}^{\tau}(\boldsymbol{x}) - f(\boldsymbol{x}) | \le \sup_{\boldsymbol{x} \in \mathbb{D}} | \widehat{\mathsf{HyCNN}}^{\tau}(\boldsymbol{x}) - \widehat{\mathsf{HyCNN}}(\boldsymbol{x}) |+\sup_{\boldsymbol{x} \in \mathbb{D}} | \widehat{\mathsf{HyCNN}}(\boldsymbol{x}) - f(\boldsymbol{x}) | \le \varepsilon. 
    \]
    Note that the smooth HyCNN and standard HyCNN share the same order of network architecture complexity. 
    Hence, the hyperparameters of smooth HyCNN can also be chosen as $L \cdot h = \mathcal{O}((\frac{DL_f} {\varepsilon})^d)$ to guarantee the $\varepsilon$-approximation. 
    This completes the proof of Theorem~\ref{theorem-hycnn-rp}\ref{theorem-hycnn-rp-smooth-1}. 
    
    For Theorem~\ref{theorem-hycnn-rp}\ref{theorem-hycnn-rp-smooth-2}, we select a sequence of smooth HyCNNs on a compact domain $[-n, n]^d$ with sufficiently small $\tau$ and each satisfying $L_n \cdot h_n = \mathcal{O}((\frac{DL_f} {\varepsilon})^d)$. 
    By setting $\varepsilon_n = 1/n$, according to Theorem~\ref{theorem-hycnn-rp}\ref{theorem-hycnn-rp-smooth-1}, we have $\sup_{\boldsymbol{x} \in \mathbb{D}} | \widehat{\mathsf{HyCNN}}_{\varepsilon_n}^{\tau}(\boldsymbol{x}) - f(\boldsymbol{x}) | \le \frac{1}{n}$ where the right-hand-side term converges to 0 as $n \to \infty$. 
    Hence, we can prove $\widehat{\mathsf{HyCNN}}_n^{\tau}$ converges to $f$ uniformly on $\mathbb{D}$. 
    This completes the proof. 
\end{proof}

\subsection{Proof of Theorem~\ref{theo-hycnn-gradient-convergence}}\label{appsec-theo-hycnn-gradient-convergence}

\begin{proof}[Proof of Theorem~\ref{theo-hycnn-gradient-convergence}]
    For Theorem~\ref{theo-hycnn-gradient-convergence}\ref{theo-hycnn-gradient-convergence-1}, one can see Theorem 2 in~\cite{huang2020convex} for a detailed proof. 

    For Theorem~\ref{theo-hycnn-gradient-convergence}\ref{theo-hycnn-gradient-convergence-2}, let the function $f$ be $M$-Lipschitz smooth, and then we have 
    \begin{equation} \label{Eq-gradient-converge-1}
        f(\boldsymbol{y}) - f(\boldsymbol{x}) \le \nabla f(\boldsymbol{x})^{\top}(\boldsymbol{y} -\boldsymbol{x}) + \frac{M}{2} \cdot \| \boldsymbol{y} -\boldsymbol{x}\|^2. 
    \end{equation}
    For any $\widehat{\mathsf{HyCNN}}_n^{\tau}$ in the sequence $\{\widehat{\mathsf{HyCNN}}_n^{\tau}\}_{n=1}^{\infty}$ with hyperparameters satisfying $L_n \cdot h_n = \mathcal{O}({\varepsilon_n}^{-d})$ for any $\varepsilon_n>0$, we have 
    \begin{equation}\label{Eq-gradient-converge-2}
        \widehat{\mathsf{HyCNN}}_n^{\tau}(\boldsymbol{y}) - \widehat{\mathsf{HyCNN}}_n^{\tau}(\boldsymbol{x}) \ge \nabla \widehat{\mathsf{HyCNN}}_n^{\tau}(\boldsymbol{x})^{\top}(\boldsymbol{y} -\boldsymbol{x}), \quad \forall \boldsymbol{x}, \boldsymbol{y}\in \mathbb{R}^d.     
    \end{equation}
    By Eqs.~\eqref{Eq-gradient-converge-1} and~\eqref{Eq-gradient-converge-2}, we have 
    \[
        \Big(\nabla \widehat{\mathsf{HyCNN}}_n^{\tau}(\boldsymbol{x}) - \nabla f(\boldsymbol{x})\Big)^{\top} (\boldsymbol{y} -\boldsymbol{x})\le \Big(\widehat{\mathsf{HyCNN}}_n^{\tau}(\boldsymbol{y})- f(\boldsymbol{y})\Big) -\Big(\widehat{\mathsf{HyCNN}}_n^{\tau}(\boldsymbol{x})- f(\boldsymbol{x})\Big) + \frac{M}{2} \cdot \| \boldsymbol{y} -\boldsymbol{x}\|^2. 
    \]
    Let $\boldsymbol{y}$ be obtained by moving $\boldsymbol{x}$ along the direction $\nabla \widehat{\mathsf{HyCNN}}_n^{\tau}(\boldsymbol{x}) - \nabla f(\boldsymbol{x})$ with a proper step length $\eta$, i.e.,  $\boldsymbol{y} -\boldsymbol{x} = \eta \frac{\nabla \widehat{\mathsf{HyCNN}}_n^{\tau}(\boldsymbol{x}) - \nabla f(\boldsymbol{x})}{\|\nabla \widehat{\mathsf{HyCNN}}_n^{\tau}(\boldsymbol{x}) - \nabla f(\boldsymbol{x})\|}$ and $\boldsymbol{y} \in \mathbb{D}$. 
    Then, for the point-wise supremum value of $\nabla \widehat{\mathsf{HyCNN}}_n^{\tau}(\boldsymbol{x}) - \nabla f(\boldsymbol{x})$ on compact set $\mathbb{D} \subset \mathbb{R}^d$, we have 
    \[
    \begin{aligned}
        \sup_{\boldsymbol{x} \in \mathbb{D}} \Big \| \nabla \widehat{\mathsf{HyCNN}}_n^{\tau}(\boldsymbol{x}) - \nabla f(\boldsymbol{x})\Big\| & \le \frac{1}{\eta} \cdot \sup_{\boldsymbol{x} \in \mathbb{D}} \Big(\widehat{\mathsf{HyCNN}}_n^{\tau}(\boldsymbol{y})- f(\boldsymbol{y}) -  \Big(\widehat{\mathsf{HyCNN}}_n^{\tau}(\boldsymbol{x})- f(\boldsymbol{x})\Big)\Big) + \frac{M}{2\eta} \cdot \| \boldsymbol{y} -\boldsymbol{x}\|^2 \\
        & \le  \frac{2\varepsilon_n}{\eta} + \frac{M \eta}{2} \le 2\sqrt{M \varepsilon_n} = \mathcal{O}(\sqrt{\varepsilon_n}) , 
    \end{aligned}
    \]
    where the second inequality is due to Theorem~\ref{theorem-hycnn-rp}\ref{theorem-hycnn-rp-smooth-1} and $\| \boldsymbol{y} -\boldsymbol{x}\|^2 = \eta^2$, the last inequality is due to the basic inequality, and the last equality is due to $L_n \cdot h_n = \mathcal{O}({\varepsilon_n}^{-d})$ for any $\varepsilon_n>0$. 
    Hence, the gradient sequence $\nabla \widehat{\mathsf{HyCNN}}_n$ converges to $\nabla f$ uniformly at a rate of $\mathcal{O}\left( \sqrt{\varepsilon_n}\right)$ when function $f$ is Lipschitz smooth. 
    This completes the proof. 
\end{proof}

\subsection{Proof of Theorem~\ref{theo-hycnn-universality}}\label{app-sec-theo-hycnn-universality}

\begin{proof}[Proof of Theorem~\ref{theo-hycnn-universality}.]
    Since $\mu$ is absolutely continuous and both $\mu$ and $\nu$
    have finite second moments, Brenier's theorem yields a finite
    convex function $f:\mathbb R^d\to\mathbb R$ such that
    \[
        (\nabla f)_{\#}\mu=\nu.
    \]
    Since every finite convex function on $\mathbb R^d$ is locally
    Lipschitz, for every $n\in\mathbb N$, $f$ is Lipschitz on a
    neighborhood of the compact set
    \[
        \mathbb D_n=[-n,n]^d.
    \]
    By Theorem~\ref{theorem-hycnn-rp}
    \ref{theorem-hycnn-rp-smooth-1}, there exists a smooth HyCNN
    $\widehat{\mathsf{HyCNN}}_n^{\tau_n}$ satisfying
    \[
        \sup_{\boldsymbol{x}\in\mathbb D_n}
        \left|
        \widehat{\mathsf{HyCNN}}_n^{\tau_n}(\boldsymbol{x})
        -f(\boldsymbol{x})
        \right|
        \le \frac1n.
    \]

    For any fixed compact set $\mathbb K\subset\mathbb R^d$, there
    exists $N$ such that $\mathbb K\subseteq\mathbb D_n$ for every
    $n\ge N$. Hence,
    \[
        \sup_{\boldsymbol{x}\in\mathbb K}
        \left|
        \widehat{\mathsf{HyCNN}}_n^{\tau_n}(\boldsymbol{x})
        -f(\boldsymbol{x})
        \right|
        \le\frac1n\to0.
    \]
    Thus,
    $\widehat{\mathsf{HyCNN}}_n^{\tau_n}\to f$
    locally uniformly, and in particular pointwise on
    $\mathbb R^d$.

    By Theorem~\ref{theo-hycnn-gradient-convergence}
    \ref{theo-hycnn-gradient-convergence-1},
    \[
        \nabla\widehat{\mathsf{HyCNN}}_n^{\tau_n}
        (\boldsymbol{x})
        \to
        \nabla f(\boldsymbol{x})
    \]
    at almost every $\boldsymbol{x}\in\mathbb R^d$.
    Since $\mu$ is absolutely continuous with respect to the
    Lebesgue measure, this implies the weak convergence of the pushforward measure holds $\mu$-almost everywhere, i.e., for 
    $\boldsymbol{X}\sim\mu$,
    \[
        \nabla\widehat{\mathsf{HyCNN}}_n^{\tau_n}(\boldsymbol{X})
        \to
        \nabla f(\boldsymbol{X})
        \qquad\text{almost surely}.
    \]
    Consequently,
    \[
        (\nabla\widehat{\mathsf{HyCNN}}_n^{\tau_n})_{\#}\mu
        \Rightarrow
        (\nabla f)_{\#}\mu
        =
        \nu, 
    \]
    where $\Rightarrow$ denotes weak convergence of probability measures on $\mathbb{R}^d$. 
    This completes the proof.
\end{proof}

\section{Technical Note for Training and Testing HyCNN}\label{app-sec-hypercnn-details}

In this section, we provide details to compute the gradients of the problem~\eqref{Eq-SDRO-convex-functional} in Algorithm~\ref{algorithm-HyperCNN}. 

\subsection{Gradient of the Risk Function Estimation}\label{app-sec-hypercnn-details-risk}

In this subsection, we take the $k=0$ term in Eq.~\eqref{Eq-Phi-risk} as an example. For brevity, we omit the superscript $(i)$. Denote the parameters of the neural network $\mathsf{HyCNN}_k$ as $\boldsymbol{\theta}_k$ for each $k \in \mathbb{K}$. 
Since our approximated convex function $\widetilde{\varphi}_k$ is strongly convex, we have $\det\, \nabla^2 \widetilde{\varphi}_k (\boldsymbol{z}) > 0$ for any $k$ and $\boldsymbol{z}$. 

\begin{enumerate}
    \item \textbf{Gradient with respect to $\boldsymbol{\theta}_0$. }\\ 
    For brevity, let $h(\boldsymbol{\theta}_0, \boldsymbol{\theta}_1) := \exp \Big[ \frac{1}{2} (\log \, p_{1} (\boldsymbol{\omega}_0) - \log \, p_{0} (\boldsymbol{\omega}_0) ) \Big] $, which is the same as $h_0(\boldsymbol{\omega}_0, \boldsymbol{\theta}_0, \boldsymbol{\theta}_1)$ in Algorithm~\ref{algorithm-risk-estimator}. Then, we have 
    \begin{equation}\label{Eq-gradient-theta0}
        \frac{\partial h(\boldsymbol{\theta}_0, \boldsymbol{\theta}_1)}{\partial \boldsymbol{\theta}_0}  = \frac{h(\boldsymbol{\theta}_0, \boldsymbol{\theta}_1)}{2} \cdot \Big [ \frac{\partial \log \, p_{1} (\boldsymbol{\omega}_0) }{\partial \boldsymbol{\theta}_0} - \frac{\partial \log \, p_{0} (\boldsymbol{\omega}_0)}{\partial \boldsymbol{\theta}_0}\Big]. 
    \end{equation}
    For the second item on the right-hand side of Eq.~\eqref{Eq-gradient-theta0}, based on Eq.~\eqref{Eq-change-of-variable-1}, we have 
    \begin{equation}\label{Eq-gradient-theta0-term2}
        \begin{aligned}
            \frac{\partial \log \, p_{0} (\boldsymbol{\omega}_0) }{\partial \boldsymbol{\theta}_0}  = \frac{\partial }{\partial \boldsymbol{\theta}_0} \log \, p_{0}^{\Reg} (\boldsymbol{z}_0) - \frac{\partial }{\partial \boldsymbol{\theta}_0}  \log \, \det \nabla^2 \widetilde{\varphi}_0 (\boldsymbol{z}_0) = - \frac{\partial }{\partial \boldsymbol{\theta}_0}  \log \, \det \nabla^2 \widetilde{\varphi}_0 (\boldsymbol{z}_0),
        \end{aligned}
    \end{equation}
    where $\boldsymbol{z}_0 = \nabla \varphi_0^{-1} (\boldsymbol{\omega}_0)$, the first equality is due to the strong convexity of $\widetilde{\varphi}_0$, and the second equality is due to $ \frac{\partial }{\partial \boldsymbol{\theta}_0} \log \, p_{0}^{\Reg} (\boldsymbol{z}_0) =0$. The log-determinant-Hessian in Eq.~\eqref{Eq-gradient-theta0-term2} can be efficiently computed by Algorithm~\ref{algorithm-log-determinant-estimator}. 
    
    For the first item on the right-hand side of Eq.~\eqref{Eq-gradient-theta0}, based on the chain rule, we have 
    \begin{equation}\label{Eq-gradient-theta0-term1}
        \begin{aligned}
            \frac{\partial \log \, p_{1} (\boldsymbol{\omega}_0) }{\partial \boldsymbol{\theta}_0} & = \left( \frac{\partial \boldsymbol{\omega}_0}{\partial \boldsymbol{\theta}_0} \right)^\top \left( \frac{\partial \bar{\boldsymbol{z}}_0}{\partial \boldsymbol{\omega}_0} \right)^\top \frac{\partial \log \, p_{1} (\boldsymbol{\omega}_0)}{\partial \bar{\boldsymbol{z}}_0}. 
        \end{aligned} 
    \end{equation}
    Let vector 
    \begin{equation} \nonumber 
        \boldsymbol{g} := \frac{\partial \log \, p_{1} (\boldsymbol{\omega}_0)}{\partial \bar{\boldsymbol{z}}_0} =\frac{\partial  \Big[ \log \, p_{1}^{\Reg} (\bar{\boldsymbol{z}}_0) - \log \, \det \nabla^2 \widetilde{\varphi}_{1} (\bar{\boldsymbol{z}}_0) \Big ]}{\partial \bar{\boldsymbol{z}}_0}, 
    \end{equation}
    and this derivative can be computed by automatic differentiation. 
    Let $\bar{\boldsymbol{z}}_0 = \nabla \widetilde{\varphi}_1^{-1} (\boldsymbol{\omega}_0)$. 
    As $\frac{\partial \bar{\boldsymbol{z}}_0}{\partial \boldsymbol{\omega}_0} = [\nabla^2 \varphi_1(\bar{\boldsymbol{z}}_0)]^{-1}$, Eq.\eqref{Eq-gradient-theta0-term1} can be reformulated as 
    \begin{equation}\label{Eq-gradient-theta0-vjp}
        \frac{\partial \log \, p_{1} (\boldsymbol{\omega}_0) }{\partial \boldsymbol{\theta}_0} = \left( \frac{\partial \boldsymbol{\omega}_0}{\partial \boldsymbol{\theta}_0} \right)^\top [\nabla^2 \widetilde{\varphi}_1(\bar{\boldsymbol{z}}_0)]^{-1} \boldsymbol{g} ,
    \end{equation}
    where $[\nabla^2 \widetilde{\varphi}_1(\bar{\boldsymbol{z}}_0)]^{-1} \boldsymbol{g}$ can be efficiently solved as the minimizer of a quadratic optimization problem similar to~\eqref{Eq-quadratic-optimization} without explicitly computing the inverse of the Hessian matrix. 
    Then, $\left( \frac{\partial \boldsymbol{\omega}_0}{\partial \boldsymbol{\theta}_0} \right)^\top [\nabla^2 \widetilde{\varphi}_1(\bar{\boldsymbol{z}}_0)]^{-1} \boldsymbol{g}$ can be computed by vector-Jacobian product in an automatic differentiation framework. 
    In summary, the gradient of $h(\boldsymbol{\theta}_0, \boldsymbol{\theta}_1)$ with respect to $\boldsymbol{\theta}_0$ is given by 
    \begin{equation}\nonumber
        \frac{\partial h(\boldsymbol{\theta}_0, \boldsymbol{\theta}_1)}{\partial \boldsymbol{\theta}_0}  = \frac{h(\boldsymbol{\theta}_0, \boldsymbol{\theta}_1)}{2} \cdot \Big [ \left( \frac{\partial \boldsymbol{\omega}_0}{\partial \boldsymbol{\theta}_0} \right)^\top [\nabla^2 \widetilde{\varphi}_1(\bar{\boldsymbol{z}}_0)]^{-1} \boldsymbol{g} + \frac{\partial }{\partial \boldsymbol{\theta}_0}  \log \, \det \nabla^2 \widetilde{\varphi}_0 (\boldsymbol{z}_0)\Big]. 
    \end{equation}
        
    \item \textbf{Gradient with respect to $\boldsymbol{\theta}_1$. }\\ 
    Similar to Eq.~\eqref{Eq-gradient-theta0}, we have 
    \begin{equation}\label{Eq-gradient-theta1}
        \frac{\partial h(\boldsymbol{\theta}_0, \boldsymbol{\theta}_1)}{\partial \boldsymbol{\theta}_1}  = \frac{h(\boldsymbol{\theta}_0, \boldsymbol{\theta}_1)}{2} \cdot \Big [ \frac{\partial \log \, p_{1} (\boldsymbol{\omega}_0) }{\partial \boldsymbol{\theta}_1} - \frac{\partial \log \, p_{0} (\boldsymbol{\omega}_0)}{\partial \boldsymbol{\theta}_1}\Big] = \frac{h(\boldsymbol{\theta}_0, \boldsymbol{\theta}_1)}{2} \cdot \frac{\partial \log \, p_{1} (\boldsymbol{\omega}_0) }{\partial \boldsymbol{\theta}_1} . 
    \end{equation}
    Since $\bar{\boldsymbol{z}}_0 = \nabla \widetilde{\varphi}_1^{-1} (\boldsymbol{\omega}_0)$, we have 
    \begin{equation}\label{Eq-gradient-theta1-term1}
        \begin{aligned}
            \frac{\partial \log \, p_{1} (\boldsymbol{\omega}_0) }{\partial \boldsymbol{\theta}_1} &= \frac{\partial }{\partial \boldsymbol{\theta}_1} \Big[ \log \, p_{1}^{\Reg} (\bar{\boldsymbol{z}}_0) - \log \, \det \nabla^2 \widetilde{\varphi}_{1} (\bar{\boldsymbol{z}}_0) \Big] \\
            & = (\frac{\partial \bar{\boldsymbol{z}}_0}{\partial \boldsymbol{\theta}_1})^{\top} \frac{\partial \log \, p_{1}^{\Reg} (\bar{\boldsymbol{z}}_0) - \log \, \det \nabla^2 \widetilde{\varphi}_{1} (\bar{\boldsymbol{z}}_0)}{\partial \bar{\boldsymbol{z}}_0} - \frac{\partial }{\partial \boldsymbol{\theta}_1} \log \, \det \nabla^2 \widetilde{\varphi}_{1} (\boldsymbol{z}) \Big |_{\boldsymbol{z} = \bar{\boldsymbol{z}}_0} , 
        \end{aligned}
    \end{equation}
    where the second equality is based on the total derivative. 
    To evaluate the Jacobian $\frac{\partial \bar{\boldsymbol{z}}_0}{\partial \boldsymbol{\theta}_1}$ in the first term, we take the total derivative with respect to $\boldsymbol{\theta}_1$ on both sides of $ \nabla \widetilde{\varphi}_1 (\bar{\boldsymbol{z}}_0) = \boldsymbol{\omega}_0$:
    \begin{equation}\nonumber
        \nabla^2 \widetilde{\varphi}_1(\bar{\boldsymbol{z}}_0)  \cdot \frac{\partial \bar{\boldsymbol{z}}_0}{\partial \boldsymbol{\theta}_1} + \frac{\partial \nabla \widetilde{\varphi}_1(\boldsymbol{z})}{\partial \boldsymbol{\theta}_1} \Big |_{\boldsymbol{z} = \bar{\boldsymbol{z}}_0} = \boldsymbol{0}, 
    \end{equation}
    and it follows that
    \begin{equation}\nonumber
        (\frac{\partial \bar{\boldsymbol{z}}_0}{\partial \boldsymbol{\theta}_1})^{\top} = - \Big [ \nabla^2 \widetilde{\varphi}_1(\bar{\boldsymbol{z}}_0) \Big ]^{-1} \frac{\partial \nabla \widetilde{\varphi}_1(\bar{\boldsymbol{z}}_0)}{\partial \boldsymbol{\theta}_1}. 
    \end{equation}
    Therefore, Eq.~\eqref{Eq-gradient-theta1-term1} can be reformulated as 
    \begin{equation} \nonumber
        \frac{\partial \log \, p_{1} (\boldsymbol{\omega}_0) }{\partial \boldsymbol{\theta}_1} = -\Big( \frac{\partial \nabla \widetilde{\varphi}_1 (\bar{\boldsymbol{z}}_0)}{\partial \boldsymbol{\theta}_1}\Big )^{\top} \Big [ \nabla^2 \widetilde{\varphi}_1(\bar{\boldsymbol{z}}_0) \Big ]^{-1} \boldsymbol{g} - \frac{\partial }{\partial \boldsymbol{\theta}_1} \log \, \det \nabla^2 \widetilde{\varphi}_{1} (\bar{\boldsymbol{z}}_0), 
    \end{equation}
    where the first term can be calculated in the same way as Eq.~\eqref{Eq-gradient-theta0-vjp} and the second term can be computed by Algorithm~\ref{algorithm-log-determinant-estimator} as $\bar{\boldsymbol{z}}_0$ is a constant here, thanks to the total derivative. 
    In summary, the gradient of $h(\boldsymbol{\theta}_0, \boldsymbol{\theta}_1)$ with respect to $\boldsymbol{\theta}_1$ is given by 
    \begin{equation}\nonumber
        \frac{\partial h(\boldsymbol{\theta}_0, \boldsymbol{\theta}_1)}{\partial \boldsymbol{\theta}_1} = - \frac{h(\boldsymbol{\theta}_0, \boldsymbol{\theta}_1)}{2} \cdot \Big[ \Big( \frac{\partial \nabla \widetilde{\varphi}_1 (\bar{\boldsymbol{z}}_0)}{\partial \boldsymbol{\theta}_1}\Big )^{\top} [ \nabla^2 \widetilde{\varphi}_1(\bar{\boldsymbol{z}}_0) ]^{-1} \boldsymbol{g} + \frac{\partial }{\partial \boldsymbol{\theta}_1} \log \, \det \nabla^2 \widetilde{\varphi}_{1} (\bar{\boldsymbol{z}}_0) \Big]. 
    \end{equation}
\end{enumerate}

\subsection{Gradient of the Sinkhorn Discrepancy Estimation}\label{app-sec-hypercnn-details-sinkhorn}

When the potential $f$ is optimally solved, the total derivative of our Sinkhorn discrepancy estimator~\eqref{Eq-estimator-sinkhorn} with respect to the HyCNN parameters is given by 
\begin{equation}
    \begin{aligned}
        \frac{\partial \, \widetilde{\mathcal{S}}_{\Reg, (n_k, N)}(\widehat{\mathbb{P}}_k,\mathbb{P}_k)}{\partial \boldsymbol{\theta}_k}  = \frac{1}{N} \sum_{i=1}^{N} \frac{\partial g_{\Reg, k}^c(\boldsymbol{\omega}_k^{(i)})}{\partial \boldsymbol{\theta}_k} \Big |_{f = f_{\Reg, k, (n_k, N)}} = \frac{1}{N} \sum_{i=1}^{N} \left( \frac{\partial \boldsymbol{\omega}_k^{(i)}}{\partial \boldsymbol{\theta}_k} \right)^{\top} \cdot \frac{\partial g_{\Reg, k}^c(\boldsymbol{\omega}_k^{(i)})}{\partial \boldsymbol{\omega}_k^{(i)}}  \Big |_{f = f_{\Reg, k, (n_k, N)}}. 
    \end{aligned}
\end{equation}
where the first equality is due to the envelope theorem, $\frac{\partial \boldsymbol{\omega}_k^{(i)}}{\partial \boldsymbol{\theta}_k}$ is the Jacobian of the generated sample with respect to the HyCNN parameters, and the gradient of $g_{\Reg, k}^c$ with respect to $\boldsymbol{\omega}_k^{(i)}$ can be explicitly computed based on Eq.~\eqref{Eq-sinkhorn-potential-extension}.

\subsection{Testing}

In the testing phase, for a testing sample $\boldsymbol{\omega} \in \Omega$, we first trace its position in distributions $\mathbb{P}_0^{\epsilon}$ and $\mathbb{P}_1^{\epsilon}$ by $\bar{\boldsymbol{z}}_0 = \nabla \widetilde{\varphi}_0^{-1} (\boldsymbol{\omega})$ and $\bar{\boldsymbol{z}}_1 = \nabla \widetilde{\varphi}_1^{-1} (\boldsymbol{\omega})$, respectively. 
Then, we evaluate the log-density of least-favorable distributions $\mathbb{P}_0$ and $\mathbb{P}_1$ at $\boldsymbol{\omega}$ by
\begin{equation} \nonumber
    \log \, p_k(\boldsymbol{\omega}) = \log \, p_{k}^{\Reg} (\bar{\boldsymbol{z}}_k) - \log \,  \det \nabla^2 \widetilde{\varphi}_k (\bar{\boldsymbol{z}}_k) , \quad \forall k \in \mathbb{K}. 
\end{equation}
Finally, we make decisions by the selected optimal detector $\phi_{\text{opt}} = \frac{1}{2} (\log\, p_0(\boldsymbol{\omega}) - \log\, p_1(\boldsymbol{\omega}))$. 

\section{Wasserstein Distributionally Robust Hypothesis Testing Model}\label{app-sec-wdro}

In this section, we provide the model of the WDRO method that is used as a baseline in our experimental part. 
The WDRO method is the soft-constraint version of the Wasserstein distributionally robust hypothesis testing model proposed by~\cite{xie2021robust} (assuming $n_0=n_1=n$):
\begin{equation} \label{WDRO-HT}
    \begin{aligned}
        \max_{\boldsymbol{p}_0, \boldsymbol{p}_1 \in \mathbb{R}_+^{n}, \boldsymbol{\gamma}_0, \boldsymbol{\gamma}_1 \in \mathbb{R}_+^{n \times n}} \quad & \sum_{i \in [n]} \min \{ \boldsymbol{p}_0^{i}, \boldsymbol{p}_1^{i} \} - \sum_{k \in \mathbb{K}} \lambda_k \cdot \sum_{i \in [n]}\sum_{j \in [n]} \boldsymbol{\gamma}_{k, i, j} \cdot c(\boldsymbol{x}_k^{(i)}, \boldsymbol{x}_k^{(j)}) \\
        \text{s.t.} \quad & \sum_{j\in [n]} \boldsymbol{\gamma}_{k, i, j} = \widehat{\mathbb{P}}_k^{i} \quad \quad \forall i \in [n], k  \in \mathbb{K}, \\ 
        \quad & \sum_{i\in [n]} \boldsymbol{\gamma}_{k, i, j} = \boldsymbol{p}_k^{j}, \quad \quad \forall j \in [n], k  \in \mathbb{K}, \\ 
    \end{aligned}
\end{equation}
where $\lambda_k \ge 0$ is a penalty parameter for each $k \in \mathbb{K}$, $\boldsymbol{\gamma}_{k, i, j}$ represents the $ij$-th entry of $\boldsymbol{\gamma}_k$, and the $i$-th entry of $\boldsymbol{p}_k$ (respectively, $\widehat{\mathbb{P}}_k$) is specified by $\boldsymbol{p}_k^{i}$ (respectively, $\widehat{\mathbb{P}}_k^{i}$). 
This model assumes that the samples in the testing set have the same support as the empirical distribution. 
To extend the resulting least-favorable distribution to the entire space, a kernel smoothing method is introduced. 
Let $\boldsymbol{p}_k^{*}$ be an optimal solution of~\eqref{WDRO-HT}. 
This method convolves the discrete distribution with a kernel function $G_h: \mathbb{R}^d \to \mathbb{R}$ parameterized by a bandwidth $h>0$:
\begin{equation} \nonumber
    \boldsymbol{p}_k^{h} (\boldsymbol{\omega}) = \sum_{i=1}^{n} \boldsymbol{p}_k^{*} (\boldsymbol{x}_k^{(i)}) \cdot  G_h(\boldsymbol{\omega} - \boldsymbol{x}_k^{(i)}), \quad \forall \boldsymbol{\omega} \in \Omega, k \in \mathbb{K},
\end{equation}
where we choose the Gaussian kernel function $G_h(\boldsymbol{x}) = \exp(-\| \boldsymbol{x} \|^2/2h^2)$. 
Note that the number of decision variables of this linear program is $\mathcal{O}(n_0 \cdot n_1)$.
Thus, this finite-dimensional convex program is not sensitive to dimension $d$, but is sensitive to the number of training data points of the empirical distribution. 


\end{document}